\RequirePackage{fix-cm}
\documentclass{svjour3}                     
\smartqed  
\usepackage{graphicx}
\usepackage{comment}

\usepackage{amsmath}
\usepackage{amssymb}
\usepackage{bm}
\usepackage{mathrsfs}

\usepackage{algorithm}
\usepackage{algorithmic}

\usepackage{color}

\usepackage{url}

\usepackage{supertabular}
\usepackage{enumitem}

\DeclareMathAlphabet{\mathsfsl}{OT1}{cmss}{m}{sl}

\renewcommand{\phi}{\varphi}

\newcommand{\argmin}{\operatorname*{arg\,min}}

\newcommand{\sign}{\operatorname{sign}}

\newcommand{\mode}{\operatorname{mode}}

\newcommand{\bx}{\boldsymbol{x}}

\newcommand{\bX}{\boldsymbol{X}}

\newcommand{\bI}{\boldsymbol{I}}

\newcommand{\bP}{\boldsymbol{P}}

\def\b0{\boldsymbol{0}}
\def\bP{\boldsymbol{P}}

\def\bY{\boldsymbol{Y}}

\usepackage[toc,page]{appendix}

\DeclareMathOperator{\Tr}{Tr}

\usepackage{xspace} 

\begin{document}
\newcommand{\ith}{\textsc{A}\xspace}
 \newcommand{\er}{\textsc{B}\xspace}
 \newcommand{\eqc}[1]{%
  #1/{\sim}%
  }

\title{Robust Group Synchronization via Cycle-Edge Message Passing \thanks{This work was supported by NSF award DMS-1821266.}
}


\author{Gilad Lerman        \and
        Yunpeng Shi 
}


\institute{Gilad Lerman \at
              School of Mathematics, University of Minnesota, 127 Vincent Hall, 206 Church Street SE, Minneapolis, MN
55455 USA \\
              Tel.: +1-612-624-5541\\
              Fax: +1-612-624-6702\\
              \email{lerman@umn.edu}           
           \and
           Yunpeng Shi \at
              Program in Applied and Computational Mathematics, Princeton University, Princeton, NJ
08544 USA, \email{yunpengs@princeton.edu} 
}

\date{Received: date / Accepted: date}

\maketitle

\begin{abstract}
We propose a general framework for solving the  group synchronization problem, where we focus on the setting of adversarial or uniform corruption and sufficiently small noise.  Specifically, we apply a novel message passing procedure that uses cycle consistency information in order to estimate the corruption levels of group ratios and consequently solve the synchronization problem in our setting. We first explain why the group cycle consistency information is essential for effectively solving group synchronization problems. We then establish exact recovery and linear convergence guarantees for the proposed message passing procedure under a deterministic setting with adversarial corruption. These guarantees hold as long as the ratio of corrupted cycles per edge is bounded by a reasonable constant. We also establish the stability of the proposed procedure to sub-Gaussian noise. We further establish exact recovery with high probability under a common uniform corruption model.  
\end{abstract}

\section{Introduction}
The problem of synchronization arises in important data-related tasks, such as structure from motion (SfM), simultaneous localization and mapping (SLAM), Cryo-EM, community detection and sensor network localization. The underlying setting of the problem includes objects with associated states, where examples of states are locations, rotations and binary labels. The main problem is estimating the states of objects from the relative state measurements between pairs of objects. One example is rotation synchronization, which aims to recover rotations of objects from the relative rotations between pairs of objects. The problem is simple when one has the correct measurements of all relative states. However, in practice the measurements of some relative states can be erroneous or missing. The main goal of this paper is to establish a theoretically-guaranteed solution for general compact group synchronization that can tolerate large amounts of measurement error. 

We mathematically formulate the general problem in Section \ref{sec:math_formulation} and discuss common special cases of this problem in Section \ref{sec:examples}. Section \ref{sec:previous_short} briefly mentions the computational difficulties in solving this problem and the disadvantages of the common convex relaxation approach. Section \ref{sec:ours} non-technically describes our method, and Section \ref{sec:contribution} highlights its contributions. At last,  Section \ref{sec:rest} provides a roadmap for the rest of the paper.

\subsection{Problem Formulation}
\label{sec:math_formulation}
The most common mathematical setting of synchronization is group synchronization, which asks to recover group elements from their noisy group ratios. It assumes a group $\mathcal{G}$, a subset of this group $\{g_i^*\}_{i=1}^n$ and a graph $G([n],E)$ with $n$ vertices indexed by $[n]=\{1,\ldots,n\}$.  The group ratio between $g_i^*$ and $g_j^*$ is defined as $g_{ij}^*=g_i^*g_j^{*-1}$.   
We use the star superscript to emphasize original elements of $\mathcal{G}$, since the actual measurements can be corrupted or noisy.  We remark that since $g_{ji}^*={g_{ij}^*}^{-1}$, our setting of an undirected graph, $G([n],E)$, is fine.

We say that a ratio $g_{ij}^*$ is corrupted when it is replaced by $\tilde g_{ij} \in  \mathcal{G}\setminus \{g_{ij}^*\}$, either deterministically or probabilistically. 
We partition $E$ into the sets of uncorrupted (good) and corrupted (bad) edges, which we denote by $E_g$ and $E_b$, respectively. 

We denote the group identity by  $e_{\mathcal G}$. We assume a metric $d_\mathcal G$ on $\mathcal G$, which is bi-invariant. This means that for any $g_1,g_2,g_3\in \mathcal G$, 
 \begin{equation*}
 d_{\mathcal G}(g_1,g_2)=d_{\mathcal G}(g_3g_1,g_3g_2)=d_{\mathcal G}(g_1g_3,g_2g_3).
 \end{equation*}
We further assume that $\mathcal G$ is bounded with respect to $d_\mathcal G$
and we thus restrict our theory to compact groups. 
We  appropriately scale $ d_{\mathcal G}$ so that the diameter of $\mathcal G$ is at most 1. 
 
Additional noise can be applied to the group ratios associated with edges in $E_g$.  
For $ij \in E_g$, the noise model replaces $g^*_{ij}$ with $g^*_{ij}g_{ij}^\epsilon$, where $g_{ij}^\epsilon$ is a $\mathcal G$-valued random variable such that
$d_{\mathcal G}(g_{ij}^\epsilon, e_{\mathcal G})$ is sub-Gaussian. 
We denote the corrupted and noisy group ratios by $\{g_{ij}\}_{ij \in E}$ and summarize their form as follows:
\begin{align}\label{eq:model}
g_{ij}=\begin{cases}
g^*_{ij}g_{ij}^\epsilon, & ij \in E_g;\\
\tilde g_{ij}, & ij \in E_b.
\end{cases}
\end{align}
We refer to the case where $g_{ij}^\epsilon=e_\mathcal G$ for all $ij \in E$ as the noiseless case.
We view \eqref{eq:model} as an adversarial corruption model since the corrupted group ratios and the corrupted edges in $E_b$ can be arbitrarily chosen; however,  our theory introduces some restrictions on both of them.

The problem of group synchronization asks to recover the original group elements $\{g_i^*\}_{i\in [n]}$ given the graph $G([n],E)$ and corrupted and noisy group ratios $\{g_{ij}\}_{ij\in E}$. One can only recover, or approximate, the original group elements $\{g_i^*\}_{i\in [n]}$ up to a right group action.
Indeed, for any $g_0\in\mathcal G$, $g_{ij}^*$ can also be written as $g_i^*g_0(g_j^*g_0)^{-1}$ and thus $\{g_i^*g_0\}_{i\in [n]}$ is also a solution. 
It is natural to assume that $G([n], E_g)$ is connected, since in this case the arbitrary right multiplication is the only degree of freedom of the solution.

In the noiseless case, one aims to exactly recover the original group elements under certain conditions on the corruption and the graph.  
In the noisy case, one aims to nearly recover the original group elements with recovery error depending on the distribution of $d_{\mathcal G}(g_{ij}^\epsilon, e_{\mathcal G})$.

At last, we remark that for similar models where the measurement $g_{ij}$ may not be in $\mathcal G$ but in an embedding space, one can first project $g_{ij}$ onto $\mathcal G$ and then apply our proposed method. Any theory developed for our model can extend for the latter one by projecting onto $\mathcal G$.

\subsection{Examples of Group Synchronization}
\label{sec:examples}
We review the three common instances of group synchronization. 

\subsubsection{$\mathbb Z_2$ Synchronization}
\label{sec:z2}
This is the simplest and most widely known problem of group synchronization. The underlying group, $\mathbb Z_2$, is commonly represented in this setting by $\{-1,1\}$ with direct multiplication.
A natural motivation for this problem is binary graph clustering, where one wishes to recover the labels in $\{-1,1\}$ of two different clusters of graph nodes from corrupted measurements of signed interactions between pairs of nodes connected by edges. Namely, the signed interaction of two nodes is 1 if they are in the same cluster and -1 if they are in a different cluster. Note that without any erroneous measurement, the signed interaction is obtained by multiplying the corresponding labels and thus it corresponds to the group ratio  $g_{ij}^*=g_i^*g_j^{*-1}$. Also note that clusters are determined up to a choice of labels, that is, up to multiplication by an element of $\mathbb Z_2$.
The $\mathbb Z_2$ synchronization problem is directly related to the Max-Cut problem~\cite{wang2013exact} and to a special setting of community detection~\cite{Z2abbe,Z2}. It was also applied to solve a specific problem in sensor network localization~\cite{Z2singer}.
\subsubsection{Permutation Synchronization}
\label{sec:permutation}
The underlying group of this problem is the symmetric group, that is, the discrete group of permutations, $S_N$. This synchronization problem was proposed in computer vision in order to find globally consistent image matches from relative matches~\cite{deepti}.
More specifically, one has a set of images and $N$ feature points that are common to all images, such as distinguished corners of objects that appear in all images (they correspond to a set of $N$ points in the 3D scene). These feature points, often referred to as keypoints, are arbitrarily labeled in each image. For any pair of images one is given possibly corrupted versions of the relative permutations between their keypoints. 
One then needs to consistently label all keypoints in the given images. That is, one needs to find absolute permutations of the labels of keypoints of each image into the fixed labels of the $N$ 3D scene points. 
\subsubsection{Rotation Synchronization}
\label{sec:rotation}
The problem of rotation synchronization, or equivalently, $SO(3)$ synchronization, asks to recover absolute rotations from corrupted relative rotations up to a global rotation. Its special case of angular synchronization, or $SO(2)$ synchronization, asks to recover the locations of points on a circle  (up to an arbitrary rotation) given corrupted relative angles between pairs of points. More generally, one may consider $SO(d)$ synchronization for any $d \geq 2$.
Rotation synchronization is widely used in 3D imaging and computer vision tasks. In particular, \cite{singer_cryo} applies rotation synchronization for solving absolute rotations of molecules and \cite{Nachimson_LS,ChatterjeeG13_rotation,Govindu04_Lie,HartleyAT11_rotation,MartinecP07_rotation,OzyesilSB15_SDR,wang2013exact} synchronize the relative rotations of cameras to obtain the global camera rotations in the problem of structure from motion.

\subsection{On the Complexity of the Problem and Its Common Approach}
\label{sec:previous_short}
Many groups, such as $\mathbb Z_2$, $S_N$ and $SO(d)$ are non-convex and their synchronization problems are usually NP-hard~\cite{rotationNP,Z2NP,deepti}. Thus, many classic methods of group synchronization instead solve a relaxed semidefinite programming (SDP) problem (see review of previous methods and guarantees in Section \ref{sec:previous}). However, relaxation techniques may change the original problem and may  thus not recover the original group elements when the group ratios are severely corrupted. Furthermore, the SDP formulations and analysis are specialized to the different groups. Moreover, their computational time can still be slow in practice. 

\subsection{Short and Non-technical Description of Our Work and Guarantees}
\label{sec:ours}
The goal of this work is to formulate a universal and flexible framework that can address different groups in a similar way. It exploits cycle consistency, which is a common property shared by any group. That is, let $L=\{i_1i_2, i_2i_3\dots i_m i_1\}$ be any cycle of length $m$\footnote{Recall that a cycle in a graph is a closed trail whose first and last vertices are the only repeated vertices.} and define $g^*_L=g^*_{i_1i_2}g^*_{i_2i_3}\cdots g^*_{i_mi_1}$,
then the cycle consistency constraint is
\begin{equation}\label{eq:cyclecons}
g^*_L= e_\mathcal G.
\end{equation}
That is, the multiplication of the original group ratios along a cycle yields the group identity.
In practice, one may only compute the following approximation for $g^*_L$:
\begin{equation}\label{eq:defgL}
g_L=g_{i_1i_2}g_{i_2i_3}\cdots g_{i_mi_1},
\end{equation}
where for faster computation we prefer using only 3-cycles, so that $g_L = g_{ij} \, g_{jk} \, g_{ki}$. 

One basic idea is that the distances between $g_L$ and $e_\mathcal G$ for cycles $L$ containing edge $ij$, which we refer to as cycle inconsistencies, provide information on the distance between $g^*_{ij}$ and $g_{ij}$, which we refer to as the corruption level of edge $ij$. 
Our proposed Cycle-Edge Message Passing (CEMP) algorithm thus estimates these corruption levels using the cycle inconsistencies by 
alternatingly updating messages between cycles and edges. The edges with high corruption levels can then be confidently removed.

In theory, the latter cleaning (or removal) procedure can be used for recovering the original group elements in the noiseless case and for nearly recovering them in the case of sufficiently small noise. In fact, we obtain the strongest theoretical guarantees for general group synchronization with adversarial corruption.  

In practice, Section \ref{sec:postproc} suggests methods of using the estimated corruption levels for solving the group synchronization problem in general scenarios.

The basic idea of this work was first sketched for the different problem of camera location estimation in a conference paper \cite{AAB} (we explain this problem later in Section \ref{sec:energy}).
In addition to formulating this idea to the general group synchronization problem as well as carefully explaining it in the context of message passing, we present nontrivial theoretical guarantees, unlike the very basic and limited ones in \cite{AAB}. Most importantly, we establish exact and fast recovery of the underlying group elements.

\subsection{Contribution of This Work}
\label{sec:contribution}
The following are the main contributions of this work:\\
\textbf{New insight into group synchronization:} We mathematically establish the relevance of cycle consistency information to the group synchronization problem (see Section~\ref{sec:reform}).\\
\textbf{Unified framework via message passing:} 
CEMP 
applies to any compact group. This is due to the careful incorporation of cycle consistency, which is a general property of groups. As later explained in Section~\ref{sec:compare}, our algorithm is different from all previous message passing approaches, and in particular, does not require assumptions on the underlying joint distributions. \\
\textbf{Strongest theory for adversarial Corruption:} We claim that CEMP is the first algorithm that is guaranteed to exactly recover group elements from adversarially corrupted group ratios under reasonable assumptions (see Section \ref{sec:determin}).
Previous guarantees for group synchronization assume very special generative models and often asymptotic scenarios and special groups.
We are only aware of somewhat similar guarantees in \cite{HandLV15,truncatedLS,LUDrecovery}, but for the different problem of camera location estimation.
We claim that our theory is stronger since it only requires a constant uniform upper bound on the local corruption levels, whereas a similar upper bound in \cite{HandLV15,LUDrecovery} depends on $n$ and the sparsity of the graph. Moreover, our argument is much simpler than \cite{HandLV15,LUDrecovery} and we also need not assume the restrictive Erd\H{o}s-R\'{e}nyi model for generating the graph. 
While \cite{HandLV15,LUDrecovery} suggest a constructive solution and we only estimate the corruption levels, the guarantees of \cite{HandLV15,LUDrecovery} only hold for the noiseless case, and in this case correct estimation of corruption by our method is equivalent with correct solution of the group elements. \\
\textbf{Stability to noise:} We establish results for approximate recovery of CEMP in the presence of both adversarial corruption and noise (see Sections \ref{sec:bounded} and \ref{sec:subg}). For sub-Gaussian noise, we only require that the distribution of $d_\mathcal  G(g_{ij}, g_{ij}^*)$ is independent and sub-Gaussian, unlike previous specific noise distribution assumptions of $g_{ij}$ \cite{deepti,AMP_compact}. For the case where $d_\mathcal  G(g_{ij}, g_{ij}^*)$ is bounded for $ij\in E_g$, we state a deterministic perturbation result.\\
\textbf{Recovery under uniform corruption}:
When the edges in $G([n],E)$ are generated by the Erd\H{o}s-R\'{e}nyi model and the corrupted group ratios are i.i.d.~sampled from the Haar measure on $\mathcal G$, we can guarantee with high probability exact recovery and fast convergence by CEMP for any fixed corruption probability, $0 \leq q < 1$, and any edge connection probability, $0 < p \leq 1$, as long as the sample size is sufficiently large. 
Our analysis is not restricted anymore by a sufficiently small uniform upper bound on the local corruption levels.
Using these results, we derive sample complexity bounds for CEMP with respect to common groups. We point at a gap between these bounds and the information-theoretic ones as well as ones for other algorithms. Nevertheless, to the best of our knowledge, there are no other results for continuous groups that hold for any $q < 1$.  

\subsection{Organization of the Paper}
\label{sec:rest}
Section \ref{sec:previous} gives an overview of previous relevant works. 
Section \ref{sec:reform} mathematically establishes the relevance of cycle-based information to the solution of group synchronization.
Section \ref{sec:cemp} describes our proposed method, CEMP, and carefully interprets it as a message passing algorithm.
Section \ref{sec:theory} establishes exact recovery and fast convergence for CEMP under the adversarial corruption model and shows the stability of CEMP under bounded and sub-Gaussian noise.  Section \ref{sec:uniform} establishes guarantees under a special random corruption model. Section \ref{sec:experiment} demonstrates the numerical performance of CEMP using artificial datasets generated according to either adversarial or uniform corruption models.
Section \ref{sec:conclusion} concludes this work, while discussing possible extensions of it. The appendix contains various proofs and technical details, where the central ideas of the proofs are in the main text.

\section{Related Works}
\label{sec:previous}
This section reviews existing algorithms and guarantees for group synchronization and also reviews methods that share similarity with the proposed approach. Section~\ref{sec:energy} overviews previous works that utilize energy minimization and formulates a general  framework for these works. Section~\ref{sec:inference} reviews previous methods for inferring corruption in special group synchronization problems by the use of cycle-consistency information. Section~\ref{sec:message passing} reviews message passing algorithms and their applications to group synchronization.  

\subsection{Energy Minimization}\label{sec:energy}
Most works on group synchronization require minimizing an energy function. We first describe a general energy minimization framework for group synchronization and then review relevant previous works. This framework uses a metric $d_\mathcal G$ defined on $\mathcal G$
and a function $\rho$ from $\mathbb R_+^{|E|}$ to $\mathbb R_+$. We remark that $\mathbb R_+$ denotes the set of nonnegative numbers and $|\cdot|$ denotes the cardinality of a set. The general framework aims to solve
\begin{align}\label{eq:gs}
\min_{g_i\in \mathcal G} \rho \left(\left( d_\mathcal G(g_{ij},g_ig_j^{-1})\right)_{ij\in E} \right).
\end{align}
Natural examples of $\rho$ include the sum of $p$th powers of elements, where $p>0$, the number of non-zero elements, and the maximal element. 

The elements of $\mathbb Z_2$, $S_m$ and $SO(d)$ (that is, the most common groups that arise in synchronization problems), can be represented by orthogonal matrices with sizes  $N=1$, $m$, and $d$, respectively. For these groups,  it is common to identify each $g_i$, $i \in [n]$, with its representing matrix, choose $d_\mathcal G$ as the Frobenius norm of the difference of two group elements (that is, their representing matrices), $\rho(\cdot)= \| \cdot \|_{\nu}^{\nu}$, where $\nu=2$ or $\nu=1$, and consider the following minimization problem
\begin{align}
\label{eq:frob}
\min_{g_i\in \mathcal G} \sum_{ij\in E}  \|g_ig_j^{-1}-g_{ij}\|_F^\nu.
\end{align}

The best choice of $\nu$ depends on $\mathcal G$, the underlying noise and the corruption model. For Lie groups, $\nu=2$ is optimal under Guassian noise, and $\nu=1$ is more robust to outliers (i.e., robust to significantly corrupted group ratios). For some examples of discrete groups, such as $\mathbb Z_2$ and $S_N$, $\nu=2$ is information-theoretically optimal for both Gaussian noise and uniform corruption.

For $\nu=2$, one can form an equivalent formulation of \eqref{eq:frob}. It uses the block matrix $\bY\in \mathbb R^{nN\times nN}$, where $\bY_{ij}=g_{ij}$ if $ij\in E$ and $\bY_{ij}=\b0_{N\times N}$, otherwise. Its solution is a block matrix $\bX \in \mathbb R^{nN\times nN}$, whose $[i,j]$-th block, $i,j \in [n]$, is denoted by $\bX_{ij}$.  
It needs to satisfy $\bX_{ij} = g_ig_j^T$, where $\{g_i\}_{i \in [n]}$, is the solution of \eqref{eq:frob}, or equivalently, $\bX=\bx \bx^T$, where $\bx=(g_i)_{i\in [n]}\in \mathbb R^{nN\times N}$. In order to obtain this, $\bX$ needs to be positive semi-definite of rank $N$ and its blocks need to represent elements of $\mathcal G$, where the diagonal ones need to be the identity matrices. 
For $SO(2)$, it is more convenient to represent $g_i$ and $g_{ij}$ by elements of $U(1)$, the unit circle in $\mathbb C$, and thus replace $\mathbb R^{2n\times 2}$ and $\mathbb R^{2n\times 2n}$ with $\mathbb C^{n\times 1}$ and $\mathbb C^{n\times n}$. 
Using these components, the equivalent formulation can be written as
\begin{equation}\label{eq:SDR}
\begin{array} { c l } { \underset { \bX \in \mathbb { R } ^ { nN \times nN }}  { \operatorname { max } } } & { \operatorname { Tr } ( \bX^T\bY) } \\ { \text { subject to } } & {\{\bX_{ij}\}_{i,j=1}^n\subset \mathcal G}\\
{ } & { \bX _ { i i } = \bI_{N\times N} , i = 1 , \ldots , n } \\ { } & { \bX \succeq \b0 }
 \\ { } & { \text{rank}\,(\bX)=N.} \end{array}
\end{equation}

The above formulation is commonly relaxed by removing its two nonconvex constraints: rank$(\bX)=N$ and $\{\bX_{ij}\}_{i,j=1}^n\subset \mathcal G$.
The solution $\hat\bX$ of this relaxed formulation can be found by an SDP solver. One then commonly computes its top $N$ eigenvectors and stacks them as columns to obtain the $n \times 1$ vector of $N \times N$ blocks, $\tilde{\bx}$ (note that $\tilde{\bx}\tilde{\bx}^T$ is the best rank-N approximation of $\hat\bX$ in Frobenius norm). 
Next, one projects each of the $N$ blocks of $\tilde{\bx}$ (of size $N\times N$)  onto $\mathcal G$. This whole procedure, which we refer to in short as SDP, is typically slow to implement \cite{singer2011angular}. A faster common method, which we refer to as Spectral, applies a similar procedure while ignoring all constraints in \eqref{eq:SDR}. In this case, the highly relaxed solution of \eqref{eq:SDR} is $\hat\bX: = \bY$ and one only needs to find its top $N$ eigenvectors and project their blocks on the group elements \cite{singer2011angular}.

The formulation \eqref{eq:SDR} and its SDP  relaxation first appeared in the celebrated work of Goemans and Williamson  \cite{Goemans_Williamson_1995} on the max-cut problem. Their work can be viewed as a formulation for solving $\mathbb Z_2$ synchronization. Amit Singer \cite{singer2011angular} proposed the generalized formulation and its relaxed solutions for group synchronization, in particular, for angular synchronization. 

The exact recovery for $\mathbb Z_2$ synchronization is studied in \cite{Z2Afonso2,Z2Afonso}  by assuming an Erd\H{o}s-R\'{e}nyi graph, where each edge is independently corrupted with probability  $q<1/2$. Abbe et al.~\cite{Z2Afonso2} specified an information-theoretic lower bound on the average degree of the graph in terms of $q$. 
Bandeira~\cite{Z2Afonso} established asymptotic exact recovery for SDP for $\mathbb{Z}_2$ synchronization
w.h.p.~(with high probability) under the above information-theoretic regime. 
Montanari and Sen \cite{STOC_Montanari} studied the detection of good edges, instead of their recovery, under  i.i.d.~additive Gaussian noise.

Asymptotic exact recovery for convex relaxation methods of permutation synchronization appears in \cite{chen_partial,deepti}. In \cite{deepti}, noise is added to the relative permutations in $S_N$. The permutations are represented by $N \times N$ matrices and the elements of the additive $N \times N$ noise matrix are i.i.d.~$N(0,\eta^2)$. 
In this setting, exact recovery can be guaranteed when $\eta^2<{(n/N)}/{(1+4(n/N)^{-1})}$ as $nN\to \infty$. An SDP relaxation, different from~\eqref{eq:SDR}, is proposed in \cite{chen_partial,Huang13}. It is shown in \cite{Huang13} that for fixed $N$ and probability of corruption less than $0.5$, their method exactly recovers the underlying permutations w.h.p.~as $n\to\infty$. We remark that \cite{Huang13} assumes element-wise corruption of permutation matrices which is different from ours. An improved theoretical result is given by Chen et al.~\cite{chen_partial}, which matches the information-theoretic bound.

Rotation synchronization has been extensively studied \cite{Nachimson_LS,ChatterjeeG13_rotation,Govindu04_Lie,HartleyAT11_rotation,MartinecP07_rotation,wang2013exact}. In order to deal with corruption, it is most common to use $\ell_1$ energy minimization \cite{ChatterjeeG13_rotation,HartleyAT11_rotation,wang2013exact}. For example, Wang and Singer formulated a robust $SO(d)$ synchronization, for any $d \geq 2$, as the solution of \eqref{eq:frob} with $\nu=1$ and $\mathcal G = SO(d)$. 
Inspired by the analysis of \cite{ZhangL14_novel,LMTZ2014}, 
they established asymptotic and probabilistic exact recovery by the solution of their minimization problem under the following very special probabilistic model: The graph is complete or even Erd\H{o}s-R\'{e}nyi, the corruption model for edges is Bernoulli with corruption probability less than a critical probability $p_c$ that depends on $d$, and the corrupted rotations are i.i.d.~sampled from the Haar distribution on $SO(d)$. They proposed an alternating direction augmented Lagrangian method for practically solving their formulation, but their analysis only applies to the pure minimizer.

A somewhat similar problem to group synchronization is camera location estimation \cite{HandLV15,cvprOzyesilS15,OzyesilSB15_SDR,AAB}. It uses the non-compact group $\mathbb{R}^3$ with vector addition and its input includes possibly corrupted measurements of $\{Tg_{ij}^*\}_{ij\in E}$, where $T(g_{ij}^*)=g_{ij}^*/\|g_{ij}^*\|$ and $\|\cdot\|$ denotes the Euclidean norm.
The application of $T$ distorts the group structure and may result in loss of information. 

For this problem other forms of energy minimization have been proposed, which often differ from the framework in \eqref{eq:frob}. 
The first exact recovery result for a specific energy-minimization algorithm was established by Hand, Lee and Voroniski \cite{HandLV15}. The significance of this work is in the   weak assumptions of the corruption model,
whereas in the previously mentioned works on exact recovery \cite{Z2Afonso2,Z2Afonso,Z2,Huang13,deepti,wang2013exact}, the corrupted group ratios followed very specific probability distributions. 
More specifically, the main model in \cite{HandLV15} assumed an Erd\H{o}s-R\'{e}nyi graph $G([n],E)$ with parameter $p$ for connecting edges and an arbitrary corrupted set of edges $E_b$, whose corruption is quantified by the maximal degree of $G([n],E_b)$ divided by $n$, which is denoted by $\epsilon_b$. The transformed group ratios, $T(g_{ij})$, are $T(g_{ij}^*)$ for $ij \in E_g$ and  are arbitrarily chosen in $S^2$, the unit sphere, for  $ij \in E_b$. 
They established exact recovery under this model with $\epsilon_b=O(p^5/\log^3 n)$. A similar exact recovery theory for another energy-minimization algorithm, namely the Least Unsquared Deviations (LUD) \cite{cvprOzyesilS15}, was established by Lerman, Shi and Zhang \cite{LUDrecovery}, but with the stronger corruption bound, $\epsilon_b=O(p^{7/3}/\log^{9/2} n)$. 

Huang et al.~\cite{truncatedLS} solved an $\ell_1$ formulation for 1D translation synchronization, where $\mathcal G=\mathbb R$ with regular addition. They proposed a special version of IRLS and provided a deterministic exact recovery guarantee that depends on $\epsilon_b$ and a quantity that uses the graph Laplacian.

\subsection{Synchronization Methods Based on Cycle Consistency} \label{sec:inference}
Previous methods that use the cycle consistency constraint in \eqref{eq:cyclecons} only focus on synchronizing camera rotations. Additional methods use a different cycle consistency constraint to synchronize camera locations.  Assuming that $\mathcal G$ lies in a metric space with a metric $d_{\mathcal G}(\cdot\,,\cdot)$, the corruption level in a cycle $L$ can be indicated by the cycle inconsistency measure $
d_{\mathcal G}(g_L\,, e_{\mathcal G})
$, where $g_L$ was defined in \eqref{eq:defgL}.
 There exist few works that exploit such information to identify and remove the corrupted edges. A likelihood-based method \cite{Zach2010} was proposed to classify the corrupted and uncorrupted edges (relative camera motion) from  observations $d_{\mathcal G}(g_L\,, e_{\mathcal G})$ of many sampled $L$'s. 
This work has no theoretical guarantees. It seeks to solve the following problem: 
\begin{equation}\label{eq:bp}
\max_{x_{ij}\in \{0,1\}}\prod_{L}\Pr\left(\{x_{ij}\}_{ij\in L}|d_{\mathcal G}(g_L\,, e_{\mathcal G})\right).
\end{equation}
The variables $\{x_{ij}\}_{ij \in E}$ provide the assignment of edge $ij$ in the sense that $x_{ij}=\mathbf 1_{\{ij\in E_g\}}$, where $\mathbf 1$ denotes the indicator function. 
One of the proposed solutions in \cite{Zach2010} is a linear programming relaxation of \eqref{eq:bp}. The other proposed solution of \eqref{eq:bp} uses belief propagation. It is completely different from the message passing approach proposed in this work.

Shen et al.~\cite{shen2016} finds a cleaner subset of edges by searching for consistent cycles. In particular, if a cycle $L$ of length $m$ satisfies $d_{\mathcal G}(g_L\,, e)<\epsilon/\sqrt{m}$, then all the edges in the cycle are treated as uncorrupted. However, this approach lacks any theoretical guarantees and may fail in various cases. For example, the case where edges are maliciously corrupted and some  cycles with corrupted edges satisfy $d_{\mathcal G}(g_L\,, e)<\epsilon/\sqrt{m}$.

An iterative reweighting strategy, referred to as IR-AAB, was proposed in \cite{AAB} to identify corrupted pairwise directions when estimating camera locations. Experiments on synthetic data showed that IR-AAB was able to detect exactly the set of corrupted pairwise directions that were uniformly distributed on $S^2$ with low or medium corruption rate. However, this strategy was only restricted to camera location estimation and no exact recovery guarantees were provided for the reweighting algorithm. We remark that our current work is a generalization of \cite{AAB} to compact group synchronization problems. We also provide a message-passing interpretation for the ideas of \cite{AAB} and stronger mathematical guarantees in our context, but we do not address here the camera location estimation problem.

\subsection{Message Passing Algorithms}\label{sec:message passing}
  Message passing algorithms are efficient methods for statistical inference on graphical models. The most famous message passing algorithm is belief propagation (BP)~\cite{BP}. It is an efficient algorithm for solving marginal distribution or maximizing the joint probability density of a set of random variables that are defined on a Bayesian network. The joint density and the corresponding Bayesian network can be uniquely described by a  factor graph that encodes the dependencies of factors on the random variables. In particular, each factor is considered as a function of a small subset of random variables and the joint density is assumed as the product of these factors. The BP algorithm passes messages between the random variables and factors in the factor graph. When the factor graph is a tree, then BP is equivalent to dynamic programming and can converge in finite iterations. However, when the factor graph contains loops, BP has no guarantee of convergence and accuracy.  The BP algorithm is applied in \cite{Zach2010} to solve the maximal likelihood problem
  \eqref{eq:bp}. However, since the factor graph defined in \cite{Zach2010} contains many loops, there are no convergence and accuracy guarantees of the solution.

Another famous class of message passing algorithms is  approximate message passing (AMP) \cite{AMP_Donoho,AMP_compact}. AMP can be viewed as a modified version of BP and it is also used to compute marginal distribution and maximal likelihood. The main advantage of AMP over BP is that it enjoys asymptotic convergence guarantees even on loopy factor graphs. AMP was first proposed by Donoho, Maleki, and Montanari~\cite{AMP_Donoho} to solve the compressed sensing problem. They formulated the convex program for this problem as a maximal likelihood estimation problem and then solved it by AMP. Perry et al.~\cite{AMP_compact} applies AMP to group synchronization over any compact group. However, they have no corruption and only assume additive i.i.d.~Gaussian noise model, where they seek an asymptotic solution that is statistically optimal. 

  Another message passing algorithm \cite{Z2} was proposed for $\mathbb{Z}_2$ synchronization. It assigns probabilities of correct labeling to each node and each edge. These probabilities are iteratively passed and updated between nodes and edges until convergence. There are several drawbacks of this method. First of all, it cannot be generalized to other group synchronization problems. Second, its performance is worse than  SDP under high corruption \cite{Z2}. At last, no theoretical guarantee of exact recovery is established. We remark that this method is completely different from the method proposed here.

\section{Cycle Consistency is Essential for Group Synchronization}\label{sec:reform}
In this section, we establish a fundamental relationship between cycle consistency and group synchronization, while assuming the noiseless case.
We recall that $d_{\mathcal G}$ is a  bi-invariant metric on $G$ and that the diameter of $\mathcal G$ is 1, that is, $d_{\mathcal G}(\cdot\,,\cdot)\leq 1$.

Although the ultimate goal of this paper is to estimate group elements $\{g_i^*\}_{i\in [n]}$ from group ratios $\{g_{ij}\}_{ij\in E}$, we primarily focus on a variant of such a task. That is, estimating the corruption level
\begin{equation}
\label{eq:def_sij*}
s_{ij}^*= d_\mathcal G(g_{ij},g_{ij}^*),\quad ij\in E,
\end{equation}
from the cycle-inconsistency measure
\begin{equation}
\label{eq:def_dL}
d_L = d_{\mathcal G}(g_L, e_\mathcal G), \quad L\in \mathcal{C},
\end{equation}
where $\mathcal C$ is a set of cycles that are either randomly sampled or deterministically selected.
 We remark that in our setting, exact estimation of $\{s_{ij}^*\}_{ij\in E}$ is equivalent to exact recovery of $\{g_i^*\}_{i\in [n]}$. Proposition~\ref{prop:equiv}, which is proved in Appendix \ref{sec:prop_equiv}, clarifies this point. In Section~\ref{sec:postproc}
we discuss how to practically use $\{s_{ij}^*\}_{ij\in E}$ to infer $\{g_i^*\}_{i\in [n]}$ in more general settings.

\begin{proposition}\label{prop:equiv}
Assume data generated by the noiseless adversarial corruption model, where $G([n], E_g)$ is connected. Then the following problems are equivalent
\begin{enumerate}
    \item\label{item:Eg} Exact recovery of $E_g$;
    \item\label{item:sij} Exact estimation of  $\{s_{ij}^*\}_{ij\in E}$;
    \item\label{item:gi} Exact recovery of $\{g_i^*\}_{i\in [n]}$.
\end{enumerate}    
\end{proposition}

We first remark that, in practice, shorter cycles are preferable due to faster implementation and less uncertainties~\cite{Zach2010}, and thus when establishing the theory for CEMP in Sections \ref{sec:theory} and \ref{sec:uniform} we let $\mathcal C$ be the set of 3-cycles $\mathcal C_3$. However, we currently leave the general notation as our work extends to the more general case.

We further remark that for corruption estimation, only the set of real numbers $\{d_L\}_{L\in \mathcal C}$ is needed, which is simpler than the set of given group ratios $\{g_{ij}\}_{ij\in E}$. This may enhance the underlying statistical inference. 

We next explain why cycle-consistency information is essential for solving the problems of corruption estimation and group synchronization. 
Section \ref{sec:wp}
shows that under a certain condition the set of cycle-inconsistency measures, $\{d_L\}_{L\in\mathcal C}$, provides sufficient information for recovering corruption levels. Section \ref{sec:cc} shows that cycle consistency is closely related to group synchronization and plays a central role in its solution. It further explains that many previous works implicitly exploit cycle consistency information.
\subsection{Exact Recovery Relies on a Good-Cycle Condition}\label{sec:wp}
In general, it is not obvious that the set $\{d_L\}_{L\in \mathcal C}$ contains sufficient information for recovering $\{s_{ij}^*\}_{ij \in E}$. Indeed, the former set generally contains less information than the original input of our problem, $\{g_{ij}\}_{ij\in E}$. 
Nevertheless, Proposition \ref{prop:good cycle} implies that if 
every edge is contained in a good cycle (see formal definition below), then $\{d_L\}_{L\in \mathcal C}$ actually contains the set  $\{s_{ij}^*\}_{ij\in E}$. 
\begin{definition}[Good-Cycle Condition]\label{def:good cycle}
$G([n], E)$, $E_g$ and $\mathcal C$ satisfy the good-cycle condition if
for each $ij\in E$, there exists at least one cycle $L\in \mathcal C$ containing $ij$ such that $L\setminus \{ij\}\subseteq E_g$.
\end{definition}
\begin{proposition}\label{prop:good cycle}
Assume data generated by the noiseless adversarial corruption model, satisfying the good-cycle condition. Then, $s_{ij}^*=d_L$ $\forall ij \in E$,  $L \in \mathcal C$ such that $ij \in L$ and $L\setminus \{ij\}\subseteq E_g$.
\end{proposition}
\begin{proof}
Fix $ij \in E$ and let  $L=\{ij,jk_1,k_1k_2,k_2k_3,\dots,k_mi\} \ni ij$ be a good cycle, i.e., $L\setminus \{ij\}\subseteq E_g$. Applying the definitions of $d_L$ and then $g_L$, next right multiplying with $g_{ij}^*$ while using the bi-invariance of $d_\mathcal G$, then applying  \eqref{eq:cyclecons} and at last using the definition of $s_{ij}^*$, yield 
\begin{multline*}
    d_L=d_\mathcal G(g_L,e_\mathcal G )=d_\mathcal G(g_{ij}g_{jk_1}^*\cdots g_{k_mi}^*,e_\mathcal G )\\ = d_\mathcal G(g_{ij}g_{jk_1}^* \cdots g_{k_mi}^*g_{ij}^*,g_{ij}^* )
=d_\mathcal G(g_{ij},g_{ij}^*)=s_{ij}^*.
\end{multline*}
\qed
\end{proof}

We formulate a stronger
quantitative version of Proposition~\ref{prop:good cycle}, which we frequently use in 
establishing our exact recovery theory. We prove it in Appendix~\ref{sec:appbi}.
\begin{lemma}\label{lemma:bi}
For all $ij\in E$ and any cycle 
$L$ containing  $ij$ in 
$G([n], E)$,
\begin{align*}
|d_L-s_{ij}^*|\leq \sum_{ab\in L\setminus\{ij\}}s_{ab}^*.
\end{align*}
\end{lemma} 
 
\subsection{A Natural Mapping of Group Elements onto Cycle-Consistent Ratios}\label{sec:cc}
Another reason for exploiting the cycle consistency constraint \eqref{eq:cyclecons} is its crucial connection to group synchronization. Before stating the relationship clearly, we define the following notation.

Denote by $(g_i)_{i\in [n]}\in \mathcal G^n$ and $(g_{ij})_{ij\in E}\in \mathcal G^{|E|}$ the elements of the product spaces $\mathcal{G}^{n}$ and $\mathcal{G}^{|E|}$, respectively.
We say that $(g_i)_{i\in [n]}$ and $(g_i')_{i\in [n]}$ are equivalent, which we denote by $(g_i)_{i\in [n]}\sim (g_i')_{i\in [n]}$, if there exists $g_0\in\mathcal G$ such that $g_i=g_i'g_0$ for all $i\in [n]$. This relationship induces an equivalence class $[(g_i)_{i\in [n]}]$ for each $(g_i)_{i\in [n]}\in \mathcal G^n$. In other words, each $[(g_i)_{i\in [n]}]$ is an element of the quotient space $\eqc{\mathcal G^n }$.
We define the set of cycle-consistent $(g_{ij})_{ij\in E}$ with respect to $\mathcal C$ by \[\mathcal G_\mathcal C=\{(g_{ij})_{ij\in E} \in \mathcal G^{|E|}: g_L=e_{\mathcal G}, \forall L\in \mathcal C\}.\]

The following proposition demonstrates a bijection between the group elements and cycle-consistent group ratios. Its proof is included in Appendix \ref{sec:proofs_bijection}.

\begin{proposition}\label{prop:cycle}
Assume that $G([n],E)$ is connected and any $ij\in E$ is contained in at least one cycle in $\mathcal C$.  Then, $h: \eqc{\mathcal G^n} \rightarrow\mathcal G_\mathcal C$ defined by $h([(g_{i})_{i\in [n]}])=(g_ig_j^{-1})_{ij\in E}$ is a bijection. 
\end{proposition}

\begin{remark}
The function $f$ is an isomorphism, that is, 
$$h([(g_{i})_{i\in [n]}]\cdot[(g_{i}')_{i\in [n]}])=(g_ig_j^{-1})_{ij\in E}\cdot (g_i'g_j^{\prime -1})_{ij\in E}, $$
if and only if $\mathcal G$ is Abelian.
Indeed, if $\mathcal G$ is Abelian the above equation is obvious. If the above equation holds $\forall   (g_i)_{i\in [n]}$, 
$(g_i')_{i\in [n]}\in \mathcal G^n$, 
then 
$g_ig_i'g_j^{\prime -1}g_j^{-1}=g_ig_j^{-1}g_i'g_j^{\prime -1}$  
$\, \forall (g_i)_{i\in [n]}$, 
$(g_i')_{i\in [n]}\in \mathcal G^n$. 
Letting $g_i=g_j'=e_\mathcal G$ yields that $g_i'g_j^{-1}=g_j^{-1}g_i'$ $\, \forall g_i', g_j \in \mathcal G$, and thus $\mathcal G$ is Abelian.

\end{remark}

\begin{remark}
The condition on $\mathcal C$ of Proposition \ref{prop:cycle} holds under the good-cycle condition.
\end{remark}

This proposition signifies that previous works on group synchronization implicitly enforce cycle consistency information. Indeed, consider the formulation in \eqref{eq:gs} that searches for $(g_{i})_{i\in [n]} \in \mathcal G^n$ (more precisely, $[(g_{i})_{i\in [n]}] \in \eqc{\mathcal G^n}$) that minimize a function of $\{ d_\mathcal G(g_{ij},g_ig_j^{-1}) \}_{ij\in E}$. In view of the explicit expression for the bijection $f$
in Proposition \ref{prop:cycle}, this is equivalent to finding the closest cycle-consistent group ratios $(g_{ij}')_{ij\in E}\in \mathcal G_\mathcal C$ to the given group ratios   $(g_{ij})_{ij\in E}$. However, direct solutions of \eqref{eq:gs} are hard and proposed algorithms often relax the original minimization problem and thus their relationship with cycle-consistent group ratios may not be clear.
A special case that may further demonstrate the implicit use of cycle-consistency in group synchronization is when using $\rho(\cdot)=\|\cdot\|_0$ (that is, $\rho$ is the number of non-zero elements) in \eqref{eq:gs}. 
We note that this formulation asks  to minimize among ${g_i\in \mathcal G}$ the number of non-zero elements in $(d_\mathcal G(g_{ij},g_ig_j^{-1}))_{ij\in E}$.
By Proposition~\ref{prop:cycle}, it is equivalent to minimizing among $(g'_{ij})_{ij\in E} \in \mathcal{G}_\mathcal{C}$ the number of elements in $\{ij \in E: g_{ij} \neq g'_{ij}\}$, or similarly, maximizing the number of elements in $\{ij \in E: g_{ij} = g'_{ij}\}$.
Thus the problem can be formulated as finding the maximal $E'\subseteq E$ such that $\{g_{ij}\}_{ij\in E'}$ is cycle-consistent. If the maximal set is $E_g$, which makes the problem well-defined, then in view of Proposition~\ref{prop:equiv}, its recovery is equivalent with exact recovery of $\{s_{ij}^*\}_{ij\in E}$.

\section{Cycle-Edge Message Passing (CEMP)}\label{sec:cemp}
We describe CEMP and explain the underlying statistical model that motivates the algorithm.  Section~\ref{sec:CEG} defines the cycle-edge graph (CEG) that will be used to describe the message passing procedure. Section~\ref{sec:des} describes CEMP and discusses at length its interpretation and some of its properties. Section~\ref{sec:compare} compares CEMP with BP, AMP and IRLS.
\subsection{Cycle-Edge Graph}
\label{sec:CEG}
We define the notion of a cycle-edge graph (CEG), which is analogous to the factor graph in belief propagation. We also demonstrate it in Figure \ref{fig:ceg}. Given the graph $G([n], E)$ and a set of cycles $\mathcal C$, the corresponding cycle-edge graph $G_{CE}(V_{CE}, E_{CE})$ is formed in the following way.
\begin{enumerate}
\item The set of vertices in $G_{CE}$ is $V_{CE}=\mathcal C\cup E$. All $L\in \mathcal C$ are called cycle nodes and all $ij\in E$ are called edge nodes.
\item 
$G_{CE}$ is a bipartite graph, where
the set of edges in $G_{CE}$ is all the pairs $(ij,L)$ such that $ij\in L$ in the original graph $G([n], E)$.
\end{enumerate}

\begin{figure}[htbp]
\begin{center}
   \includegraphics[width=0.9\linewidth]{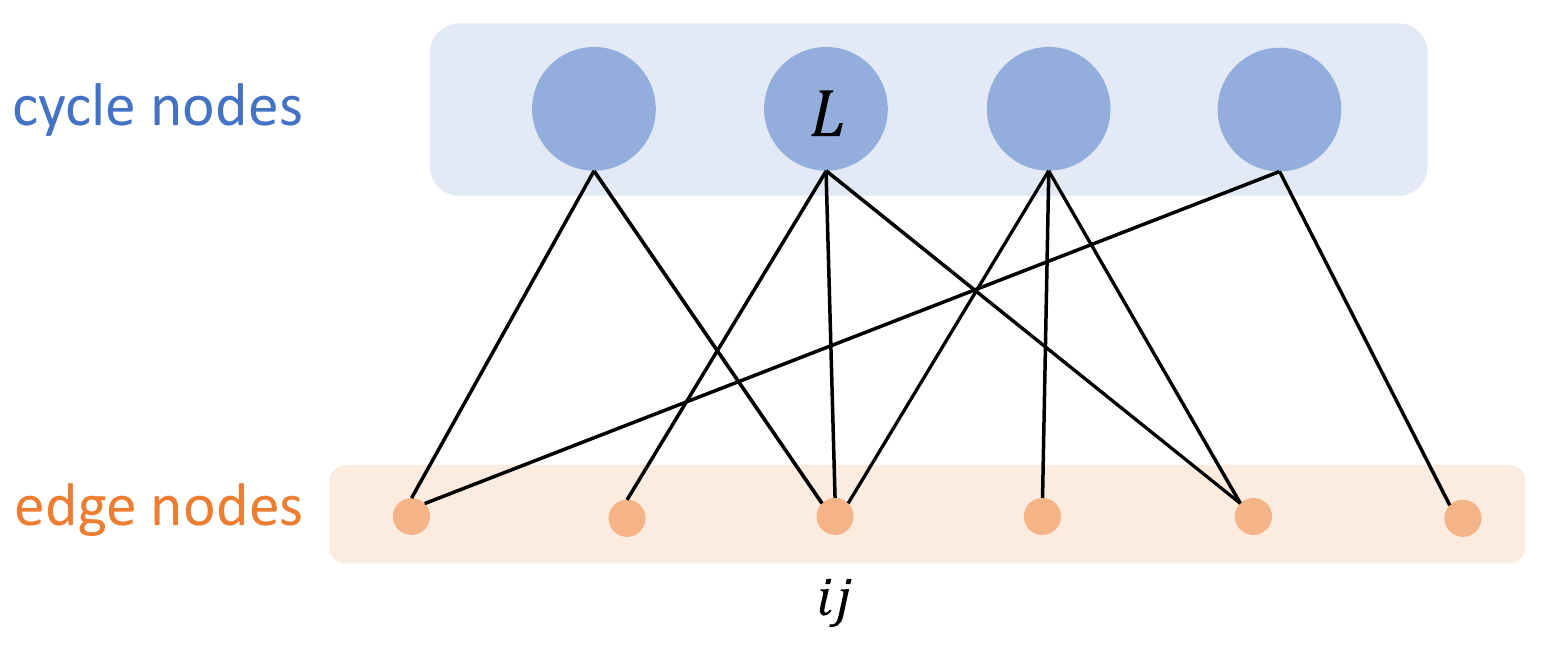}
\end{center}
   \caption{An illustration of the cycle-edge graph. Cycle node $L$ and edge node $ij$ are connected if $ij\in L$ in the original graph $G([n], E)$.\label{fig:ceg}}
\end{figure}

For each cycle node $L$ in $G_{CE}$, the set of its neighboring edge nodes in $G_{CE}$ is $N_L=\{ij\in E: ij\in L\}$. We can also describe it as the set of edges contained in $L$ in the original graph $G([n], E)$. We remark that we may treat edges and cycles as elements of either $G_{CE}$ or $G([n], E)$ depending on the context. 
For each edge node $ij$ in $G_{CE}$, the set of its neighboring cycle nodes in $G_{CE}$ is
$N_{ij} = \{L \in \mathcal C : ij\in L\}$. Equivalently, it it is the set of cycles containing  $ij$ in the original graph $G([n], E)$.

\subsection{Description of CEMP}\label{sec:des}
Given relative measurements $(g_{ij})_{ij\in E}$ with respect to a graph $G([n],E)$, the CEMP algorithm tries to estimate the corruption levels $s_{ij}^*$, $ij\in E$, defined in \eqref{eq:def_sij*} by using the inconsistency measures $d_L$, $L\in \mathcal C$, defined in \eqref{eq:def_dL}. It does it iteratively, where we denote by $s_{ij}(t)$ the estimate of $s_{ij}^*$ at iteration $t$. Algorithm~\ref{alg:cemp} sketches CEMP and Figure~\ref{fig:cemp_illustration} illustrates its main idea. We note that Algorithm~\ref{alg:cemp} has the following stages: 1) generation of CEG (which is described in Section \ref{sec:CEG}); 2) computation of the cycle inconsistency measures (see \eqref{eq:compute_cycle_con}); 3) corruption level initialization for message passing (see \eqref{eq:initialize_wij}); 4) message passing from edges to cycles (see \eqref{eq:from_edge_to_cycle}); and 5) message passing from cycles to edges (see \eqref{eq:from_cycle_to_edge}).

\begin{algorithm}[ht]
\caption{Cycle-Edge Message Passing (CEMP)}\label{alg:cemp}
\begin{algorithmic}
\REQUIRE graph $G([n],E)$, relative measurements $(g_{ij})_{ij\in E}$, choice of metric $d_{\mathcal G}$, the  set of sampled/selected cycles $\mathcal C$ (default: $\mathcal C=\mathcal C_3$), total time step $T$, increasing parameters $\{\beta_t\}_{t=0}^T$ (theoretical choices are discussed in Sections \ref{sec:theory}
and \ref{sec:uniform}), reweighting function 
\begin{equation}
\label{eq:choice_f}
f(x;\beta_t)= \mathbf{1}_{\{x\leq\frac{1}{\beta_t}\}} \ \text{ or } \ f(x;\beta_t)= e^{-\beta_t x}
\end{equation}
\STATE \textbf{Steps:}
\STATE Generate CEG from $G([n], E)$ and $\mathcal C$
     \FOR {$ij\in E$ and $L\in N_{ij}$}
\STATE  \begin{equation} d_L = d_{\mathcal G}(g_L, e_\mathcal G) \label{eq:compute_cycle_con} \end{equation}
\ENDFOR
\FOR {$ij\in E$}
\STATE \begin{equation} s_{ij}(0) = \frac{1}{|N_{ij}|} \sum_{L\in N_{ij}} d_L\label{eq:initialize_wij} \end{equation}
\ENDFOR
\FOR {$t=0:T$}
\FOR {$ij\in E$ and $L\in N_{ij}$}
\STATE \begin{equation} w_{ij,L}(t) = \frac{1}{Z_{ij}(t)}\prod_{ab\in N_L\setminus \{ij\}}f(s_{ab}(t);\beta_t), \quad Z_{ij}(t)=\sum_{L\in N_{ij}} \prod_{ab\in N_L\setminus \{ij\}}f(s_{ab}(t);\beta_t)
\label{eq:from_edge_to_cycle}
\end{equation}
\ENDFOR
\FOR {$ij\in E$}
\STATE \begin{equation} \label{eq:from_cycle_to_edge} s_{ij}(t+1)=\sum_{L\in N_{ij}} w_{ij,L}(t)d_L \end{equation}
\ENDFOR

\ENDFOR

\ENSURE $\left(s_{ij}(T)\right)_{ij\in E}$
\end{algorithmic}\label{alg:iraab}
\end{algorithm}

\begin{figure}[htbp]
\begin{center}
   \includegraphics[width=1\linewidth]{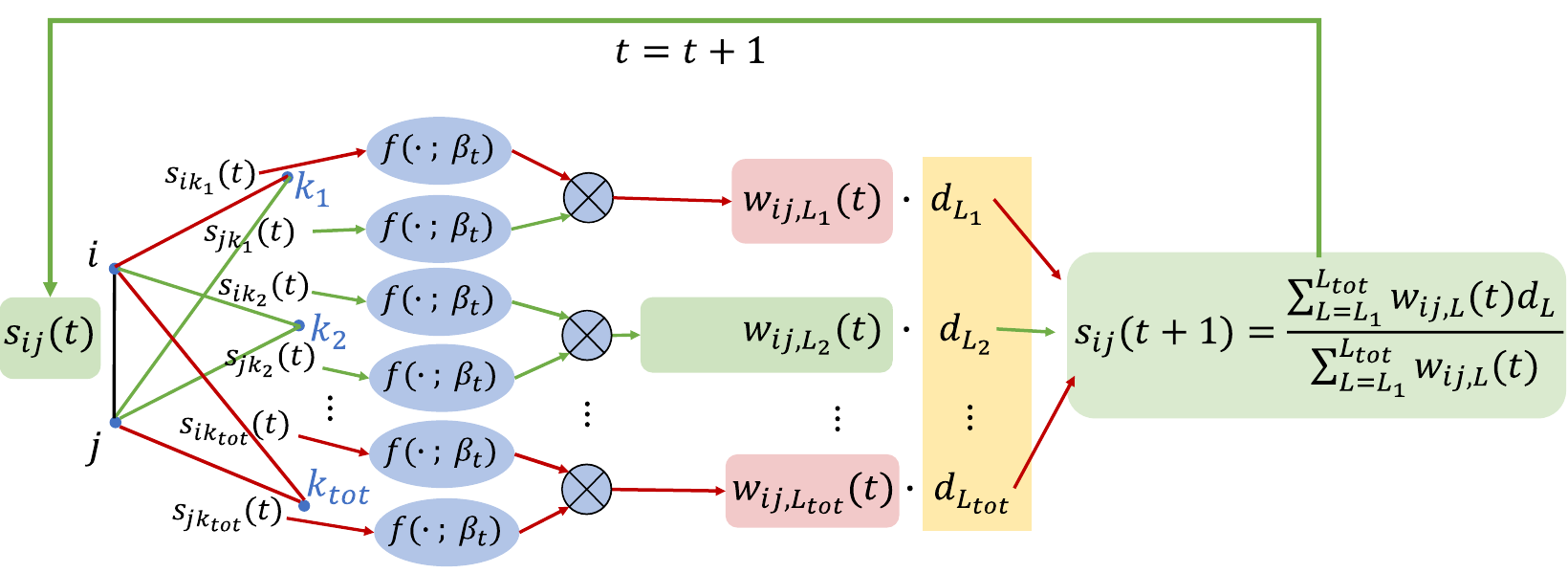}
\end{center}
   \caption{Illustration of CEMP with 3-cycles. The good and bad edges are respectively marked as green and red. For the  edge $ij$ and any 3-cycle $L = \{ij,jk,ki\}$, CEMP computes $w_{ij,L}(t)$, an 
   estimate for the probability that $L$ is good given estimates of the corruption levels for edges $ik$ and $ij$, $s_{ik}(t)$ and $s_{jk}(t)$. Note that $\{w_{ij,L}(t)\}_{L\in N_{ij}}$ (normalized so that its sum is 1) is a discrete distribution on the set of cycle inconsistencies, $\{d_{ij,L}\}_{L\in N_{ij}}$. This distribution aims to emphasize good cycles $L$ with respect to the edge $ij$. For example, for the good cycle w.r.t.~$ij$, $L_2=\{i,j,k_2\}$, both $f(s_{ik_2}(t),\beta_t)$ and $f(s_{jk_2}(t),\beta_t)$ are expected to be relatively high (since the edges $ik_2$ and $jk_2$ are good); thus the weight $w_{ij,L_2}(t)=f(s_{ik_2}(t),\beta_t) \cdot f(s_{jk_2}(t),\beta_t)$ is relatively high. On the other hand, $f(s_{ik_2}(t),\beta_t)$ is expected to be relatively low (since $ik_2$ is bad) and thus the weight $w_{ij,L_2}(t)$ is relatively low. At last, an estimate of $s_{ij}(t+1)$ is obtained by a
   weighted average that uses the weights $\{w_{ij,L}(t)\}_{L\in N_{ij}}$. 
     }\label{fig:cemp_illustration}
\end{figure}

The above first three steps of the algorithms are straightforward. In order to explain the last two steps we introduce some notation in Section \ref{sec:steps_notation}. Section \ref{sec:fifth} explains the fourth step of CEMP and for this purpose it introduces a statistical model. We emphasize that this model and its follow-up extensions are only used for clearer interpretation of CEMP, but are not used in our theoretical guarantees. 
Section \ref{sec:fourth_probabilistic} explains the fifth step of CEMP using this model with additional two assumptions. Section \ref{sec:summarizing_diagram} summarizes the basic insights about CEMP in a simple diagram.  Section \ref{sec:cemp_functions} interprets the use of two specific reweighting functions in view of the statistical model (while extending it).
Section \ref{sec:practical} explains why the exponential reweighting function is preferable in practice.
Section \ref{sec:computational} clarifies the computational complexity of CEMP. 
Section \ref{sec:postproc} explains how to post-process CEMP in order to recover the underlying group elements (and not just the corruption levels) in general settings.

We remark that we separate the fourth and fifth steps of CEMP for clarity of presentation, however, one may combine them 
using a single loop that computes for each $ij \in E$ 
\begin{align}\label{eq:update}
    s_{ij}(t+1)=\frac{\sum\limits_{L\in N_{ij}}\prod\limits_{ab\in N_L\setminus \{ij\}}f(s_{ab}(t);\beta_t) \cdot d_L}{\sum\limits_{L\in N_{ij}} \prod\limits_{ab\in N_L\setminus \{ij\}}f(s_{ab}(t);\beta_t)}.
\end{align}
For $\mathcal C=\mathcal C_3$, the update rule \eqref{eq:update} can be further simplified (see \eqref{eq:updateA} and \eqref{eq:updateB}).

\subsubsection{Notation}
\label{sec:steps_notation}
Let $D_{ij}=\{d_L:L\in N_{ij}\}$ denote the set of inconsistencies levels with respect to $ij$, $G_{ij}=\{L\in N_{ij}: N_L\setminus\{ij\}\subseteq E_g\}$ denote the set of ``good cycles'' with respect to $ij$, and $CI_{ij}=\{L\in N_{ij}: d_L=s_{ij}^*\}$ denote the set of cycles with correct information of corruption with respect to $ij$.

\subsubsection{Message Passing from Edges to Cycles  and a Statistical Model}
\label{sec:fifth}
Here we explain the fourth step of the algorithm, which estimates $w_{ij,L}(t)$ according to \eqref{eq:from_edge_to_cycle}. 
We remark that $Z_{ij}(t)$  is the normalization factor assuring that $\sum_{L\in N_{ij}}w_{ij,L}(t)=1$.

In order to better interpret our procedure, we propose a statistical model. We assume that $\{s_{ij}^*\}_{ij \in E}$ and $\{s_{ij}(t)\}_{ij \in E}$ are both i.i.d.~random variables and that for any $ij \in E$, $s_{ij}^{*}$ is independent of $s_{kl}(t)$ for $kl \neq ij \in E$. We further assume that  
\begin{equation}
\label{eq:pr_f}
\Pr(s_{ab}^*=0|s_{ab}(t)=x) = f(x;\beta_t).
\end{equation}
Unlike common message passing models, we do not need to specify other probabilities, such as joint densities.
In view of these assumptions,  \eqref{eq:from_edge_to_cycle} can be formally rewritten as
\begin{equation}\label{eq:wp}
w_{ij,L}(t) = \frac{1}{Z_{ij}(t)}\prod_{ab\in N_L\setminus \{ij\}}\Pr\left(s_{ab}^*=0\,|\,s_{ab}(t)\right) = \frac{1}{Z_{ij}(t)}\prod_{ab\in N_L\setminus \{ij\}} f(s_{ab}(t);\beta_t).
\end{equation}
We note that the choices for $f(x;\beta_t)$ in \eqref{eq:choice_f} lead to the following update rules:
\begin{align}
\text{Rule \ith:}&\quad w_{ij,L}(t)=\frac{1}{Z_{ij}(t)}\mathbf{1}_{\left\{\max\limits_{ab\in N_L\setminus\{ij\}}s_{ab}(t)\leq \frac{1}{\beta_t}\right\}}\label{eq:ith}\\
\text{Rule \er:}&\quad w_{ij,L}(t)=\frac{1}{Z_{ij}(t)}\exp\left(-\beta_t\sum\limits_{ab\in N_L\setminus\{ij\}}s_{ab}(t)\right).\label{eq:er}
\end{align}
 We refer to CEMP with rules \ith and \er as CEMP-A and CEMP-B, respectively. 

Given this statistical model, in particular, using the i.i.d.~property of $s_{ij}^*$ and $s_{ij}(t)$, the update rule~\eqref{eq:wp} can be rewritten as
\begin{align}
w_{ij,L}(t)&=\frac{1}{Z_{ij}(t)}\Pr\left(\left(s_{ab}^*\right)_{ab\in N_L\setminus \{ij\}}=\b0 \Big|\left(s_{ab}(t)\right)_{ab\in N_L\setminus \{ij\}}\right)\nonumber\\
&=\frac{1}{Z_{ij}(t)}\Pr\left(N_L\setminus \{ij\}\subseteq E_g \Big|\left(s_{ab}(t)\right)_{ab\in N_L\setminus \{ij\}}\right)\nonumber\\
&=\frac{1}{Z_{ij}(t)}\Pr\left(L\in G_{ij} \Big|\left(s_{ab}(t)\right)_{ab\in N_L\setminus \{ij\}}\right).\label{eq:wij}
\end{align}

Finally, we use the above new interpretation of the weights to demonstrate a natural fixed point of the update rules \eqref{eq:from_cycle_to_edge} and \eqref{eq:wij}. Theory of convergence to this fixed point is presented later in Section~\ref{sec:theory}.
We first note that \eqref{eq:from_cycle_to_edge} implies the following ideal weights for good approximation:
\begin{equation}w_{ij,L}^{*} = \frac{1}{|G_{ij}|}\mathbf 1_{\{L\in G_{ij}\}} \ \text{ for }  {ij\in E} \text{ and } L\in N_{ij}. \label{eq:wstar_def}\end{equation}
Indeed,
\begin{equation} \label{eq:wstar_after_def}
\sum_{L\in N_{ij}} w_{ij,L}^{*}d_L = \frac{1}{|G_{ij}|} \sum_{L\in N_{ij}} \mathbf 1_{\{L\in G_{ij}\}} d_L =
\frac{1}{|G_{ij}|} \sum_{L\in G_{ij}} d_L = \frac{1}{|G_{ij}|} \sum_{L\in G_{ij}} s_{ij}^* = s_{ij}^*,
\end{equation}
where the equality before last uses Proposition~\ref{prop:good cycle}. 
We further note that
\begin{equation}
w_{ij,L}^{*}=\frac{1}{Z_{ij}^{*}}\Pr\left(L\in G_{ij} \Big|\left(s^*_{ab}\right)_{ab\in N_L\setminus \{ij\}}\right).\label{eq:wij_star}
\end{equation}
This equation follows from the fact that the events $L\in G_{ij}$ and $s_{ab}^{*} = 0$  $\forall ab\in N_L\setminus \{ij\}$ coincide, and thus \eqref{eq:wstar_def} and \eqref{eq:wij_star} are equivalent, where
the normalization factor ${Z_{ij}^{*}}$ equals $|G_{ij}|$.
Therefore, in view of  \eqref{eq:from_cycle_to_edge} and \eqref{eq:wstar_after_def} as well as \eqref{eq:wij} and \eqref{eq:wij_star}, 
$((s_{ij}^{*})_{ij\in E}, (w_{ij,L}^{*})_{ij\in E, L \in N_{ij}})$ is a fixed point of the system of the update rules \eqref{eq:from_cycle_to_edge} and \eqref{eq:wij}.

\subsubsection{Message Passing from Cycles to Edges and Two Additional Assumptions}
\label{sec:fourth_probabilistic}
Here we explain the fifth step of the algorithm, which estimates, at iteration
$t$, $s_{ij}^*$ according to \eqref{eq:from_cycle_to_edge}. 
We further assume the good-cycle condition and that $G_{ij} = CI_{ij}$. We remark that the first assumption implies that 
$G_{ij}\subseteq CI_{ij}$ according to Proposition~\ref{prop:good cycle}, but it does not imply that $G_{ij} \supseteq CI_{ij}$. The first assumption, which we can state as $G_{ij} \neq \emptyset$, also implies that $CI_{ij}\neq \emptyset$, or equivalently,
\begin{equation}\nonumber\label{eq:Dij}
s_{ij}^*\in D_{ij}, \quad  ij\in E.
\end{equation}
This equation 
suggests an estimation procedure of $\{s_{ij}^*\}_{ij \in E}$.
One may greedily search for $s_{ij}^*$ among all elements of $D_{ij}$, but this is a hard combinatorial problem.
Instead, \eqref{eq:from_cycle_to_edge} relaxes this problem and searches over the convex hull of $D_{ij}$, using a weighted average.

We further interpret \eqref{eq:from_cycle_to_edge} in view of the above statistical model. 
Applying the assumption $G_{ij}=CI_{ij}$, we rewrite \eqref{eq:wij} as 
 \begin{equation}
 w_{ij,L}(t)=\frac{1}{Z_{ij}(t)}\Pr\left(s_{ij}^*=d_L \Big|\left(s_{ab}(t)\right)_{ab\in N_L\setminus \{ij\}}\right).\label{eq:ws}
 \end{equation}
The update rule \eqref{eq:from_cycle_to_edge} can thus be interpreted as an iterative voting procedure for estimating $s_{ij}^*$, where cycle $L\in N_{ij}$ estimates $s_{ij}^*$ at iteration $t$ by $d_L$ with confidence $w_{ij,L}(t-1)$ that $s_{ij}^*=d_L$. If $L\notin G_{ij}$, then its inconsistency measure $d_L$ is contaminated by corrupted edges in $L$ and we expect its weight to decrease with the amount of corruption. This is demonstrated in the update rules of \eqref{eq:ith} and \eqref{eq:er}, where any corrupted edge $ab$ in a cycle $L\in N_{ij}$, whose corruption is measured by the size of $s_{ab}(t)$, would decrease the weight $w_{ij,L}(t)$.

We can also express \eqref{eq:from_cycle_to_edge} in terms of the following probability mass function $\mu_{ij}(x;t)$ on $D_{ij}$:
\begin{equation*}
\mu_{ij}(x;t)=\sum_{L\in N_{ij}, d_L=x} w_{ij,L}(t) \ \text{ for any } x\in D_{ij}.
\end{equation*}
This probability mass function can be regarded as the estimated posterior mass function of $s_{ij}^*$ given the estimated corruption levels $\left(s_{ab}(t)\right)_{ab\in E\setminus \{ij\}}$. The update rule~\eqref{eq:from_cycle_to_edge} can then be reformulated as follows:
\begin{equation} \label{eq:expectation}
s_{ij}(t)=\mathbb E_{\mu_{ij}(x;t-1)}s_{ij}^*=\sum_{x\in D_{ij}}\mu_{ij}(x;t-1) \cdot x.
\end{equation}

\subsubsection{Summarizing Diagram for the Message Passing Procedure}
\label{sec:summarizing_diagram}
We further clarify the message passing procedure by the following simple diagram in Figure \ref{fig:diagram}. 
The right hand side (RHS) of the diagram expresses two main distributions. The first one is for edge $ij$ being uncorrupted and the second one is that cycle $L \in N_{ij}$ provides the correct information for edge $ij$. We use the term ``Message Passing" since CEMP iteratively updates these two probabilistic distributions by using each other in turn. The update of the second distribution by the first one is more direct. The opposite update requires the estimation of corruption levels.
\begin{figure}[hbp]
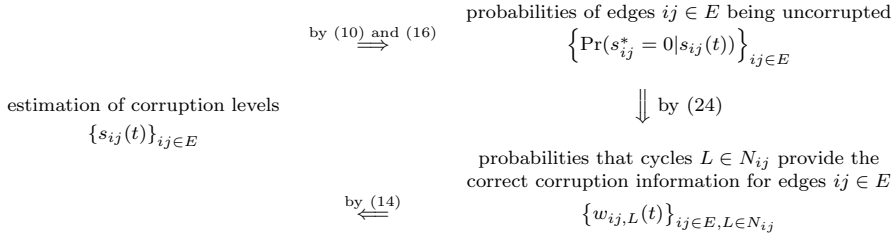

\resizebox{\textwidth}{!}{
\begin{tabular}{ccc}
 &&\text{probabilities of edges $ij \in E$ being uncorrupted}\\
&$\overset{\text{by } \eqref{eq:choice_f} \text{ and } \eqref{eq:pr_f}}{\Longrightarrow}$&$\left\{\Pr(s_{ij}^*=0|s_{ij}(t))\right\}_{ij\in E}$\\
&&\\
\text{estimation of corruption levels}&  &$\Big\Downarrow$ \footnotesize{by \eqref{eq:ws} }\\
$\left\{s_{ij}(t)\right\}_{ij\in E}$&&\\
&&\text{probabilities that cycles $L \in N_{ij}$ provide the}\\
&&\text{correct corruption information for edges $ij \in E$}\\
&$\overset{\text{by}~\eqref{eq:from_cycle_to_edge}}{\Longleftarrow}$& $\left\{w_{ij,L}(t)\right\}_{ij\in E, L\in N_{ij}}$
\end{tabular}
}\caption{Diagram for explaining CEMP}\label{fig:diagram}
\end{figure}

\subsubsection{Refined Statistical Model for the Specific Reweighting Functions}\label{sec:cemp_functions}

The two choices of $f(x;\beta_t)$ in \eqref{eq:choice_f} correspond to a more refined probabilistic model on $s_{ij}^*$ and $s_{ij}(t)$, which can also apply to other choices of reweighting functions. 
In addition to the above assumptions, 
this model assumes that
the edges in $E$ are independently  corrupted with probability $q$.

 We denote by 
 $F_g(x;t)$ and $F_b(x;t)$ the probability distributions of $s_{ij}(t)$ conditioned on the events $s_{ij}^*=0$ and $s_{ij}^*\neq 0$, respectively.
We further denote by $p_g(x;t)$ and $p_b(x;t)$, the respective probability density functions of $F_g(x;t)$ and $F_b(x;t)$ and define $r(x;t)=p_b(x;t)/p_g(x;t)$.
By Bayes' rule and the above assumptions, for any $ij\in E$
\begin{align}
f(s_{ij}(t);\beta_t)&=\Pr(s_{ij}^*=0|s_{ij}(t))\nonumber
=\frac{(1-q)\cdot p_g\left(s_{ij}(t);t\right)}{(1-q)\cdot p_g\left(s_{ij}(t);t\right)+q\cdot p_b\left(s_{ij}(t);t\right)}\\
&= \left(1+\frac{q}{1-q}\cdot r\left(s_{ij}(t);t\right)\right)^{-1}.\label{eq:wij3}
\end{align}
One can note that the update rule \ith in \eqref{eq:ith} corresponds to \eqref{eq:wp} with \eqref{eq:wij3} and
\begin{equation}\label{eq:rt1}
r(x;t)\propto \frac{\mathbf{1}_{\{x\leq 1\}}}{\mathbf{1}_{\{x\leq \frac{1}{\beta_t}\}}}=\begin{cases}1,&0\leq x<\frac{1}{\beta_t};\\ \infty, &  \frac{1}{\beta_t}\leq x\leq 1. \end{cases}
\end{equation}
 Due to the normalization factor and the fact that each cycle has the same length, the update rule \ith is invariant to the scale of $r(x;t)$, and we thus used the proportionality symbol.  Note that there are infinitely many $F_g(x;t)$ and $F_b(x;t)$ that result in such $r(x;t)$.  One simple example is uniform $F_g(x;t)$ and $F_b(x;t)$ on  $[0,1/\beta_t]$ and $[0,1]$,  respectively.
 
 One can also note that the update rule \er approximately corresponds to \eqref{eq:wp} with \eqref{eq:wij3} and \begin{equation}\label{eq:rt2}
r(x;t) = \alpha e^{\beta_t x} \ \text{ for sufficiently large }  \alpha  \text{ and }  x \in [0,1].
\end{equation}
Indeed, by plugging \eqref{eq:rt2} in \eqref{eq:wij3} we obtain that for $\alpha'=\alpha q/(1-q)$
\begin{equation}\label{eq:approx_e}
    f(s_{ij}(t);\beta_t) = (1+\alpha'e^{\beta_t x})^{-1} \approx e^{-\beta_t x}/ \alpha'.
\end{equation}
Since the update rule \er is invariant to scale (for the same reason explained above for the update rule \ith),
$\alpha$ can be chosen arbitrarily large to yield a good approximation in \eqref{eq:approx_e} with sufficiently large $\alpha'$. 
One may obtain \eqref{eq:rt2} by choosing $F_g(x;t)$ and $F_b(x;t)$ as exponential distributions restricted to $[0,1]$, or normal distributions restricted to $[0,1]$ with the same variance but different means.

As explained later in Section \ref{sec:theory}, $\beta_t$ needs to approach infinity in the noiseless case. We note that this implies that $r(x;t)$ (in either \eqref{eq:rt1} or \eqref{eq:rt2}) is infinite at $x \in (0,1]$ and finite at  $x=0$. Therefore, in this case, $F_g(x;t)\to \delta_0$. This makes sense since $s_{ij}^*=0$ when $ij \in E_g$.

\begin{remark}
Neither rules \ith nor \er makes explicit assumptions on the distributions of  $s_{ij}(t)$ and $s_{ij}^*$ and thus there are infinitely many choices of $F_g(x;t)$ and $F_b(x;t)$, which we find flexible.
\end{remark}

\subsubsection{The Practical Advantage of the Exponential Reweighting Function}\label{sec:practical}
In principle, one may choose any nonincreasing reweighting functions  $f(x;\beta)$ such that 
\begin{equation}
\label{eq:require_f}
\lim_{\beta\to 0}f(x;\beta)=1 \ \text{ and } \ \lim_{\beta\to +\infty} f(x; \beta) = 1-\sign(x) \text{ for all } \ x\in [0,1].
\end{equation}
In practice, we advocate using $f(x;\beta) = \exp(-\beta x)$ of CEMP-\er due to its nice property of shift invariance, which we formulate next and prove in Appendix \ref{sec:prop_shift}.
\begin{proposition}\label{prop:shift}
Assume that $\mathcal C$ consists of cycles with equal length $l$. For any fixed $ij\in E$ and $s\in \mathbb R$, the estimated corruption levels $\{s_{ij}(t)\}_{ij\in E}$ and $\{s_{ij}(t)+s\}_{ij\in E}$ result in the same cycle weights $\{w_{ij,L}(t)\}_{ij\in E, L\in N_{ij}}$ in CEMP-\er.
\end{proposition}

We demonstrate the advantage of the above shift invariance property with a simple example: Assume that an edge $ij$ is only contained in two cycles $L_1=\{ij,ik_1, jk_1\}$ and $L_2=\{ij, ik_2, jk_2\}$. Using the notation $s_{ij,L_1}(t):=s_{ik_1}(t)+s_{jk_1}(t)$ and $s_{ij,L_2}(t):=s_{ik_2}(t)+s_{jk_2}(t)$,  
we obtain that 
$$w_{ij,L_1}(t)/w_{ij,L_2}(t)=\exp(-\beta_t (s_{ij,L_1}(t)-s_{ij,L_2}(t)))$$ 
and since 
$w_{ij,L_1}(t)+w_{ij,L_2}(t)=1$, $w_{ij,L_1}(t)$ and $w_{ij,L_2}(t)$ are determined by  $s_{ij,L_1}(t)-s_{ij,L_2}(t)$. Therefore, the  choice of $\beta_t$ for CEMP-\er only depends on the ``corruption variation'' for edge $ij$, $s_{ij,L_1}(t)-s_{ij,L_2}(t)$. It is completely independent of the average scale of the corruption levels, which is proportional in this case to $s_{ij,L_1}(t)+s_{ij,L_2}(t)$.
On the contrary, CEMP-\ith heavily depends on the average scale of the corruption levels. Indeed, the general expression in CEMP-\ith is
\begin{align*}
  w_{ij,L}(t) =  \frac{\mathbf 1_{\{\max_{ab\in N_L\setminus \{ij\}}s_{ab}(t)\leq \frac{1}{\beta_t}\}}}{\sum_{L'\in N_{ij}} \mathbf 1_{\{\max_{ab\in N_{L'}\setminus \{ij\}}s_{ab}(t)\leq \frac{1}{\beta_t}\}}}.
\end{align*}
Redefining $s_{ij,L_1}(t):=\max\{s_{ik_1}(t),s_{jk_1}(t)\}$ and $s_{ij,L_2}(t):=\max\{s_{ik_2}(t),s_{jk_2}(t)\}$, we obtain that
$$w_{ij,L_1}(t)/w_{ij,L_2}(t)={\mathbf 1}_{\{s_{ij,L_1}(t) \leq \frac{1}{\beta_t}\}}/{\mathbf 1}_{\{s_{ij,L_2}(t) \leq \frac{1}{\beta_t}\}}.$$
The choice of $\beta_t$ depends on both values of  $s_{ij,L_1}(t)$ and $s_{ij,L_2}(t)$ and not on any meaningful variation. One can see that in more general cases, the correct choice of $\beta_t$ for CEMP-\ith can be rather restrictive and will depend on different local corruption levels of edges.

\subsubsection{On the Computational Complexity of CEMP}
\label{sec:computational}
We note that for each $L\in\mathcal C$, CEMP needs to compute $d_L$ and thus the complexity at each iteration is of order $O(\sum_{L\in \mathcal C}|L|)$. In the case of $\mathcal C=\mathcal C_3$, which we advocate later, this complexity is of order  $O(|\mathcal C_3|)$ and thus $O(n^3)$ for sufficiently dense graphs.
In practice, one can implement a faster version of CEMP by only selecting a fixed number of $3$-cycles per edge which reduces the complexity per iteration to  $O(|E|)$, which is $O(n^2)$ for sufficiently dense graphs. Nevertheless we have not discussed the full guarantees for this procedure.
In order to obtain the overall complexity, and not the complexity per iteration, one needs to guarantee sufficiently fast convergence. Our later theoretical statements guarantee linear convergence of CEMP under various conditions and consequently guarantee that the overall complexity is practically of the same order of the complexity per iteration. 
We note that  the upper bound $O(n^3)$ for the complexity of CEMP with $\mathcal C=\mathcal C_3$ is lower than  the complexity of SDP for common group synchronization problems. We thus refer to our method as   fast. 
In general, for CEMP with $\mathcal C=\mathcal C_l$, the complexity is $O(ln^l)$.

We exemplify two different scenarios where the complexity of CEMP can be lower than the bounds stated above. 

\noindent
{\bf An example of lower complexity due to graph sparsity:}
We assume the special case, where the underlying graph $G([n], E)$ is generated by the Erd\H{o}s-R\'{e}nyi model $G(n,p)$ and the group is  $\mathcal G=SO(d)$. We estimate the complexity of CEMP with $\mathcal{C} = \mathcal{C}_3$.
Note that the number of edges concentrates at $n^2p$. Each edge is contained in about $np^2$ 3-cycles. Thus, the number of $d_L$'s concentrate at $n^3p^3$. The computational complexity of each $d_L$ is $d^3$. Therefore, the computational complexity of initializing CEMP is about $n^3p^3d^3$. In each iteration of the reweighting stage, one only needs to compute $w_{ij,L}$ for each 3-cycle and average over  $ij\in E$ and thus the complexity is $n^3p^3$. We assume, e.g., the noiseless adversarial case, where the convergence is linear. Thus, the total complexity of CEMP is $O((npd)^3)$.  The complexity of Spectral is of order $O(n^2d^3)$, so the complexity of CEMP is lower than that of Spectral when $p=O(n^{-1/3})$. Observe that the upper bound $p=n^{-1/3}$ is higher than the phase transition threshold for the existence of 3-cycles (which is $p=n^{-1/2}$); thus, this fast regime of CEMP (with $\mathcal{C}= \mathcal{C}_3$) is nontrivial.   

\noindent
{\bf An example of low complexity with high-order cycles:}
In \cite{shi_NEURIPS2020}, a lower complexity of CEMP with $\mathcal{C}=\mathcal{C}_l$, $l>3$, is obtained for the following special case:  $\mathcal G = S_N$ with $d_\mathcal G(g_1, g_2)=\|g_1-g_2\|_F^2/(2N)$ (recall that $\|\cdot\|_F$ denotes the Frobenius norm and we associate $g_1$ and $g_2$ with their matrix representations). In this case, it is possible to compute the weights $\{w_{ij}\}_{ij\in E}$ by calculating powers of the graph connection weight matrix \cite{VDM_singer}. Consequently, the complexity of CEMP is reduced to $O(ln^3)$. It seems that this example is rather special and we find it difficult to generalize its ideas.

\subsubsection{Post-processing:  Estimation of the Group Elements}\label{sec:postproc}
After running CEMP for $T$ iterations, one obtains the estimated corruption levels $s_{ij}(T)$, $ij \in E$. As a byproduct of CEMP, one also obtains $$p_{ij}:=f(s_{ij}(T);\beta_T), \ \text{ for } ij \in E,$$ 
which we interpreted in \eqref{eq:pr_f} as the estimated probability that $ij\in E_g$ given the value of $s_{ij}(T)$. Alternatively, one may normalize $p_{ij}$ as follows:
$$\tilde p_{ij}=\frac{f(s_{ij}(T);\beta_T)}{\sum_{j\in [n] :  ij \in E} f(s_{ij}(T);\beta_T)}, \ \text{ for } ij \in E,$$  
so that $\sum_{j\in [n] :  ij \in E} \tilde p_{ij}=1$.
Using either of these values
($s_{ij}(T)$, $p_{ij}(T)$
or $\tilde{p}_{ij}(T)$ for all $ij \in E$), we describe different possible strategies for estimating the underlying group elements $\{g_i^*\}_{i=1}^n$ in more general settings than that of our proposed theory. 
As we explain below we find our second proposed method (CEMP+GCW) as the most appropriate one in the context of the current paper. Nevertheless, there are settings where other methods will be preferable.

\noindent
{\bf Application of the minimum spanning tree (CEMP+MST):} 
One can assign the weight $s_{ij}(T)$ for each edge $ij\in E$ and find the minimum spanning tree (MST) of the weighted graph. The resulting spanning tree minimizes the average of the estimated corruption levels. Next, one can fix $g_1=e_\mathcal G$ and estimate the rest of the group elements by subsequently multiplying group ratios (using the formula $g_i=g_{ij}g_j$) along the spanning tree. We refer to this procedure as CEMP+MST. Alternatively, one can assign edge weights  $\{p_{ij}(T)\}_{ij\in E}$ and find the maximum spanning tree, which aims to maximize the expected number of good edges. These methods can work well when there is a connected inlier graph with little noise, but will generally not perform well in noisy situations. Indeed, when the good edges are noisy, estimation errors will rapidly accumulate with the subsequent applications of the formula $g_i=g_{ij}g_j$ and the final estimates are expected to be erroneous.

\noindent
{\bf A CEMP-weighted spectral method (CEMP+GCW):} 
Using $\{\tilde{p}_{ij}(T)\}_{ij\in E}$ obtained by CEMP one may try to approximately solve the following weighted least squares problem:
\begin{align}\label{eq:weighted}
    \{\hat g_i\}_{i\in [n]}=\argmin_{\{g_i\}_{i\in [n]}\subset \mathcal G} \sum_{ij\in E}\tilde p_{ij}d^2_{\mathcal G}\left( g_{ij}, g_i g_j^{-1}\right)
\end{align}
and use this solution as an estimate of $\{g_i^*\}_{i=1}^n$. Note that since $\mathcal G$ is typically not convex, the solution of this problem is often hard.
When $\mathcal G$ is a subgroup of the orthogonal group $O(N)$, an argument of \cite{se3_sync} for the same optimization problem with $\mathcal{G}=SE(3)$ suggests the following relaxed spectral solution to \eqref{eq:weighted}. First, build a matrix $\bY_p$ whose $[i,j]$-th block is $\tilde p_{ij}g_{ij}$ for $ij\in E$, and $\mathbf{0}$ otherwise. Next, compute the top $N$ eigenvectors of $\bY_p$ to form the block vector $\hat\bx$, and finally project the $i$-th block of $\hat\bx$ onto $\mathcal G$ to obtain the estimate of $g_i^*$ for $i\in [n]$. Note that $\bY_p$ is exactly the graph connection weight (GCW) matrix in vector diffusion maps \cite{VDM_singer}, given the edge weights $\tilde p_{ij}$. Thus, we refer to this method as CEMP+GCW.  We note that the performance of CEMP+GCW is mainly determined by the accuracy of estimating the corruption levels. Indeed, if  the corruption levels  $\{s_{ij}(T)\}_{ij \in E}$ are sufficiently accurate, then the weights $\{\tilde{p}_{ij}\}_{ij \in E}$ are sufficiently accurate and \eqref{eq:weighted} is close to a direct least squares solver for the inlier graph. Since the focus of this paper is accurate estimation of $\{s_{ij}^*\}_{ij \in E}$, we mainly test CEMP+GCW as a direct CEMP-based group synchronization solver (see  Section \ref{sec:experiment}).

\noindent
{\bf Iterative application of CEMP and weighted least squares (MPLS):} 
In highly corrupted and noisy datasets, iterative application of CEMP and the weighted least squares solver in \eqref{eq:weighted} may result in a satisfying solution. After the submission of this paper, the authors proposed a special procedure like this, which they called Message Passing Least Squares (MPLS) \cite{MPLS}.

\noindent
{\bf Combining CEMP with any another solver:} 
CEMP can be used as an effective cleaning procedure for removing some bad edges (with estimated corruption levels above a chosen threshold). One can then apply any group synchronization solver using the cleaned graph. Indeed, existing solvers often cannot deal with high and moderate levels of corruption and should benefit from initial application of CEMP. Such a strategy was tested with the AAB algorithm \cite{AAB}, which motivated the development of CEMP.

\subsection{Comparison of CEMP with BP, AMP and IRLS}\label{sec:compare}
CEMP is different from BP~\cite{BP} in the following ways. First of all, unlike BP that needs to explicitly define the joint density and the statistical model a-priori, CEMP does not use an explicit objective function, but only makes weak assumptions on the corruption model. Second, CEMP is guaranteed (under a certain level of corruption) to handle factor graphs that contain loops. Third, CEMP utilizes the auxiliary variable $s_{ij}(t)$ that connects the two binary distributions on the RHS of the diagram in Figure \ref{fig:diagram}. Thus,
unlike~\eqref{eq:bp} of BP that only distinguishes the two events: $ij\in E_g$ and $ij\in E_b$, CEMP also tries to approximate the exact value of corruption levels $s_{ij}^*$ for all $ij\in E$, which can help in inferring corrupted edges.

In practice, AMP \cite{AMP_compact} directly solves group elements, but with limited theoretical guarantees for group synchronization. CEMP has two main advantages over AMP when assuming the theoretical setting of this paper. First of all, AMP for group synchronization \cite{AMP_compact} assumes additive Gaussian noise without additional corruption and  thus it is not robust to outliers. In contrast, we guarantee the robustness of CEMP to both adversarial and uniform corruption. We further establish the stability of CEMP to sufficiently small bounded and sub-Gaussian noise. Second of all, the heuristic argument for deriving AMP for group synchronization (see Section 6 of \cite{AMP_compact}) provides asymptotic convergence theory, whereas CEMP has convergence guarantees under certain deterministic conditions for finite sample with attractive convergence rate.

Another related line of work is IRLS  that is commonly used to solve $\ell_1$ minimization problems. At each iteration, it utilizes the  residual of a weighted least squares solution to quantify the corruption level at each edge. New weights,  which are typically inversely proportional to this residual, are assigned for an updated weighted least squares problem, and the process continues till convergence. The IRLS reweighting strategy is rather aggressive, and in the case of  high corruption levels, it may wrongly assign extremely high weights to corrupted edges and consequently it can get stuck at local minima.  When the group is discrete, some residuals of corrupted edges can be 0 and the corresponding weights can be extremely large. Furthermore, in this case the residuals and the edge weights lie in a discrete space and therefore IRLS can easily get stuck at local minima. For general groups, the $\ell_1$ formulation that IRLS aims to solve is statistically optimal to a very special heavy-tailed distribution, and is not optimal, for example, to the corruption model proposed in \cite{wang2013exact}. Instead of assigning weights to edges, CEMP assigns weights to cycles and uses the weighted cycles to infer the corruption levels of edges. It starts with a conservative reweighting strategy with $\beta_t$ small and gradually makes it more aggressive by increasing $\beta_t$. This reweighting strategy is crucial for guaranteeing the convergence of CEMP.  CEMP is also advantageous when the groups are discrete because it  estimates conditional expectations whose values lie in a continuous space. This makes CEMP less likely to get stuck in a local minima.

\section{Theory for Adversarial Corruption}\label{sec:theory}
We show that when the ratio between the size of $G_{ij}$ (defined in Section \ref{sec:steps_notation}) and the size of $N_{ij}$ (defined in Section \ref{sec:CEG})
is uniformly above a certain threshold and  $\{\beta_t\}_{t=0}^T$ is increasing and chosen in a certain way, then for all $ij\in E$, the estimated corruption level $s_{ij}(t)$ linearly converges to $s_{ij}^*$, and the convergence is uniform over all $ij\in E$. The theory is similar for both update rules A and B.  
Note that the uniform lower bound on the above ratio is a geometric restriction on the set $E_b$. This is the only restriction we consider in this section; indeed, we follow the adversarial setting, where the group ratios $g_{ij}$ for $ij\in E_b$ can be arbitrarily chosen, either deterministically or randomly. 
We mentioned in Section \ref{sec:contribution}
that the only other guarantees for such adversarial corruption but for a different problem are in \cite{HandLV15,LUDrecovery} and that we found them weaker.

The rest of the section is organized as follows. Section \ref{sec:pre} presents preliminary notation and background. Section \ref{sec:determin} establishes the linear convergence of CEMP to the ground truth corruption level under adversarial corruption. Section \ref{sec:bounded} establishes the stability of CEMP to bounded noise, and 
 Section \ref{sec:subg} extends these results to sub-Gaussian noise.

\subsection{Preliminaries}\label{sec:pre}
For clarity of our presentation, we assume that $\mathcal C = \mathcal C_3$ and thus simplify some of the above notation and claims. Note that $L\in \mathcal C_3$ contains 3 edges and 3 vertices. 
Therefore, given $i,j \in [n]$ and $L\in N_{ij}$, we index $L$ by the vertex $k$, which is not $i$ or $j$. We thus replace the notation $d_L$ with  $d_{ij,k}$. We also note that the sets $N_{ij}$ and $G_{ij}$ can be expressed as follows: $N_{ij}=\{k\in [n]:ik,jk\in E\}$ and $G_{ij}=\{k\in N_{ij}: ik,jk\in E_g\}$. We observe that if $A$ is the adjacency matrix of $G([n], E)$ (with 1 if $ij \in E$ and 0 otherwise), then by the definitions of matrix multiplication and $N_{ij}$, $A^2(i,j) = |N_{ij}|$. Similarly, if $A_g$ is the adjacency matrix of $G([n], E_g)$, then 
$A_g^2(i,j) = |G_{ij}|$.
We define the corrupted cycles containing the edge $ij$ as $B_{ij} =N_{ij} \setminus G_{ij}$,
so that $|B_{ij}| = A^2(i,j)-A_g^2(i,j)$. 
We also define 
$\epsilon_{ij} (t)=|s_{ij}(t)-s_{ij}^*|$, $\epsilon (t)=\max_{ij\in E}|s_{ij}(t)-s_{ij}^*|$, 
\begin{equation}
\label{def_lambda}
\lambda_{ij}=|B_{ij}|/|N_{ij}| 
\ \text{ and } \
\lambda=\max_{ij\in E} \lambda_{ij}.
\end{equation}

An upper bound for the parameter $\lambda$ quantifies our adversarial corruption model. Let us clarify more carefully the ``adversarial corruption'' model and the parameter $\lambda$, while repeating some previous information. This model assumes a graph $G([n], E)$ whose nodes represent group elements and whose edges are assigned group ratios satisfying \eqref{eq:model}, where $E = E_b \cup E_g$ and $E_b \cap E_g = \emptyset$. When $g_{ij}^\epsilon=e_\mathcal G$ for all $ij \in E_g$ (where $g_{ij}^\epsilon$ appear in \eqref{eq:model}), we refer to this model as noiseless, and otherwise, we refer to it as noisy. For the noisy case, we will specify assumptions on the distribution of
$d_{\mathcal G}(g_{ij}^\epsilon,e_{\mathcal G})$ for all $ij \in E_g$, or equivalently (since $d_{\mathcal G}$ is bi-invariant) the distribution of $s_{ij}^*$ for all $ij \in E_g$. 

In view of the above observations, we note that the parameter $\lambda$, whose upper bound quantifies some properties of this model, can be directly expressed using the adjacency matrices $A$ and $A_g$ as follows
\begin{equation}
\label{def_lambda2}
\lambda=\max_{ij\in E} \left( 1 - \frac{A_g^2(i,j)}{A^2(i,j)}  \right).    
\end{equation} 
Thus an upper bound $m$ on $\lambda$ is the same as a lower bound $1-m$ on $\min_{ij\in E} {A_g^2(i,j)}$ $/{A^2(i,j)}$. 
This lower bound is equivalent to a lower bound on the ratio between the size of $G_{ij}$  and the size of $N_{ij}$. We note that this bound implies basic properties mentioned earlier. First of all, it implies that $G_{ij}$ is nonempty for all $ij \in E$ and it thus implies that the good-cycle condition holds. This in turn implies that $G([n], E_g)$ is connected (since if $ij \in E$ and $k \in G_{ij}$, then $ik$, $kj \in E_g$).

Our proofs frequently use Lemma \ref{lemma:bi}, which can be stated in our special case of $\mathcal C = \mathcal{C}_3$ as  
\begin{align}
\label{eq:bi}
|d_{ij,k}-s_{ij}^*|\leq s_{ik}^* + s_{jk}^* \ \text{ for all } ij \in E \text{ and } k \in N_{ij}.
\end{align}
We recall that $d_{\mathcal G}(\cdot\,,\cdot)\leq 1$ and thus 
\begin{equation}
\label{eq:prop_0_1}
\text{For all }  i,j \in E  \text{ and } k \in N_{ij}, \     0 \leq s_{ij}^* \leq 1 \ \text{ and } \ 0 \leq d_{ij,k} \leq 1.
\end{equation}

Since $\mathcal C=\mathcal C_3$, the update rule \eqref{eq:update} can be further simplified as follows. For CEMP-A,
\begin{align}\label{eq:updateA}
    s_{ij}(t+1)=\frac{\sum\limits_{k\in N_{ij}} \mathbf 1_{\left\{s_{ik}(t)\,,s_{jk}(t)\leq 1/\beta_t\right\}} d_{ij,k}}{\sum\limits_{k\in N_{ij}} \mathbf 1_{\left\{s_{ik}(t)\,,s_{jk}(t)\leq 1/\beta_t\right\}}},
\end{align}
and for CEMP-B,
\begin{align}\label{eq:updateB}
s_{ij}(t+1)=\frac{\sum\limits_{k\in N_{ij}}e^{-\beta_t\left(s_{ik}(t)+s_{jk}(t)\right)}d_{ij,k}}{\sum\limits_{k\in N_{ij}}e^{-\beta_t\left(s_{ik}(t)+s_{jk}(t)\right)}}.
\end{align}
The initial corruption estimate at $ij\in E$ in \eqref{eq:initialize_wij} for both versions of CEMP is 
\begin{equation}
\label{eq:initial_corruption}
s_{ij}(0)=\frac{1}{|N_{ij}|}\sum_{k\in N_{ij}}d_{ij,k}.    
\end{equation}

\subsection{Deterministic Exact Recovery}\label{sec:determin}
The following two theorems establish linear convergence of CEMP-A and CEMP-B,  assuming adversarial corruption and exponentially increasing $\beta_t$. The proofs are straightforward.
\begin{theorem}\label{thm:it1}
Assume data generated by the noiseless adversarial corruption model with parameter $\lambda<1/4$. Assume further  that the parameters $\{\beta_t\}_{t \geq 0}$ of CEMP-A satisfy: 
$1<\beta_0\leq 1/\lambda$ and for all $t \geq 1$ $\beta_{t+1}=r\beta_t$
for some $1<r<1/(4\lambda)$. Then 
the estimates $\{s_{ij}(t)\}_{ij \in E}^{t \geq 0}$ of $\{s_{ij}^{*}\}_{ij \in E}$
computed by CEMP-A  satisfy
\begin{align}
\label{eq:thmit1}
\max_{ij\in E}|s_{ij}(t)-s_{ij}^*| \leq  \frac{1}{\beta_0 r^t} \ \text{ for all } t\geq 0.
\end{align}
\end{theorem}

\begin{proof}
The proof uses the following estimate, which applies  first \eqref{eq:updateA} and then \eqref{eq:bi}:
\begin{align}
\epsilon_{ij}(t+1)=|s_{ij}(t+1)-s_{ij}^*|&\leq \frac{\sum\limits_{k\in N_{ij}}\mathbf{1}_{\{s_{ik}(t), s_{jk}(t)\leq \frac{1}{\beta_t}\}}|d_{ij,k}-s^*_{ij}|}{\sum\limits_{k\in N_{ij}}\mathbf{1}_{\{s_{ik}(t), s_{jk}(t)\leq \frac{1}{\beta_t}\}}}
\nonumber\\
\label{eq:estimate_eps}
&\leq\frac{\sum\limits_{k\in N_{ij}}\mathbf{1}_{\{s_{ik}(t), s_{jk}(t)\leq \frac{1}{\beta_t}\}}(s_{ik}^*+ s_{jk}^*)}{\sum\limits_{k\in N_{ij}}\mathbf{1}_{\{s_{ik}(t), s_{jk}(t)\leq \frac{1}{\beta_t}\}}}.
\end{align}
Using the notation $A_{ij}(t):=\{k\in N_{ij}:s_{ik}(t),s_{jk}(t)\leq 1/\beta_t\}$ and the fact that $s_{ik}^*+ s_{jk}^* = 0$ for $ij \in G_{ij}$, we can rewrite the estimate in \eqref{eq:estimate_eps} as follows
\begin{align}
\epsilon_{ij}(t+1)
\label{eq:estimate_eps2}
\leq\frac{\sum\limits_{k\in B_{ij}}\mathbf{1}_{\{s_{ik}(t), s_{jk}(t)\leq \frac{1}{\beta_t}\}}(s_{ik}^*+ s_{jk}^*)}{|A_{ij}(t)|}.
\end{align}

The rest of the proof uses simple induction. For $t=0$, \eqref{eq:thmit1} is verified as follows
\begin{multline}
\label{eq:proof_for_t_0}
    \epsilon_{ij}(0) =|s_{ij}(0)-s_{ij}^*|\leq \frac{\sum\limits_{k\in N_{ij}}|d_{ij,k}-s^*_{ij}|}{|N_{ij}|}\\
    = \frac{\sum\limits_{k\in B_{ij}}|d_{ij,k}-s^*_{ij}|}{|N_{ij}|}\leq \frac{|B_{ij}|}{|N_{ij}|}\leq \lambda \leq \frac{1}{\beta_0},
\end{multline}
where the first inequality uses  \eqref{eq:initial_corruption}, the
second equality follows from the fact that $d_{ij,k}=s_{ij}^*$ for $k\in G_{ij}$, the second inequality follows from \eqref{eq:prop_0_1} (which implies that $|d_{ij,k}-s_{ij}^*|\leq 1$) and the last two inequalities use the assumptions of the theorem.
Next, we assume that $1/\beta_t\geq \epsilon (t)$ for an arbitrary $t>0$ and show that $1/\beta_{t+1}\geq \epsilon (t+1)$. We note that the induction assumption implies that  \begin{equation}\label{eq:inclusion}
    \frac{1}{\beta_t}\geq \epsilon (t)\geq \max_{ij\in E_g}\epsilon_{ij} (t)=\max_{ij\in E_g}s_{ij}(t),
\end{equation} 
and consequently, for $ij\in E$
$G_{ij}\subseteq A_{ij}(t)$. Combining this observation with $G_{ij} \cap B_{ij} = \emptyset$ yields
\begin{equation}\label{eq:eqal_sets_card}
A_{ij}(t)\cup \left(B_{ij}\setminus A_{ij}(t)\right)=N_{ij}.
\end{equation}
We further note that 
\begin{equation}
\label{eq:if_then}
\text{if  } s_{ij}(t)\leq\frac{1}{\beta_t},
\text{ then } 
s_{ij}^* \leq s_{ij}(t) + \epsilon (t)\leq s_{ij}(t)+\frac{1}{\beta_t} \leq \frac{2}{\beta_t}.
\end{equation}
Combining \eqref{eq:estimate_eps2} and \eqref{eq:if_then} and then applying basic properties of the different sets, in particular, \eqref{eq:eqal_sets_card} and the fact that $B_{ij}\setminus A_{ij}(t)$ is disjoint with both $A_{ij}(t)$ and $B_{ij}\cap A_{ij}(t)$, yields
\begin{multline}\label{eq:more_ratio_eps}
\epsilon_{ij}(t+1)
\leq\frac{\sum\limits_{k\in B_{ij}}\mathbf{1}_{\{s_{ik}(t), s_{jk}(t)\leq \frac{1}{\beta_t}\}}\frac{4}{\beta_t}}{|A_{ij}(t)|}=\frac{4}{\beta_t}\frac{|B_{ij}\cap A_{ij}(t)|}{|A_{ij}(t)|}
\\
\leq \frac{4}{\beta_t}\frac{|B_{ij}\cap A_{ij}(t)\cup \left(B_{ij}\setminus A_{ij}(t)\right)|}{|A_{ij}(t)\cup \left(B_{ij}\setminus A_{ij}(t)\right)|}
=4 \, \frac{|B_{ij}|}{|N_{ij}|}\frac{1}{\beta_t}.
\end{multline}
By taking the maximum of the left hand side (LHS) and RHS of \eqref{eq:more_ratio_eps} over 
$ij\in E$ and using the assumptions $\lambda<1/4$ and $4\lambda \beta_{t+1}<\beta_{t}$, we obtain that
\begin{align*}
\epsilon (t+1)\leq 4\lambda\frac{1}{\beta_t}<\frac{1}{\beta_{t+1}}.
\end{align*}
\end{proof}

\begin{theorem}\label{thm:ir1}
Assume data generated by the noiseless adversarial corruption model with parameter $\lambda<1/5$. Assume further  that the parameters $\{\beta_t\}_{t \geq 0}$ of CEMP-B satisfy:
$\beta_0\leq 1/(4\lambda)$ and for all $t \geq 1$ $\beta_{t+1}=r\beta_t$ for some $1<r<(1-\lambda)/(4\lambda)$. Then the estimates $\{s_{ij}(t)\}_{ij \in E}^{t \geq 0}$ of $\{s_{ij}^{*}\}_{ij \in E}$ computed by CEMP-B  satisfy
\begin{align*}
\max_{ij\in E}|s_{ij}(t)-s_{ij}^*|\leq \frac{1}{4\beta_0 r^t} \ \text{ for all } \ t\geq 0.
\end{align*}
\end{theorem}
\begin{proof}
Combining  \eqref{eq:bi} and \eqref{eq:updateB}
yields
\begin{align}
\label{eq:bound_2_eps}
\epsilon_{ij}(t+1)=|s_{ij}(t+1)-s_{ij}^*| \leq\frac{\sum\limits_{k\in B_{ij}}e^{-\beta_t\left(s_{ik}(t)+s_{jk}(t)\right)}\left(s_{ik}^*+ s_{jk}^*\right)}{\sum\limits_{k\in N_{ij}}e^{-\beta_t\left(s_{ik}(t)+s_{jk}(t)\right)}}.
\end{align}
Applying \eqref{eq:bound_2_eps}, the definition of $\epsilon_{ij} (t)$ and the facts that $G_{ij} \subseteq N_{ij}$ and $s_{ik}^*+ s_{jk}^*=0$ for $k \in G_{ij}$, we obtain that
\begin{align}
\label{eq:bound_2_eps2}
\epsilon_{ij}(t+1)
&\leq\frac{\sum\limits_{k\in B_{ij}}e^{-\beta_t\left(s_{ik}^*+s_{jk}^*-\epsilon_{ik}(t)-\epsilon_{jk}(t)\right)}\left(s_{ik}^*+ s_{jk}^*\right)}{\sum\limits_{k\in G_{ij}}e^{-\beta_t\left(\epsilon_{ik}(t)+\epsilon_{jk}(t)\right)}}\nonumber\\
&\leq{\frac{1}{|G_{ij}|}
\sum\limits_{k\in B_{ij}}e^{2\beta_t\left(\epsilon_{ik}(t)+\epsilon_{jk}(t)\right)}} e^{-\beta_t\left(s_{ik}^*+s_{jk}^*\right)}\left(s_{ik}^*+ s_{jk}^*\right).
\end{align}

The proof follows by induction. For $t=0$, \eqref{eq:proof_for_t_0} implies that $\lambda\geq \epsilon (0)$ and thus $1/(4\beta_0)\geq \lambda\geq \epsilon (0)$. Next, we assume that $1/(4\beta_t) \geq \epsilon (t)$ and show that $1/(4\beta_{t+1})\geq \epsilon (t+1)$.
We do this by simplifying and weakening \eqref{eq:bound_2_eps2} as follows. We first  bound each term in the sum on the RHS of \eqref{eq:bound_2_eps2}
by applying the inequality $xe^{-ax}\leq 1/(ea)$ for $x\geq 0$ and $a>0$. We let $x=s_{ik}^{*}+s_{jk}^{*}$ and $a=\beta_t$ and thus each term is bounded by $1/(ea)$. We then use the induction assumption ($\epsilon (t) \leq 1/(4\beta_t)$) to bound the exponential term in the numerator on the RHS of \eqref{eq:bound_2_eps2} by e. We therefore conclude that
\begin{align}
\label{eq:bound_2_eps3}
\epsilon_{ij}(t+1)
\leq \frac{|B_{ij}|}{|G_{ij}|}\cdot\frac{1}{\beta_t}.
\end{align}
By applying the assumption $\lambda<1/5$ and maximizing over $ij\in E$ both the LHS and RHS of \eqref{eq:bound_2_eps3}, we conclude the desired induction as follows 
\begin{align}
\label{eq:bound_2_eps4}
\epsilon (t+1)
\leq \frac{\lambda}{1-\lambda}\cdot\frac{1}{\beta_t}<\frac{1}{4\beta_{t}}<\frac{1}{4\beta_{t+1}}.
\end{align}
\end{proof}

\subsection{Stability to Bounded Noise}\label{sec:bounded}
We assume the noisy adversarial corruption model in \eqref{eq:model} and an upper bound on $\lambda$. 
We further assume that there exists $\delta>0$, such that for all $ij \in E_g$, $s_{ij}^* \equiv d_\mathcal G(g_{ij}^\epsilon, e_\mathcal G) \leq \delta$. This is a general setting of perturbation without probabilistic assumptions. Under these assumptions, we show that CEMP can approximately recover the underlying corruption levels, up to an error of order $\delta$.
The proofs of the two theorems below are similar to the proofs of the theorems in Section \ref{sec:determin} and are thus included in Appendices~\ref{sec:it2} and \ref{sec:ir2}.

\begin{theorem}\label{thm:it2}
Assume data generated by adversarial corruption with bounded noise, where the model parameters satisfy $\lambda<1/4$ and $\delta>0$. Assume further that the parameters $\{\beta_t\}_{t \geq 0}$ of CEMP-A satisfy:  $1/\beta_0>\max\{(3-4\lambda)\delta/(1-4\lambda)\,, $$\, \lambda+3\delta\}$ and $4\lambda/\beta_t+(3-4\lambda)\delta\leq 1/\beta_{t+1}<1/\beta_t$. Then 
the estimates $\{s_{ij}(t)\}_{ij \in E}^{t \geq 0}$ of $\{s_{ij}^{*}\}_{ij \in E}$
computed by CEMP-A  satisfy
\begin{equation}
\max_{ij\in E}|s_{ij}(t)-s_{ij}^*| \leq  \frac{1}{\beta_t}-\delta
\ \text{ for all } \ t\geq 0.\label{eq:itall}
\end{equation}
Moreover,  $\varepsilon := \lim\limits_{t\to\infty}\beta_t(1-4\lambda)/((3-4\lambda)\delta)$  satisfies $0<\varepsilon\leq 1$ and the following asymptotic bound holds
\begin{equation}
\lim_{t\to\infty}\max_{ij\in E}|s_{ij}(t)-s_{ij}^*| \leq  \left(\frac{3-4\lambda}{\varepsilon(1-4\lambda)}-1\right)\delta.\label{eq:itinf}
\end{equation}
\end{theorem}

\begin{theorem}\label{thm:ir2}
Assume data generated by 
adversarial corruption with independent bounded noise, where the model parameters satisfy $\lambda<1/5$ and $\delta>0$.  Assume further  that the parameters $\{\beta_t\}_{t \geq 0}$ of CEMP-B satisfy:  $1/(4\beta_0)>\max\{(5(1-\lambda)\delta)$ $/(2(1-5\lambda))\,, \lambda+5\delta/2\}$ and  $10\delta+4\lambda/((1-\lambda)\beta_t)\leq 1/\beta_{t+1}<1/\beta_t$. Then  the estimates $\{s_{ij}(t)\}_{ij \in E}^{t \geq 0}$ of $\{s_{ij}^{*}\}_{ij \in E}$
computed by CEMP-B  satisfy
\begin{equation}
\quad \max_{ij\in E}|s_{ij}(t)-s_{ij}^*|\leq \frac{1}{4\beta_t}-\frac12\delta \ \text{ for all } \ t\geq 0. \label{eq:irall}
\end{equation}
\text{Moreover,  } $\varepsilon := \lim\limits_{t\to\infty}\beta_t (1-5\lambda)/(10(1-\lambda)\delta)$ satisfies $0<\varepsilon\leq 1$ and the following asymptotic bound holds
\begin{equation}
\label{eq:irinf}
 \lim_{t\to\infty}\max_{ij\in E}|s_{ij}(t)-s_{ij}^*|\leq \left(\frac{5}{2\varepsilon}\cdot \frac{1-\lambda}{1-5\lambda}-\frac12\right)\delta.
\end{equation}

\end{theorem}

\begin{remark}
By knowing $\delta$, one can tune the parameters to obtain $\varepsilon=1$. 
Indeed, one can check that by taking $1/\beta_{t+1}= 4\lambda/\beta_t+(3-4\lambda)\delta$ in Theorem \ref{thm:it2}, $1/\beta_t$ linearly converges to $(3-4\lambda)\delta/(1-4\lambda)$ (with rate $4 \lambda$). Similarly, by taking $1/\beta_{t+1}=10\delta+4\lambda/((1-\lambda)\beta_t)$ in Theorem \ref{thm:ir2}, $1/\beta_t$ linearly converges to $10(1-\lambda)\delta/(1-5\lambda)$ (with rate $4 \lambda/ (1- \lambda)$). These choices clearly result in $\varepsilon=1$.
\end{remark} 

\begin{remark}
The RHSs of \eqref{eq:itinf} and \eqref{eq:irinf} imply that CEMP approximately recover the corruption levels with error $O(\delta)$.
Since this bound is only meaningful with values at most 1,  $\delta$ can be at most $1/2$ (this bound is obtained when $\varepsilon=1$ and $\lambda =0$). Furthermore, when $\lambda$ increases or $\varepsilon$ decreases, the bound on $\delta$ decreases. The bound on $\delta$ limits the applicability of the theorem, especially for discrete groups. For example, in $\mathbb Z_2$  synchronization, $s_{ij}^*\in \{0,1\}$ and thus the above theorem is inapplicable. For $S_N$ synchronization, the gap between nearby values of $s_{ij}^*$ decreases with $N$, so the theorem is less restrictive as $N$ increases. In order to address noisy situations for $\mathbb Z_2$ and $S_N$ with small $N$, one can assume instead an additive Gaussian noise model \cite{COLT_Montanari,deepti}. When the noise is sufficiently small and the graph is generated from the Erd\H{o}s-R\'{e}nyi model with sufficiently large probability of connection, projection of the noisy group ratios onto $\mathbb Z_2$ or $S_N$ results in a subset of uncorrupted group ratios whose proportion is sufficiently large (see e.g. \cite{deepti}), so that Theorems \ref{thm:it1} or \ref{thm:ir1} can be applied to the projected elements. 

\end{remark}

\subsection{Extension to Sub-Gaussian Noise}\label{sec:subg}
Here we directly extend the bounded noise stability of CEMP to sub-Gaussian noise.
We assume noisy adversarial corruption satisfying \eqref{eq:model}. 
We further assume that  $\{s_{ij}^*\}_{ij \in E_g}$ are independent and for $ij\in E_g$, $s_{ij}^*\sim sub(\mu,\sigma^2)$, namely, $s_{ij}^*$ is sub-Gaussian with mean $\mu$ and variance $\sigma^2$. More precisely, $s_{ij}^*=\sigma X_{ij}$ where $\Pr(X_{ij}-\mu>x)<\exp(-x^2/2)$ and $\Pr(X_{ij} \geq 0) = 1$.  The proof of Theorem \ref{thm:subg} is included in Appendix \ref{sec:proofs_subg}.

\begin{theorem}\label{thm:subg}
Assume data generated by the adversarial corruption model with independent sub-Gaussian noise having mean $\mu$ and variance $\sigma^2$. 
For any $x>0$, if one replaces 
$\lambda$ and $\delta$ in
Theorems~\ref{thm:it2} and~\ref{thm:ir2} with 
$\lambda+2e^{-\frac{x^2}{2}}$ and $\sigma\mu+\sigma x$, respectively, then the conclusions of these theorems hold with probability at least $1-|E|\exp(-\frac13 e^{-x^2/2}\min_{ij\in E}|N_{ij}|(1-\lambda))$.
\end{theorem}
\begin{remark}
The above probability is sufficiently large when $x$ is sufficiently small and when $\min_{ij\in E}|N_{ij}|$ is sufficiently large. We note that $\min_{ij\in E}|N_{ij}|>\min_{ij\in E}$ $|G_{ij}|>0$,
where the last inequality follows from the good-cycle condition. We expect $\min_{ij\in E}|N_{ij}|$ to depend on the size of the graph, $n$, and its density. To demonstrate this claim we note that if $G([n], E)$ is Erd\H{o}s-R\'{e}nyi with probability of connection $p$, then $\min_{ij\in E}|N_{ij}|\approx np^2$. 
\end{remark}

Theorem~\ref{thm:subg} tolerates less corruption than Theorems \ref{thm:it2} and~\ref{thm:ir2}. This is due to the fact that, unlike bounded noise, sub-Gaussian noise significantly extort 
the group ratios. Nevertheless, we show next that in the case of a graph generated by the  Erd\H{o}s-R\'{e}nyi model, the sub-Gaussian model may still tolerate a similar level of corruption as that in Theorems~\ref{thm:it2} and \ref{thm:ir2} by sacrificing the tolerance to noise. 
\begin{corollary}\label{co:subg}
Assume that $G([n], E)$ is generated by the Erd\H{o}s-R\'{e}nyi model with probability of connection $p$. If $s_{ij}^*\sim sub(\mu,\sigma^2)$ for $ij\in E_g$, then for any $\alpha>6$ and $n$ sufficiently large, Theorems~\ref{thm:it2} and~\ref{thm:ir2}, with $\lambda$ and $\delta$ replaced by 
$$\lambda_n= \lambda+\frac{12\alpha}{1-\lambda}\frac{\log(np^2)}{np^2} \ \text{ and } \ \delta_n=\sigma\mu+2\sigma \sqrt{\log\frac{(1-\lambda)np^2}{6\alpha \log(np^2)}},$$ respectively, hold with probability at least $1-O(n^{-\alpha/3+2})$.
\end{corollary}
Note that this corollary is obtained by setting $\exp(-x^2/2)=6\alpha\log( np^2)/((1-\lambda)np^2)$
in Theorem \ref{thm:subg} and noting that in
this case $\min_{ij\in E}|N_{ij}|\geq np^2/2$ with high
probability. We note that $\sigma$ needs to decay with $n$, in order to have bounded $\delta_n$. In particular, if $\sigma\lesssim 1/\sqrt{\log n}$ and $p$ is fixed, $\delta_n=O(1)$.

\section{Exact Recovery Under Uniform Corruption}\label{sec:uniform}
This section establishes exact recovery guarantees for CEMP under the uniform corruption model. Its main challenge is dealing with large values of $\lambda$, unlike the strong restriction on $\lambda$ in Theorems
\ref{thm:it1} and \ref{thm:ir1}.

Section \ref{subsec:uniform} describes the uniform corruption model. Section \ref{sec:information_theory} reviews exact recovery guarantees of other works under this model and the best information-theoretic asymptotic guarantees possible.
Section \ref{sec:main_result_simple} states the main results on the convergence of both CEMP-\ith and CEMP-B.
Section \ref{sec:sample_complexity} clarifies the sample complexity bounds implied by these theorems. Since these bounds are not sharp, Section \ref{sec:simple_estimator_sharp} explains how a simpler estimator that uses the cycle inconsistencies obtains sharper bounds.
Section \ref{sec:proofs_rand} includes the proofs of all theorems. Section \ref{sec:APP} exemplifies the technical quantities of the main theorems for specific groups of interest.

\subsection{Description of the Uniform Corruption Model}
\label{subsec:uniform}
We follow the uniform corruption model (UCM) of \cite{wang2013exact} and apply it for any compact group. It has three parameters: $n \in \mathbb{N}$, $0<p\leq 1$ and $0 \leq q <1$, and we thus refer to it as UCM$(n, p, q)$.

UCM$(n, p, q)$ assumes a graph $G([n],E)$ generated by the Erd\"{o}s-R\'{e}nyi model $G(n,p)$, where $p$ is the connection probability among edges.
It further assumes an arbitrary set of group elements $\{g_i^*\}_{i=1}^n$. Each group ratio is generated by the following model, 
where $\tilde g_{ij}$ is independently drawn from the Haar measure on $\mathcal G$ (denoted by Haar$(\mathcal G))$:
\begin{align*}
g_{ij}=\begin{cases}
g_{ij}^*, & \text{ w.p. } 1-q;\\
\tilde g_{ij}, & \text{ w.p. } q.
\end{cases}
\end{align*}
We note that the set of corrupted edges $E_b$ is thus generated in two steps. First, a set of candidates of corrupted edges, which we denote by $\tilde E_b$, is independently drawn from $E$ with probability $q$. Next, $E_b$ is independently drawn from $\tilde E_b$ with probability $1-p_0$, where $p_0=\Pr(u_\mathcal G= e_\mathcal G)$ for an arbitrarily chosen $u_\mathcal G\sim \text{Haar}(\mathcal G)$. It follows from the invariance property of the Haar measure that for any $ij \in E$, $p_0=\Pr(u_\mathcal G= g_{ij}^* )$.
Therefore, the probability that $g_{ij}$ is uncorrupted, $\Pr(ij\in E_g|ij\in E)$ is 
$q_*=1-q+qp_0$.
We further denote $q_{\min}=\min(q_*^2\,,1-q_*^2)$, $q_g=1-q$ and  $z_\mathcal G=\mathbb E(d_\mathcal G(u_\mathcal G,e_\mathcal G))$, where $u_\mathcal G\sim \text{Haar}(\mathcal G)$. 
For Lie groups, such as $SO(d)$, $p_0=0$,  $q_* = p_g$, $\Pr(ij\in E_b|ij\in E)=q$ and $E_b=\tilde E_b$.

\subsection{Information-theoretic and Previous Results of Exact Recovery for UCM}
\label{sec:information_theory}

We note that when $0 \leq q<1$ and $0<p \leq 1$, the asymptotic recovery problem for UCM is well-posed since $\Pr(d_{ij,k}=0|ij\in E_g)$ is greater than $\Pr(d_{ij,k}=0|ij\in E_b)$ and thus good and bad edges are distinguishable. 
Furthermore, when $q=1$ or $p=0$ the exact recovery problem is clearly ill-posed. It is thus desirable to consider the full range of parameters, $0 \leq q<1$ and $0<p \leq 1$, when studying the asymptotic exact recovery problem of a specific algorithm assuming UCM. It is also interesting to check the asymptotic dependence of the sample complexity (the smallest sample size needed for exact recovery) on $q$ and $p$ when $p \to 0$ and $q \to 1$.

For the special groups of interest in applications,
$\mathbb Z_2$, $S_N$, $SO(2)$ and $SO(3)$, it was shown in \cite{Z2Afonso2,Z2Afonso}, \cite{chen_partial,info_theoretic_sync}, \cite{singer2011angular} and \cite{info_theoretic_sync}, respectively,
that exact recovery is information-theoretically possible under UCM whenever
\begin{equation}
\label{eq:info_bound_z2}
n/\log n=\Omega(p^{-1}q_g^{-2}),
\end{equation}
where, for simplicity, for $S_N$ we omitted the dependence on $N$ (which is a factor of $1/N$). 
That is, ignoring logarithmic terms (and the dependence on $N$ for $S_N$), the sample complexity is 
$\Omega(p^{-1}q_g^{-2})$.

There are not many results of this kind for actual algorithms.
Bandeira \cite{Z2Afonso} and Cucuringu \cite{Z2} showed that SDP and Spectral, respectively, for $\mathbb Z_2$ synchronization achieve the information-theoretic bound in \eqref{eq:info_bound_z2}.
Chen and Cand\`es \cite{Chen_PPM} established a similar result for Spectral and the projected power method when  $\mathcal G = \mathbb Z_N$. 
Another similar result was established by \cite{chen_partial} for a variant of SDP 
when  $\mathcal G = S_N$. After the submission of this work, \cite{ling2020} extended the latter result for Spectral.

When $\mathcal G$ is a Lie group, methods that relax  \eqref{eq:frob} with $\nu=2$, such as Spectral and SDP, cannot exactly recover the group elements under UCM. Wang and Singer \cite{wang2013exact} showed that the global minimizer of the SDP relaxation of \eqref{eq:frob} with   $\nu=1$ and $\mathcal G = SO(d)$ achieves asymptotic exact  recovery under UCM when $q \equiv \Pr(ij\in E_b|ij\in E)<p_c$, where $p_c$ depends on $d$ (e.g., $p_c \leq 0.54$ and $p_c=O(d^{-1})$). Due to their limited range of $q$, they cannot estimate the sample complexity when $q \to 1$.
As far as we know, \cite{wang2013exact} is the only previous work that provides exact recovery guarantees (under UCM) for synchronization on Lie groups.

\subsection{Main Results} 
\label{sec:main_result_simple}

Section \ref{sec:main_result_simple_a}
establishes exact recovery guarantees under UCM, which are most meaningful when $q_*$ is sufficiently small. Section \ref{sec:main_result_simple_b}
sharpens the above theory by considering the complementary region of $q_*$ ($q_* \geq q_c$ for some $q_c>0$). 
The proofs of all theorems are in Section \ref{sec:proofs_rand}.

\subsubsection{Main Results when $q_*$ is Sufficiently Small
}\label{sec:main_result_simple_a}

The two exact recovery theorems below 
use different quantities: $z_{\mathcal G}$, $P_{\max}(x)$ and $V(x)$. We define these quantities before each theorem and later exemplify them for common groups in Section \ref{sec:APP}. 
For simplicity of their already complicated proofs, we use concentration inequalities that are sharper when $q_*$ is sufficiently small. Therefore, the resulting estimates for the simpler case where $q_*$ is large are not satisfying and are corrected in the next section.

The condition of the first theorem uses the cdf (cumulative density function) of the random variable $\max\{s_{ik}^*,s_{jk}^*\}$, where $ij\in E$ and $k\in B_{ij}$ are arbitrarily fixed. We denote this cdf by $P_{\max}$ and note that due to the model assumptions it is independent of $i$, $j$ and $k$.

\begin{theorem}\label{thm:rand1}
Let $0<r<1$, $0\leq q <1$,  $0<p \leq1$, $n \in \mathbb{N}$ and assume data generated by UCM($n,p,q$). If the parameters $\{\beta_t\}_{t\geq 0}$ of CEMP-A satisfy
\begin{align}
\label{eq:choices_beta0_1}
0<\frac{1}{\beta_0}-(1-q_g^2)z_\mathcal G\leq \frac{q_g^2}{4\beta_1}\text{ , } P_{\max}\left(\frac{2}{\beta_1}\right)< \frac{r}{32}\frac{q_*^2}{1-q_*^2}\text{ and } \frac{1}{\beta_{t+1}}=r\frac{1}{\beta_t} 
\end{align}
for $t\geq 1$, then with probability at least
\begin{align*}
1-n^2p\,e^{-\Omega\left((\frac{1}{\beta_0}-(1-q_g^2)z_\mathcal G)^2p^2n\right)}-n^2pe^{-\Omega(q_{\min}p^2n)}
\end{align*}
the estimates $\{s_{ij}(t)\}_{ij \in E}^{t \geq 1}$ of $\{s_{ij}^{*}\}_{ij \in E}$
computed by CEMP-A satisfy
\begin{align*}
\max_{ij\in E}|s_{ij}(t)-s_{ij}^*|<\frac{1}{\beta_1} r^{t-1} \ \text{ for all } \ t \geq 1.
\end{align*}
\end{theorem}

The second theorem uses the following notation. 
Let $Y$ denote the random variable $s_{ik}^*+s_{jk}^*$ for any arbitrarily fixed $ij\in E$ and $k\in B_{ij}$. We note that due to the model assumptions, $Y$ is independent of $i$, $j$ and $k$. 
Let $P$ denote the cdf  of $Y$ and $Q$ denote the corresponding quantile function, that is, the inverse of $P$. Denote $f_{\tau}(x) = e^{-\tau x+1}\tau x$, where $\tau \geq 0$ and define $V^*(x): [0, \infty) \to \mathbb{R}$ by $V^*(x) = \sup_{\tau>x}\text{Var}(f_\tau(Y))$, where $\text{Var}(f_{\tau}(Y))$ is the variance of  $f_\tau(Y)$ for any fixed $\tau$. Since $V^*(x)$ might be hard to compute, our theorem below is formulated with any function $V$, which dominates $V^*$, that is, $V(x)\geq V^*(x)$ for all $x\geq 0$.

\begin{theorem}\label{thm:rand2}
Let $0<r<1$, $0\leq q <1$,  $0<p \leq1$, 
\begin{align}\label{eq:nomega}
    \frac{n}{\log n}=\Omega\left( \frac{1-q_*^2}{p^2q_*^4 r^2}\right),
\end{align}
and assume data generated by UCM($n,p,q$). Assume further that either $s_{ij}^*$ for $ij \in E_b$
is supported on $[a,\infty)$, where $a\geq 1/(np^2(1-q_*^2))$, or $Q$ is differentiable and $Q'(x)/Q(x)\lesssim 1/x$ for $x<P(1)$.
If the parameters $\{\beta_t\}_{t\geq 0}$ for CEMP-B satisfy
\begin{align}
\label{eq:choices_beta0_1_2ndthm}
 0<\frac{1}{\beta_0}\leq \frac{q_g^2q_*^2}{16(1-q_*^2)}\frac{1}{\beta_1} \text{ , } V(\beta_1)<\frac{r}{32}\cdot\frac{q_*^2}{1-q_*^2}\text{ and }  \frac{1}{\beta_{t+1}}=r\frac{1}{\beta_t}  
\end{align}
for $t\geq 1$, then with probability at least
\begin{align}\label{eq:rand2prob}
1-n^2p\,e^{-\Omega\left(p^2n/\beta_0^2\right)}-n^2p\,e^{-\Omega\left(V(\beta_1)(1-q_*^2)p^2n\right)}-n^2pe^{-\Omega(q_{\min}p^2n)},
\end{align}
the estimates $\{s_{ij}(t)\}_{ij \in E}^{t \geq 1}$ of $\{s_{ij}^{*}\}_{ij \in E}$
computed by CEMP-B satisfy
\begin{align}\label{eq:thm2conclusion}
\max_{ij\in E}|s_{ij}(t)-s_{ij}^*|<\frac{1}{4\beta_1}r^{t-1} \ \text{ for all } \ t \geq 1.
\end{align}

\end{theorem}

We note that Theorem \ref{thm:rand1} requires that 
\begin{equation}\label{eq:sc1}
\frac{n}{\log{n}} = \Omega\left( \frac{1}{p^2} \cdot \max\left( \frac{1}{q_*^2} \, , \left(\frac{1}{\beta_0}-(1-q_g^2)z_\mathcal G\right)^{-2} \right) \right)    
\end{equation} 
in order to have a sufficiently large probability. Similarly, Theorem \ref{thm:rand2} requires the following lower bound on the sample size: 
\begin{equation}\label{eq:sc2}
\frac{n}{\log{n}} = \Omega\left( \frac{1}{p^2} \cdot \max\left(\frac{1}{q_*^2}\, ,\beta_0^2 \, , \frac{1}{(1-q_*^2) \, V(\beta_1)} \,, \frac{1-q_*^2}{p^2q_*^4 r^2} \right) \right).
\end{equation}
We will use these estimates in Section \ref{sec:sample_complexity} to bound the sample complexity.

\subsubsection{Main Results when $q_*$ is Sufficiently Large
}\label{sec:main_result_simple_b}

We tighten the estimates established in Section \ref{sec:main_result_simple_a} by considering two different regimes of $q_*$ divided by a fixed value $q_c$. For CEMP-A we let $q_c$ be any number in $(\sqrt3/2, 1)$.
For CEMP-B we let $q_c$ be any number in $(2/\sqrt5, 1)$. 
We restrict the results of Theorems \ref{thm:rand1} and \ref{thm:rand2} to the case $q_* < q_c$ and formulate below the following two simpler theorems for the case $q_* \geq q_c$.

\begin{theorem}\label{thm:rand1_full}
Let $0<r<1$, $0 \leq q <1$,  $0<p \leq 1$, $n \in \mathbb{N}$ and assume data generated by UCM($n,p,q$). Let $q_c$ be any number $\in (\sqrt3/2, 1)$ and $\Delta_q = q_c^2/2-3/8\in (0,1/8)$. For any $q_*\geq q_c$, if the parameters $\{\beta_t\}_{t\geq 0}$ of CEMP-A satisfy 
\begin{align}\label{eq:lambda_rand1}
    \frac{1}{4}-\Delta_q\leq\frac{1}{\beta_0}\leq 1 \text{ and } \frac{1}{\beta_{t+1}}=r\frac{1}{\beta_t} \text{ for } t\geq 0 \text{ and } 1-4\Delta_q <r<1,
\end{align}
then with probability at least
\begin{align}
\label{eq:prob_bound_rand1_full}
1-n^2p\,e^{-\Omega\left(np^2\Delta_q^2\right)}
\end{align}
the estimates $\{s_{ij}(t)\}_{ij \in E}^{t \geq 1}$ of $\{s_{ij}^{*}\}_{ij \in E}$
computed by CEMP-A satisfy
\begin{align*}
\max_{ij\in E}|s_{ij}(t)-s_{ij}^*|<\frac{1}{\beta_1} r^{t-1} \ \text{ for all } \ t \geq 1.
\end{align*}
\end{theorem}

\begin{theorem}\label{thm:rand2_full}
Let $0<r<1$, $0 \leq q <1$,  $0<p \leq 1$, $n \in \mathbb{N}$ and assume data generated by UCM($n,p,q$). Let $q_c$ be any number $\in (2/\sqrt5, 1)$ and $\Delta_q = q_c^2/2-2/5\in (0,1/10)$. For any $q_*\geq q_c$, if the parameters $\{\beta_t\}_{t\geq 0}$ of CEMP-B satisfy 
\begin{align}\label{eq:lambda_rand2}
    \frac{1}{5}-\Delta_q\leq\frac{1}{\beta_0} \text{ and } \frac{1}{\beta_{t+1}}=r\frac{1}{\beta_t} \text{ for } t\geq 0 \text{ and } \frac{4-20\Delta_q}{4+5\Delta_q} <r<1,
\end{align}
then with probability at least
\begin{align*}
1-n^2p\,e^{-\Omega\left(np^2\Delta_q^2\right)}
\end{align*}
the estimates $\{s_{ij}(t)\}_{ij \in E}^{t \geq 1}$ of $\{s_{ij}^{*}\}_{ij \in E}$
computed by CEMP-B satisfy
\begin{align*}
\max_{ij\in E}|s_{ij}(t)-s_{ij}^*|<\frac{1}{\beta_1} r^{t-1} \ \text{ for all } \ t \geq 1.
\end{align*}
\end{theorem}

Theorems \ref{thm:rand1} and \ref{thm:rand2} for the regime $q_* < q_c$ seem to express 
different conditions on $\{\beta_t\}_{t\geq 0}$ than those in Theorems \ref{thm:rand1_full} and \ref{thm:rand2_full} for the regime 
$q_* \geq q_c$. However, after carefully clarifying the corresponding conditions in Theorems \ref{thm:rand1} and \ref{thm:rand2} for specific groups of interests (see Section \ref{sec:APP}), one can formulate conditions that apply to both regimes. Consequently, one can formulate unified theorems (with the same conditions for any choice of $q_*$) for special groups of interest.

\subsection{Sample Complexity Estimates}
\label{sec:sample_complexity}

Theorems \ref{thm:rand1} and \ref{thm:rand2} imply upper bounds for the sample complexity of CEMP. However, these bounds depend on various quantities that are estimated in 
Section \ref{sec:APP} for the groups $\mathbb Z_2$, $S_N$,  $SO(2)$ and $SO(3)$, which are common in applications. Table \ref{tab:sample} below first summarizes the estimates of these quantities (only upper bounds of $P_{\max}(x)$ and $V(x)$ are needed, but for completeness we also include the additional quantity $z_{\mathcal G}$). 
It then lists the consequent upper bounds of the sample complexities of CEMP-A and CEMP-B, which we denote by SC-A and SC-B, respectively. At last, it lists the information-theoretic sample complexity bounds (discussed in Section \ref{sec:information_theory}), which we denote by SC-IT.

The derivation of the sample complexity bounds, SC-A and SC-B, requires an asymptotic lower bound of $\beta_1$ and an asymptotic upper bound of $1/\beta_0-(1-q_g^2)z_{\mathcal G}$ (or equivalently, a lower bound of $\beta_0$). Then, one needs to use these asymptotic bounds together with \eqref{eq:sc1} or \eqref{eq:sc2} to estimate SC-A or SC-B, respectively.
We demonstrate the estimation of SC-A for $\mathcal G = SO(2)$. Here we assume two bounds: $\beta_1 = \Omega(q_*^{-2})$ and $1/\beta_0-(1-q_g^2)z_{\mathcal G}\leq O(q_*^4)$.
We first note from Table \ref{tab:sample} that $P_{\max}(x) = O(x)$ and consequently the first bound  implies the required middle equation of \eqref{eq:choices_beta0_1}. The combination of both bounds with the fact that in this case $q_g = q_*$ and the obvious assumption $\beta_0>0$ yields the first equation of \eqref{eq:choices_beta0_1}. Incorporating both bounds into \eqref{eq:sc1} we obtain that a sufficient sample size $n$ for exact recovery w.h.p.~by CEMP-A satisfies 
$n/\log(n) = \Omega(p^{-2}q_*^{-8})$;
thus, the minimal sample for exact recovery w.h.p.~by CEMP-A is of order $O(p^{-2}q_*^{-8})$.

\begin{table}[H]
    \centering
    \begin{tabular}{|c|c|c|c|c|}
    \hline
      $\mathcal G$ & $\mathbb Z_2$ & $S_N$ & $SO(2)$ & $SO(3)$ \\\hline
       $z_{\mathcal G}$ & $\frac12$ & $\sum_{m=1}^N 
 \frac{m}{N!} {N\choose m}
 \left[\frac{m!}{e}\right]$&  $\frac12$ & $\frac12+\frac{2}{\pi^2}$ \\\hline
       $P_{\max}(x)$ & $\mathbf{1}_{\{x=1\}}$ & $\frac{1}{N!}+
\mathbf{1}_{\{m\geq 1\}}\sum_{l=1}^m \frac{1}{N!} {N \choose l}
\left[\frac{l!}{e}\right]$ & $O(x)$ & $O(x)$ \\\hline
       $V(x)$ & $O(e^{-x})$ & $O(e^{-2 x/N}  x^2/N^2)$ & $O(x^{-1})$ & $O(x^{-2})$      \\\hline
       SC-A &  $O(p^{-2}q_g^{-4})$   & $O(p^{-2}q_g^{-4})$ & $O(p^{-2} q_g^{-8})$ &  $O(p^{-2} q_g^{-8})$        \\\hline
       SC-B &    $O(p^{-2}q_g^{-4-\alpha})$&   $O(p^{-2}q_g^{-4-\alpha})$  &   $O(p^{-2}q_g^{-12})$ & $O(p^{-2}q_g^{-10})$  \\\hline
       SC-IT &     $\Omega(p^{-1}q_g^{-2})$ & $\Omega(p^{-1}q_g^{-2})$ & $\Omega(p^{-1}q_g^{-2})$ & $\Omega(p^{-1}q_g^{-2})$\\\hline           
    \end{tabular}
    \caption{Summary of estimates of the main quantities of Theorems \ref{thm:rand1} and \ref{thm:rand2} for the common groups in applications, and of the derived sample complexity bounds. SC-A and SC-B denote the sample complexity of CEMP-A and CEMP-B (ignoring log factors) and SC-IT denotes the information-theoretic sample complexity (ignoring log factors).}
    \label{tab:sample}
\end{table}

We remark that these asymptotic bounds were based on estimates for the regime $q_* < q_c$, but we can extend them for any $q_*$ and $p \to 0$. 
Indeed, when $q_*\geq q_c$, \eqref{eq:prob_bound_rand1_full} of Theorem \ref{thm:rand1_full} and the equivalent equation of Theorem \ref{thm:rand2_full} imply that the minimum sample required for CEMP is of order $\Omega(1/p^2)$. Clearly, this estimate coincides with all estimates in Table \ref{tab:sample}
when $q_* \geq q_c$.

Our upper bounds for the sample complexity are far from the information-theoretic ones. Numerical experiments in Section \ref{sec:experiment} may indicate a lower sample complexity of CEMP than these bounds, but still possibly higher than the information theoretic ones. We expect that one may eventually obtain the optimal dependence in $q_g$ for a CEMP-like algorithm, however, CEMP with three cycles is unable to improve the dependence on $p$ from $\Omega(1/p^2)$ to $\Omega(1/p)$. The issue is that when $\mathcal{C}=\mathcal{C}_3$, the expected number of good cycles per edge is  $np^2 q_g^2$, so that $n=\Omega(1/(p^2q_g^2))$. Indeed, the expected number of 3-cycles per edge is $np^2$ and  the expected fraction of good cycles is $q_g^2$. The use of higher-order cycles should improve the dependence on $p$, but may harm the dependence on $q_g$.

Despite the sample complexity gap, we are unaware of other estimates that hold for $q_*\to 0$ (recall that $q_*\to 0$ only for continuous groups). The current best result for $SO(d)$ synchronization appears in \cite{wang2013exact}. It only guarantees exact recovery for the global optimizer (not for an algorithm) for sufficiently large $q_*$ (e.g., $q_* > 0.5$ for $d=3$ and $q_* > 1-O(d^{-1})$ for large $d$).

\subsection{A Simple Estimator with the Optimal Order of $q_g$ for Continuous Groups}
\label{sec:simple_estimator_sharp}
We present a very simple and naive estimator for the corruption levels that uses cycle inconsistencies and achieves the optimal order of $q_g$ for continuous groups. 
We denote by ${\mode}_{D_{ij}}$ the mode of  $D_{ij}=\{d_{ij,k}\}_{k \in N_{ij}}$. The proposed simple estimates $\{\hat s_{ij}\}_{ij \in E}$ are
\begin{equation}\label{eq:mode}
    \hat s_{ij} = \mode_{D_{ij}} \text{ for } ij\in E.
\end{equation}
Their following theoretical guarantees are proved in Appendix \ref{sec:prop_mode}. 
\begin{proposition}
\label{prop:mode}
Let $0 \leq q <1$,  $0<p \leq 1$, $n \in \mathbb{N}$ such that $n/\log n\geq c/(p^2q_g^2)$ for some absolute constant $c\geq 10$. If $\mathcal G$ is a continuous group and the underlying dataset is generated by UCM($n,p,q$), then \eqref{eq:mode} yields exact estimates of $\{s_{ij}^*\}_{ij\in E}$ with probability at least $1- n^{-{2}/{15}}$.
\end{proposition}

We remark that although the naive estimator of \eqref{eq:mode} achieves tighter sample complexity bounds than CEMP in the very special setting of UCM, it suffers from the following limitations that makes it impractical to more general scenarios.
First of all, in real applications, all edges are somewhat noisy, so that all the elements in each fixed $D_{ij}$ are different and finding a unique mode is impossible. Second, the mode statistic is very sensitive to adversarial outliers. In particular, one can maliciously choose the outliers to form peaks in the histogram of each $D_{ij}$ that are different than $s_{ij}^*$. 

We currently cannot prove a similar guarantee for CEMP, but the phase transition plots of Section \ref{sec:num_uniform} seem to support a similar behavior. Nevertheless, the goal of presenting this estimator was to show that it is possible to obtain sharp estimates in $q_g$ by using cycle inconsistencies.

\subsection{Proofs of Theorems~\ref{thm:rand1}-\ref{thm:rand2_full}}
\label{sec:proofs_rand}

Section \ref{sec:prelim_ucm} formulates some preliminary results that are used in the main proofs.
Section \ref{sec:rand1} proves Theorem \ref{thm:rand1}, Section \ref{sec:rand2} proves Theorem \ref{thm:rand2} and Section \ref{sec:rand3} proves Theorems \ref{thm:rand1_full} and \ref{thm:rand2_full}.

\subsubsection{Preliminary Results}
\label{sec:prelim_ucm}
We present some results on the concentration of $\lambda$ and good initialization. The proofs of all results are in Appendix \ref{sec:proofs_prelim_ucm}.

We formulate a concentration property of the ratio of corrupted cycles, $\lambda_{ij}$, where $ij \in E$ (see  \eqref{def_lambda}), and the maximal ratio $\lambda$. 
\begin{proposition}\label{prop:lambda}
Let $0\leq q <1$,  $0<p \leq1$, $n \in \mathbb{N}$ and assume data generated by UCM($n,p,q$). For any $0<\eta<1$,
\begin{align}\label{eq:lambda1}
\Pr(|\lambda_{ij}-(1-q_*^2)|>\eta q_{\min})<2\exp\left(-\frac{\eta^2}{3}q_{\min}|N_{ij}|\right) \ \text{ for any fixed } \ ij \in E
\end{align}
and
\begin{align}
\label{eq:lambda2}
\Pr(|\lambda-(1-q_*^2)|>\eta q_{\min})<2|E|\exp\left(-\frac{\eta^2}{3}q_{\min}\min_{ij\in E}|N_{ij}|\right).
\end{align}
\end{proposition}

Proposition \ref{prop:lambda}
is not useful when $q_*\approx 1$, since then $|N_{ij}|$ needs to be rather large, and this is counter-intuitive when there is hardly any corruption. 
On the other hand, this proposition is useful when $q_*$ is sufficiently small. In this case, if $|N_{ij}|$ is sufficiently large, then  $\lambda_{ij}$ concentrates around $1-q_*^2$. In particular, with high probability $\lambda$ can be sufficiently high. The regime of sufficiently high $\lambda$ is interesting and challenging, especially as Theorems \ref{thm:it1} and \ref{thm:ir1} do not apply then.

The next concentration result is useful when $q_*$ is sufficiently large. 
 
\begin{proposition}\label{prop:lambda2}
Let $0\leq q <1$,  $0<p \leq1$, $n \in \mathbb{N}$ and assume data generated by UCM($n,p,q$). For any $x\in (0,1]$, $q_*^2>1-x$  and $ij\in E$, 
\begin{align*}
\Pr(\lambda_{ij}>x)<\exp\left(-\frac{1}{3}\left(1-\frac{1-x}{q_*^2} \right)^2q_*^2|N_{ij}|\right).
\end{align*}
\end{proposition}

Next, we show that the initialization suggested in \eqref{eq:initial_corruption} is good under the uniform corruption model. We first claim that it is good on average, while using the notation $z_\mathcal G$ of Section \ref{subsec:uniform}. 
\begin{proposition}
\label{prop:initial1}
Let $0\leq q <1$,  $0<p \leq1$, $n \in \mathbb{N}$ and assume data generated by UCM($n,p,q$).  For any $ij \in E$, $s_{ij}(0)$ is a scaled and shifted version of $s_{ij}^*$ as follows
\begin{align}
\label{eq:initial1}
\mathbb E(s_{ij}(0))=q_g^2 s_{ij}^*+(1-q_g^2) z_\mathcal G.
\end{align}

\end{proposition}

At last, we formulate the concentration of $s_{ij}(0)$
around its expectation. It follows from direct application of Hoeffding's inequality, while using the fact that $0\leq d_{ij,k}\leq 1$ are i.i.d.
\begin{proposition}\label{lemma:init}
Let $0\leq q <1$,  $0<p \leq1$, $n \in \mathbb{N}$ and assume data generated by UCM($n,p,q$). Then,
\begin{align*}
\Pr\left(\left|s_{ij}(0)-\mathbb E(s_{ij}(0))\right|>\gamma\right)&<
 2e^{-2\gamma^2|N_{ij}|}.
\end{align*}
\end{proposition}

\subsubsection{Proof of Theorem~\ref{thm:rand1}}\label{sec:rand1}

This proof is more involved than previous ones. Figure \ref{fig:roadmap_A} thus provides a simple roadmap for following it.
\begin{figure}[H]
\begin{center}
  \includegraphics[width=1\linewidth]{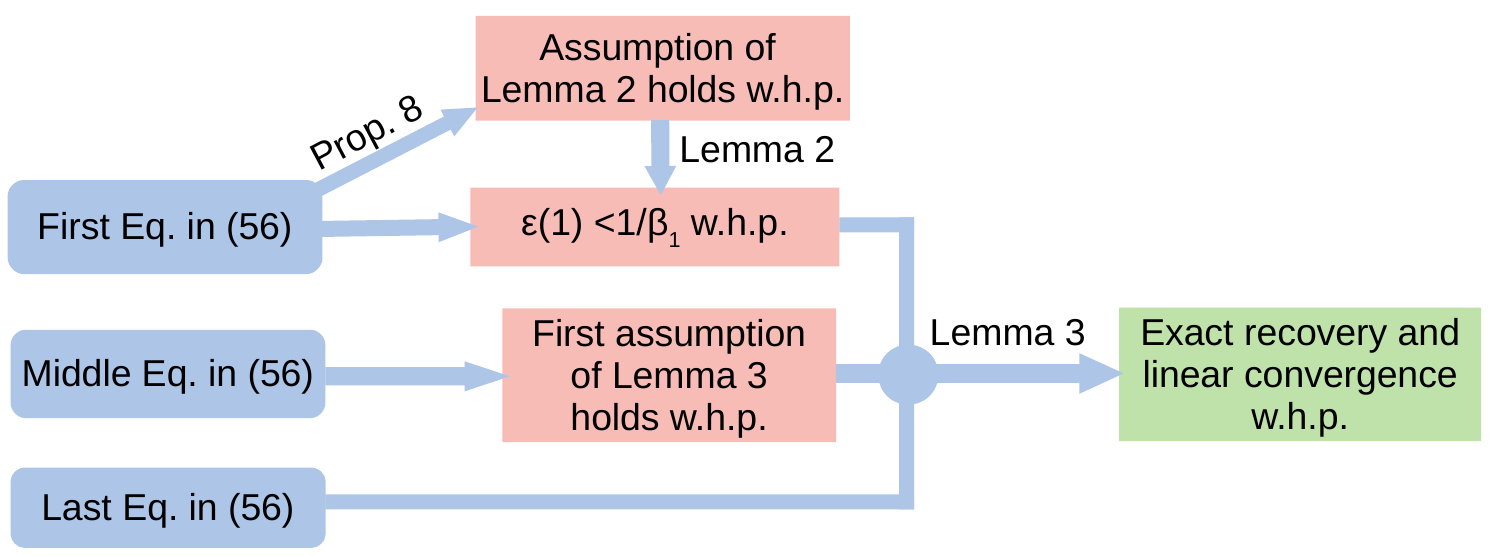}
\end{center}
   \caption{Roadmap for the proof of Theorem \ref{thm:rand1}.  }\label{fig:roadmap_A}
\end{figure}

The proof frequently uses the notation 
$$A_{ij}(x;t)=\{k\in N_{ij}: s_{ik}(t), s_{jk}(t)\leq x\}$$ and 
$$A_{ij}^*(x)=\{k\in N_{ij}: s_{ik}^*, s_{jk}^*\leq x\}.$$ It relies on the following two lemmas.
\begin{lemma}\label{lemma:eps1_iter}
If $1/\beta_0\geq (1-q_g^2)z_\mathcal G+\max_{ij\in E} \left| s_{ij}(0)-\mathbb E s_{ij}(0)\right|$,
then 
\begin{align}\label{eq:eps1_iter}
\epsilon (1)\leq 4\frac{\frac{1}{\beta_0}-(1-q_g^2)z_\mathcal G}{q_g^2}.
\end{align}
\end{lemma}
\begin{proof}
We use the following upper bound on $\epsilon_{ij}(1)$, which is obtained by plugging $t=0$ into \eqref{eq:estimate_eps}
\begin{align}\label{eq:ep1}
\epsilon_{ij}(1)&\leq\frac{\sum\limits_{k\in N_{ij}}\mathbf{1}_{\{s_{ik}(0), s_{jk}(0)\leq \frac{1}{\beta_0}\}}(s_{ik}^*+ s_{jk}^*)}{\sum\limits_{k\in N_{ij}}\mathbf{1}_{\{s_{ik}(0), s_{jk}(0)\leq \frac{1}{\beta_0}\}}}.
\end{align}
Denote $\gamma_{ij} = |s_{ij}(0) - \mathbb E(s_{ij}(0))|$ 
for $ij \in E$ and $\gamma=\max_{ij \in E} \gamma_{ij}$, so that the condition of the lemma can be written more simply as 
${1}/{\beta_0}\geq (1-q_g^2)z_\mathcal G+\gamma$.
We use \eqref{eq:initial1} to write $s_{ij}(0)=q_g^2s_{ij}^*+(1-q_g^2)z_\mathcal G+\gamma_{ij}$ and thus conclude that $s_{ij}(0) \geq q_g^2s_{ij}^*+(1-q_g^2)z_\mathcal G-\gamma$.
Consequently, if $s_{ij}(0)<\frac{1}{\beta_0}$ for $ij \in E$, then $s_{ij}^*<(1/\beta_0-(1-q_g^2)z_\mathcal G)+\gamma)/q_g^2$. The combination of the latter observation with \eqref{eq:ep1} results in
\begin{align*}
\epsilon_{ij}(1)<2 \, \frac{\frac{1}{\beta_0}+\gamma-(1-q_g^2)z_\mathcal G}{q_g^2}.
\end{align*}
Applying the assumption $1/\beta_0\geq (1-q_g^2)z_\mathcal  G+\gamma$ into the above equation, while also maximizing the LHS of this equation over $ij \in E$, results in \eqref{eq:eps1_iter}.
\qed
\end{proof}

\begin{lemma}\label{lemma:deterministic1}
Assume that $|A_{ij}^*(2/\beta_1)\setminus G_{ij}|/|B_{ij}|\leq(1-\lambda)r/(4\lambda)$ for all $ij\in E$,  $\frac{1}{\beta_1}>\epsilon (1)$ and $\beta_{t}=r\beta_{t+1}$ for all $t\geq 1$. Then, 
the estimates $\{s_{ij}(t)\}_{ij \in E}^{t \geq 1}$ 
computed by CEMP-A satisfy
\begin{align}
\label{eq:lemma:deterministic1}
\max_{ij\in E}|s_{ij}(t)-s_{ij}^*|<\frac{1}{\beta_1} r^{t-1}
\ \text{ for all } \ t \geq 1.
\end{align}
\end{lemma}
\begin{proof}
We prove \eqref{eq:lemma:deterministic1}, equivalently, $\epsilon (t)<{1}/{\beta_t}$ for all $t\geq 1$, by induction. We note that $\epsilon (1)<{1}/{\beta_1}$ is an assumption of the lemma.  We next show that $\epsilon (t+1)<{1}/{\beta_{t+1}}$ if  $\epsilon (t)<{1}/{\beta_t}$. We note that applying  \eqref{eq:inclusion} and then \eqref{eq:if_then} result in the following two inclusions
\begin{align}\label{eq:inclusion1}
G_{ij}\subseteq A_{ij}\left(\frac{1}{\beta_t};t\right) \subseteq A_{ij}^*\left(\frac{2}{\beta_t}\right) \ \text{ for } ij\in E  \text{ and } t\geq 1.
\end{align}
Applying first \eqref{eq:estimate_eps}, then \eqref{eq:inclusion1} and at last the definition of $\lambda$,
we obtain that for any given $ij\in E$
\begin{multline*}
\epsilon_{ij}(t+1)
\leq\frac{\sum\limits_{k\in B_{ij}}\mathbf{1}_{\{s_{ik}(t), s_{jk}(t)\leq \frac{1}{\beta_t}\}}(s_{ik}^*+ s_{jk}^*)}{\sum\limits_{k\in N_{ij}}\mathbf{1}_{\{s_{ik}(t), s_{jk}(t)\leq \frac{1}{\beta_t}\}}}
\\
\leq  4\frac{1}{\beta_t}\frac{|A_{ij}^*(\frac{2}{\beta_t})\setminus G_{ij}|}{|G_{ij}|}
\leq 4\frac{1}{\beta_t}\frac{|A_{ij}^*(\frac{2}{\beta_t})\setminus G_{ij}|}{|B_{ij}|}\frac{\lambda}{1-\lambda}.
\end{multline*}
Combining the above equation with the assumption $|A_{ij}^*(2/\beta_1)\setminus G_{ij}|/|B_{ij}|\leq(1-\lambda)r/(4\lambda)$ yields
\begin{align*}
\epsilon_{ij}(t+1)
\leq 4\frac{1}{\beta_t}\frac{|A_{ij}^*(\frac{2}{\beta_t})\setminus G_{ij}|}{|B_{ij}|}\frac{\lambda}{1-\lambda}\leq r\frac{1}{\beta_t}=\frac{1}{\beta_{t+1}}.
\end{align*}
Maximizing over $ij \in E$ the LHS of the above equation concludes the induction and the lemma. 
\qed
\end{proof}
To conclude the theorem, it is sufficient to show that under its setting, the first two assumptions of Lemma \ref{lemma:deterministic1} hold w.h.p.

We first verify w.h.p.~the condition $\max_{ij\in E}|A_{ij}^*(2/\beta_1)\setminus G_{ij}|/|B_{ij}|\leq(1-\lambda)r/(4\lambda)$. We note that  for each fixed $ij\in E$, $\{\mathbf 1_{\{k\in A_{ij}^*(2/\beta_1)\setminus G_{ij}\}}\}_{k\in B_{ij}}$ is a set of i.i.d.~Bernoulli random variables with mean $P_{\max}(2/\beta_1)$. 
We recall the following one-sided Chernoff bound for independent Bernoulli random variables $\{X_l\}_{l=1}^n$ with means $\{p_l\}_{l=1}^n$, $\bar p=\sum_{l=1}^n p_l/n$, and any $\eta>1$:
\begin{align}\label{eq:chernoff1}
    \Pr\left(\frac{1}{n} \sum_{l=1}^n X_l>(1+\eta)\bar p\right)<e^{-\frac{\eta}{3}\bar p n}.
\end{align}
Applying \eqref{eq:chernoff1}  with the random variables 
$\{\mathbf 1_{\{k\in A_{ij}^*(2/\beta_1)\setminus G_{ij}\}}\}_{k\in B_{ij}}$ whose means are $P_{\max}(2/\beta_1)$ and with $\eta=(1-\lambda)r/(4\lambda P_{\max}(2/\beta_1))-1$, and then assuming that $P_{\max}(2/\beta_1)<(1-\lambda)r/(8\lambda)$, results in
\begin{align}
\label{eq:prob_cond_1_before}
\Pr\left(\frac{|A_{ij}^*(\frac{2}{\beta_1})\setminus G_{ij}|}{|B_{ij}|}>\frac{1-\lambda}{4\lambda}r\right)<e^{-\frac{1}{3}\left(\frac{1-\lambda}{4\lambda}r-P_{\max}(\frac{2}{\beta_1})\right)|B_{ij}|}<e^{-\frac{1-\lambda}{24\lambda}r|B_{ij}|}.
\end{align}
We next show that the above assumption, $P_{\max}(2/\beta_1)<(1-\lambda)r/(8\lambda)$, holds w.h.p.~and thus verify w.h.p.~the desired condition.
We recall that 
$P_{\max}({2}/{\beta_1})< {r q_*^2}/{(32 (1-q_*^2))}$ (see
\eqref{eq:choices_beta0_1}).
Furthermore, by Proposition \ref{prop:lambda},  
\begin{equation}
\label{eq:snr_via_ratiosq}
\Pr\left(\frac14\frac{q_*^2}{1-q_*^2}<\frac{1-\lambda}{\lambda}<4\frac{q_*^2}{1-q_*^2}\right) \geq  1-2|E|\exp(-\Omega(q_{\min}|N_{ij}|)).     
\end{equation}
The latter two observations result in the needed bound on $P_{\max}({2}/{\beta_1})$ w.h.p. More generally, these observations and \eqref{eq:prob_cond_1_before} with a union bound over $ij\in E$ imply w.h.p.~the desired condition as follows
\begin{multline}
    \label{eq:beta1prob}
    \Pr \left( \max_{ij\in E}\frac{|A_{ij}^*(\frac{2}{\beta_1})\setminus G_{ij}|}{|B_{ij}|}\leq\frac{1-\lambda}{4\lambda}r \right)\\
    \geq
    1-|E| e^{-\Omega\left(\frac{q_*^2}{(1-q_*^2)}r\min\limits_{ij\in E}|B_{ij}|\right)}-|E|e^{-\Omega(q_{\min}\min\limits_{ij\in E}|N_{ij}|)}.
\end{multline}

To guarantee w.h.p.~the other condition, $1/\beta_1>\epsilon (1)$, we note that if the condition of Lemma~\ref{lemma:eps1_iter} holds, then an application of the conclusion of this Lemma and another application of the first equation in \eqref{eq:choices_beta0_1} imply the desired condition, that is,
\begin{align}
\epsilon (1)\leq 4\frac{\frac{1}{\beta_0}-(1-q_g^2)z_\mathcal G}{q_g^2}\leq \frac{1}{\beta_1} \text{\, if \,}  \frac{1}{\beta_0}\geq (1-q_g^2)z_\mathcal G+\max_{ij\in E} \left| s_{ij}(0)-\mathbb E s_{ij}(0)\right|.
\end{align}
In order to verify w.h.p.~the condition of Lemma~\ref{lemma:eps1_iter}, we apply Proposition~\ref{lemma:init} with $\gamma= 1/\beta_0-(1-q_g^2)z_\mathcal G$ (note that $\gamma>0$ by the first inequality of \eqref{eq:choices_beta0_1})  and a union bound over $ij\in E$ to obtain that 
\begin{multline}
\label{eq:beta0prob}
    \Pr\left(\max_{ij\in E} \left| s_{ij}(0)-\mathbb E s_{ij}(0)\right|<\frac{1}{\beta_0}-(1-q_g^2)z_\mathcal G \right)
    \\
    \geq  1-|E|e^{-\Omega\left(q_{\min}\min\limits_{ij\in E}|N_{ij}|\right)}-|E|\,e^{-\Omega\left(\left(\frac{1}{\beta_0}-(1-q_g^2)z_\mathcal G\right)^2\min\limits_{ij\in E}|N_{ij}|\right)}.
\end{multline}

We recall the following Chernoff bound for i.i.d.~Bernoulli random variables $\{X_l\}_{l=1}^m$ with means $\mu$ and any $0<\eta<1$:
\begin{align}\label{eq:chernoff2}
    \Pr\left(\left|\frac{1}{m} \sum_{l=1}^m X_l-\mu\right|>\eta\mu\right)<2e^{-\frac{\eta^2}{3}\mu m}.
\end{align}
We note that by applying \eqref{eq:chernoff2} three different times with the following random variables: $\{\mathbf{1}_{\{ij\in E\}}\}_{i, j\in [n]}$, where $\mu = p$, $m=n^2$; $\{\mathbf{1}_{\{k\in N_{ij}\}}\}_{k\in [n]}$, where $\mu = p^2$, $m=n$  (for each fixed $ij\in E$); and  $\{\mathbf{1}_{\{k\in B_{ij}\}}\}_{k\in [n]}$, where $\mu = p^2(1-q_*^2)$, $m=n$ (for each fixed $ij\in E$), and then a union bound, we obtain that with probability at least $1-\exp(-\Omega(n^2p))-n^2p$ $\exp(-\Omega(np^2))-n^2p\exp(-\Omega(np^2q_{\min}))$, or equivalently, $1-n^2p\exp(-\Omega(np^2q_{\min}))$, the following events hold: $|E|\lesssim n^2p$, 
$\min_{ij\in E}|N_{ij}|\gtrsim np^2$ and $\min_{ij\in E}|B_{ij}|\gtrsim np^2(1-q_*^2)$ . We conclude the proof by combining this observation,  \eqref{eq:beta1prob}-\eqref{eq:beta0prob}  and  Lemma \ref{lemma:deterministic1}.

\subsubsection{Proof of Theorem~\ref{thm:rand2}}\label{sec:rand2}

This proof is similar to that of Theorem \ref{thm:rand2}, but it is more difficult  
since it requires additional tools from empirical risk minimization (see Lemma \ref{lemma:risk}).
Figure \ref{fig:roadmap_B} provides a roadmap for following the proof.

\begin{figure}[H]
\begin{center}
  \includegraphics[width=1\linewidth]{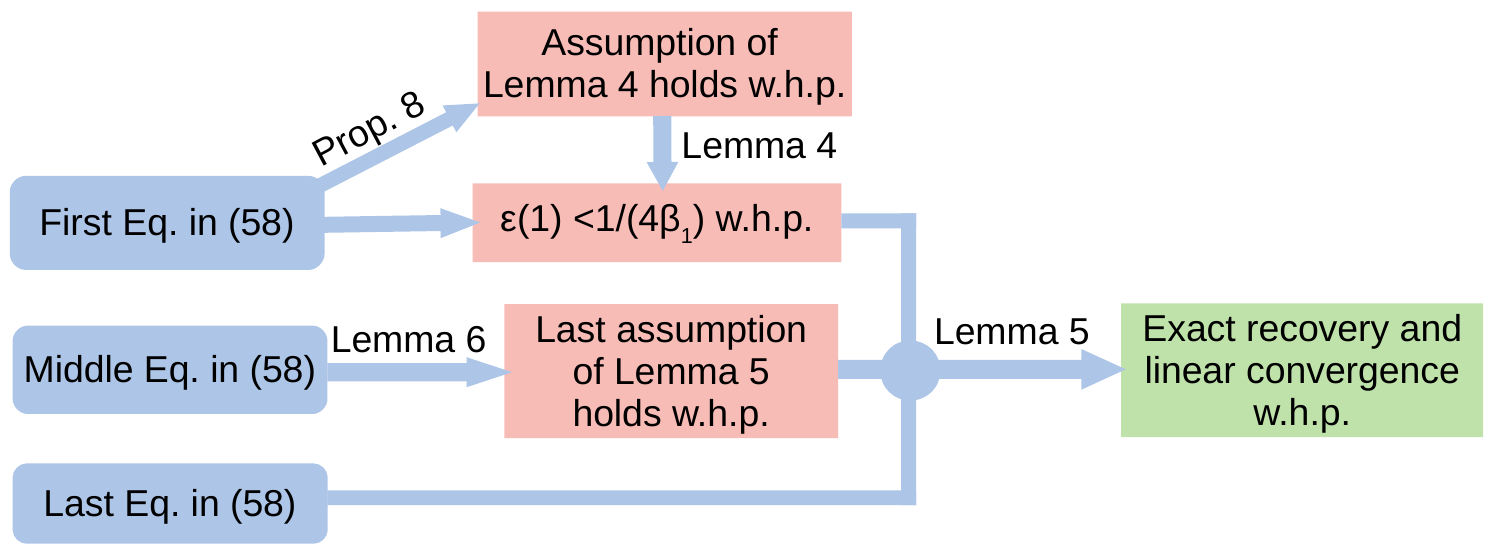}
\end{center}
   \caption{Roadmap for the proof of Theorem \ref{thm:rand2}. }\label{fig:roadmap_B}
\end{figure}

The proof of the theorem relies on the following three lemmas.
\begin{lemma}\label{lemma:tau1}
If $1/(4\beta_0)\geq \max_{ij\in E} \left| s_{ij}(0)-\mathbb E s_{ij}(0)\right|$,
 then
\begin{align*}
\epsilon (1)\leq \frac{\lambda}{1-\lambda} \frac{1}{q_g^2\beta_0}.
\end{align*}
\end{lemma}

\begin{lemma}\label{lemma:M}
 Assume that  $1/(4\beta_1)>\epsilon (1)$,  $\beta_{t}=r\beta_{t+1}$ for $t\geq 1$,
 and
\begin{equation}\label{eq:Mbeta}
\max_{ij\in E}\frac{1}{|B_{ij}|}\sum\limits_{k\in B_{ij}}e^{-\beta_t\left(s_{ik}^*+s_{jk}^*\right)}\left(s_{ik}^*+ s_{jk}^*\right)<\frac{1}{M\beta_t} \ \text{ for all } \ t \geq 1,
\end{equation}
where $M=4e\lambda/((1-\lambda)r)$. Then, 
the estimates $\{s_{ij}(t)\}_{ij \in E}^{t \geq 1}$ 
computed by CEMP-B satisfy
\begin{align}
\label{eq:lemma:M}
\max_{ij\in E}|s_{ij}(t)-s_{ij}^*|<\frac{1}{\beta_1} r^{t-1}
\ \text{ for all } \ t \geq 1.
\end{align}
\end{lemma}

The last lemma uses the notation $\mathcal F=\{f_\tau(x):  \tau>\beta\}$, where we recall that $f_\tau(x) = e^{-\tau x+1}\tau x$.
\begin{lemma}\label{lemma:risk}
If either $s_{ij}^*$ for $ij \in E_b$ is 
supported on $[a,\infty)$ and $a\geq 1/|B_{ij}|$ or $Q$ is differentiable and $Q'(x)/Q(x)\lesssim 1/x$ for $x<P(1)$, then there exists an  absolute constant $c$ such that
\begin{align}\label{eq:riskH}
\Pr\left(\sup_{f_\tau\in \mathcal F(\beta)}\frac{1}{|B_{ij}|}\sum_{k\in |B_{ij}|} f_\tau (s_{ik}^*+s_{jk}^*)>V(\beta)+c\sqrt{\frac{\log |B_{ij}|}{|B_{ij}|}}\right)<e^{-\frac{1}{3}mV(\beta)}.
\end{align}
\end{lemma}

The proofs of Lemmas~\ref{lemma:tau1} and~\ref{lemma:M} are similar to the ones of Lemmas \ref{lemma:eps1_iter} and \ref{lemma:deterministic1}. For completeness, we include them in Appendices~\ref{sec:lemma:tau1} and~\ref{sec:lemma:M}, respectively. The proof of Lemma~\ref{lemma:risk} requires tools from empirical risk minimization, and we thus provide it later in Appendix~\ref{sec:risk}.

According to Lemma~\ref{lemma:M}, the theorem follows by guaranteeing w.h.p.~the following two conditions of this lemma:  
\eqref{eq:Mbeta}  and
${1}/{4\beta_1}>\epsilon (1)$. 
We note that \eqref{eq:Mbeta} is guaranteed w.p.~at least $1-\exp(-\Omega(V(\beta_1)|B_{ij}|))$ by applying Lemma~\ref{lemma:risk} with $\beta_1$ such that $V(\beta_1)<{e}/{2M}$ and $|B_{ij}|$ sufficiently large such that $\sqrt{\log |B_{ij}|/|B_{ij}|}<{e}/{2cM}$. The combination of the middle inequality of \eqref{eq:choices_beta0_1_2ndthm} and \eqref{eq:snr_via_ratiosq} implies that  $V(\beta_1)<e/(2M)=(1-\lambda)r/(8\lambda)$ with the same probability as in \eqref{eq:snr_via_ratiosq}.
We note that \eqref{eq:snr_via_ratiosq} implies that if 
$|B_{ij}|/\log |B_{ij}|\gtrsim \left((1-q_*^2)/(q_*^2 r)\right)^2$, then with the probability specified in \eqref{eq:snr_via_ratiosq}, $\sqrt{\log |B_{ij}|/|B_{ij}|}<{e}/{2cM}$ holds. We recall that $|B_{ij}|\gtrsim np^2(1-q_*^2)$ with probability $1-\exp(\Omega(np^2q_{\min}))$. Combining this observation with \eqref{eq:nomega} concludes w.h.p.~the desired bound, that is,   
$|B_{ij}|/\log |B_{ij}|\gtrsim \left((1-q_*^2)/(q_*^2 r)\right)^2$. In summary, 
 \eqref{eq:Mbeta} holds with probability at least $1-\exp(-\Omega(V(\beta_1)|B_{ij}|))-\exp(\Omega(np^2q_{\min}))$.

Next we verify w.h.p.~the other condition of Lemma \ref{lemma:M}, namely, $1/(4\beta_1)>\epsilon (1)$. If the assumption of Lemma~\ref{lemma:tau1} holds,  then application of the conclusion of this lemma and following up with a combination of the first equation in \eqref{eq:choices_beta0_1_2ndthm} and \eqref{eq:snr_via_ratiosq} yield (with the probability specified in \eqref{eq:snr_via_ratiosq})
\begin{align*}
\epsilon (1)\leq \frac{\lambda}{1-\lambda} \frac{1}{q_g^2\beta_0}<\frac{1}{4\beta_1}
\quad \text{ if } \quad \frac{1}{4\beta_0}\geq \max_{ij\in E} \left| s_{ij}(0)-\mathbb E s_{ij}(0)\right|.
\end{align*}
In order to verify w.h.p.~the assumption of Lemma~\ref{lemma:tau1}, we apply Proposition~\ref{lemma:init} with $\gamma=1/(4\beta_0)$ (note that $\gamma>0$ by the first inequality of \eqref{eq:choices_beta0_1_2ndthm}) together with a union bound over $ij\in E$ to obtain that 
\begin{multline*}
    \Pr\left(\max_{ij\in E} \left| s_{ij}(0)-\mathbb E s_{ij}(0)\right|<\frac{1}{4\beta_0} \right)\\
    \geq  1-|E|e^{-\Omega\left(q_{\min}\min\limits_{ij\in E}|N_{ij}|\right)}-|E|\,e^{-\Omega\left(\frac{1}{\beta_0^2}\min\limits_{ij\in E}|N_{ij}|\right)}.
\end{multline*}
These arguments and our earlier observation that $|E|\lesssim n^2p$, $\min_{ij\in E}|B_{ij}|\gtrsim np^2(1-q_*^2)$ and $\min_{ij\in E}|N_{ij}|\gtrsim np^2$ w.p.~$1-n^2p\exp(-\Omega(np^2q_{\min}))$ conclude the proof.

\subsubsection{Proof of Theorems \ref{thm:rand1_full} and \ref{thm:rand2_full}}
\label{sec:rand3}
We prove Theorem \ref{thm:rand1_full}, whereas the proof of Theorem \ref{thm:rand2_full} is identical. Note that  \eqref{eq:lambda_rand1} describes the same conditions of Theorem \ref{thm:it1} with $\lambda$ replaced by $ 1/4-\Delta_q$. Thus, it suffices to prove that $\lambda\leq 1/4-\Delta_q$ with the probability 
specified in \eqref{eq:prob_bound_rand1_full}. This implies the conclusion of Theorem \ref{thm:it1} with the latter probability, or equivalently, the conclusion of Theorem \ref{thm:rand1_full}.
Applying Proposition \ref{prop:lambda2} with $x=1/4-\Delta_q$ and $q_*=\Omega(1)$ ($q_* \geq q_c$), and then the fact that $\Delta_q=\frac14-\Delta_q-(1-q_c^2)\leq \frac14-\Delta_q-(1-q_*^2)$, 
we obtain that 
\begin{align}\label{eq:prob_rand1_full_proof}
&\Pr\left(\lambda_{ij}>\frac14-\Delta_q\right)<\exp\left(-\Omega\left(\left((1/4-\Delta_q)-(1-q_*^2)\right)^2|N_{ij}|\right)\right) \nonumber\\
&\leq \exp\left(-\Omega\left(\Delta_q^2|N_{ij}|\right)\right) \ 
\text{ for any } \ ij\in E.
\end{align}
We note that application of Chernoff bound in \eqref{eq:chernoff2} twice, first with the random variables $\{\mathbf{1}_{\{ij\in E\}}\}_{i, j\in [n]}$, where $\mu = p$, $m=n^2$,  then with  $\{\mathbf{1}_{\{k\in N_{ij}\}}\}_{k\in [n]}$, where $\mu = p^2$, $m=n$  (for each fixed $ij\in E$), and then a union bound, yields that with probability at least $1-\exp(-\Omega(n^2p))-n^2p\exp(-\Omega(np^2))$, or equivalently, $1-n^2p\exp(-\Omega(np^2))$, the following events hold: $|E|\lesssim n^2p$ and 
$\min_{ij\in E}|N_{ij}|\gtrsim np^2$. Combining this observation, \eqref{eq:prob_rand1_full_proof} and a union bound over $ij\in E$ results in the desired probability bound of \eqref{eq:prob_bound_rand1_full} for the event 
$\lambda\leq 1/4-\Delta_q$.

\subsection{Clarification of Quantities Used in Theorems \ref{thm:rand1} and \ref{thm:rand2}}\label{sec:APP}
Theorems \ref{thm:rand1} and \ref{thm:rand2} use the quantities $P_{\max}(x)$, $z_\mathcal G$, $V(x)$ and $Q(x)$. In this section, we  provide explicit expressions for these quantities for common group synchronization problems. We also verify that the special condition of Theorem \ref{thm:rand2} holds in these cases. This special condition is that either $s_{ij}^*$ for $ij \in E_b$ is supported on $[a,\infty)$, where $a\geq 1/(np^2(1-q_*^2))$, or $Q$ is differentiable and $Q'(x)/Q(x)\lesssim 1/x$ for $x<P(1)$. 
When using the first part of this condition, then $Q$ is not needed and we will thus not specify it in this case.
We recall that $Y$ denotes the random variable $s_{ik}^*+s_{jk}^*$ for any arbitrarily fixed $ij\in E$ and $k\in B_{ij}$.

\subsubsection{$\mathbb{Z}_2$ Synchronization}
\label{sec:Z2_demonstration}
In this problem, $\mathcal G=\mathbb Z_2$, which is commonly represented by $\{-1,1\}$ with ordinary multiplication. It is common to use the bi-invariant metric $d_{\mathcal G}(g_1,g_2)=|g_1-g_2|/2$ and thus $d_{ij,k}=|g_{ij}g_{jk}g_{ki}-1|/2\in \{0,1\}$. The Haar measure on $\mathbb Z_2$ is the Rademacher distribution. 

We note that $z_\mathcal G=1/2$ and 
$P_{\max}(x)=\mathbf{1}_{\{x=1\}}$
(since for $k\in B_{ij}$, $\max\{s_{ik}^*, s_{jk}^*\}$ $=1$). We next show that $V(x)=6e^{-x}$. Indeed, $Y=1,2$ w.p.~$p_1=2q_*(1-q_*)/(1-q_*^2)$ and $p_2=(1-q_*)^2/(1-q_*^2)$, respectively, and thus  
\begin{align*}
    \sup_{\tau>x}\text{Var}(f_\tau(Y))&\leq e^2\sup_{\tau>x}\mathbb E(e^{-2\tau Y}\tau^2 Y^2)
    =e^2\sup_{\tau>x}(p_1 e^{-2\tau}\tau^2+4p_2 e^{-4\tau}\tau^2)\nonumber\\
    &\leq  \mathbf 1_{\{0<x<1\}}+\left(e^2(p_1 e^{-2x}x^2+ 4p_2 e^{-4x}x^2)\right)\mathbf 1_{\{x>1\}}\nonumber\\
    &\leq  \mathbf 1_{\{0<x<1\}}+e^2\max\left( e^{-2x}x^2, 4 e^{-4x}x^2\right)\mathbf 1_{\{x>1\}}<6e^{-x}.
\end{align*}
Since $s^*_{ij}$ for $ij \in E_b$ is supported on $\{1\}$, the special condition of Theorem \ref{thm:rand2} holds when $n = \Omega(1/(p^2(1-q_*^2)))$. This latter asymptotic bound is necessary so that the third term in \eqref{eq:rand2prob} is less than 1.

\subsubsection{Permutation Synchronization}

In this problem, $\mathcal G=S_N$, whose elements are commonly represented by permutation matrices in  $\mathbb R^{N\times N}$. A common bi-invariant metric on $S_N$ is $d_\mathcal G (\bP_1, \bP_2)=1-\Tr(\bP_1\bP_2^{-1})/N$ and  thus $d_{ij,k}=1- \Tr(\bP_{ij}\bP_{jk}\bP_{ki})/N$.
The cdf of $\max\{s_{ik}^*,s_{jk}^*\}$, $P_{\max}(x)$, can be complicated, but one can find a more concise formula for an upper bound for it, which is sufficient for verifying the middle inequality in \eqref{eq:choices_beta0_1}. Indeed, the cdf of $s^*_{ij}$ for $ij\in \tilde E_b$,
gives an upper bound of $P_{\max}(x)$.  For $N\in \mathbb N$, $1 \leq m \leq N$ and $ij\in \tilde 
E_b$ fixed, $s_{ij}^*=d_\mathcal G(\bP_\text{Haar}, \bI_{N\times N})$ for $\bP_\text{Haar} \sim \text{Haar}(S_N)$. Moreover, $s_{ij}^*=m/N$ is equivalent to having exactly  $m$ elements displaced (and $N-m$ fixed) by  $\bP_\text{Haar}$. Therefore, using the notation $[x]$ for the nearest integer to $x$, for  $1 \leq m \leq N$,
\begin{equation*}
P_{\max}
\left(\frac{m}{N}\right) \leq \sum_{l=0}^m\Pr\left(s_{ij}^*=
\frac{l}{N} \left| ij \in \tilde E_b \right. \right)=
\frac{1}{N!}+
\sum_{l=1}^m \frac{1}{N!} {N \choose l}
\left[\frac{l!}{e}\right].
\end{equation*}
Since $z_\mathcal G=\mathbb E (s_{ij}^*)$ for $ij\in \tilde E_b$, the exact formula for computing $z_\mathcal G$ is
\[
z_{\mathcal G}=\sum_{m=1}^N 
\frac{m}{N!}{N\choose m}
\left[\frac{m!}{e}\right].
\]

We claim that $V(x)$ can be chosen as
\begin{align}
\label{eq:V_perm}
    V(x)=\mathbf 1_{\{x\leq N\}}+\mathbf 1_{\{x> N\}}\frac{e^2 }{N^2} e^{-2 x/N}  x^2.
\end{align}
Indeed, if $q_m$ denotes the probability density function (pdf) of $Y$ and $x_m=m/N$, then
\begin{align}
   & \sup_{\tau>x}\text{Var}(f_\tau(Y))\leq e^2\sup_{\tau>x}\sum_{m=1}^{2N}e^{-2\tau x_m}\tau^2 x_m^2q_m\nonumber \\
    \leq &e^2\sup_{\tau>x}\sum_{x_m\leq \frac{1}{x}}e^{-2\tau x_m}\tau^2 x_m^2q_m + e^2\sup_{\tau>x}\sum_{x_m> \frac{1}{x}}e^{-2\tau x_m}\tau^2 x_m^2q_m \nonumber \\
    \leq &\sum_{x_m\leq \frac{1}{x}}q_m + e^2\sum_{x_m> \frac{1}{x}}e^{-2x x_m}x^2 x_m^2q_m. \label{eq:permuH}
\end{align}
where the last inequality follows from  the facts that $e^2e^{-2\tau x_m}\tau^2 x_m^2\leq 1$ for any $x_m$ and $\tau$ and $e^{-2\tau x}\tau^2 x^2$ achieves global maximum at $x=1/\tau$. To conclude \eqref{eq:V_perm} we note that for $x>1/x_1=N$
(so $x_m>1/x$ for all $m \geq 1$), the right term on the RHS of \eqref{eq:permuH} is bounded by
$ e^2e^{-2x_1 x} x_1^2 x^2=e^2 e^{-2 x/N}  x^2/N^2
$.

Since $s^*_{ij}$ for $ij \in E_b$ is supported on $\{m/N\}_{m=1}^N$, the special condition of Theorem \ref{thm:rand2} holds when $n = \Omega(N/(p^2(1-q_*^2)))$. As mentioned above the requirement $n = \Omega(1/(p^2(1-q_*^2)))$ is necessary so that the third term in \eqref{eq:rand2prob} is less than 1. The additional dependence on $N$ is specific for this
application and makes sense.

\subsubsection{Angular Synchronization}
\label{subsubsec:angular}
In this problem, $\mathcal G=SO(2)$, which is commonly associated with the unit circle, $S^1$, in the complex plane with complex multiplication. A common bi-invariant metric is $d_{\mathcal G}(\theta_1,\theta_2)=|(\theta_1-\theta_2) \mod (-\pi,\pi]|/\pi$ and thus $d_{ij,k}=     |(\theta_{ij}+\theta_{jk}+\theta_{ki})\mod (-\pi,\pi]|/\pi$.
The Haar measure is the uniform measure on $S^1$ and thus $s_{ij}^*$ for $ij \in E_b$ is uniformly distributed on $[0,1]$. 

We first compute $P_{\max}(x)$ and $z_\mathcal G$.
We note that if either $ik$ or $jk\in E_b$, but not both in $E_b$, then the cdf of $\max(s_{ik}^*,s_{jk}^*)$ is $x$. Also, if $ik$,
$jk\in E_b$, then the cdf of $\max(s_{ik}^*,s_{jk}^*)$ is $x^2$. Thus, for $k\in B_{ij}$,
$P_{\max}(x)=p_1x+p_2x^2$, where $p_1=2q_*(1-q_*)/(1-q_*^2)$ and $p_2=(1-q_*)^2/(1-q_*^2)$. Furthermore, $z_\mathcal G=1/2$. We also note that a simple upper bound for $P_{\max}(x)$ is $x$.

The pdf of $Y$ is $p(t)=p_1\mathbf 1_{\{t\leq 1\}}+p_2 (t\mathbf 1_{\{t<1\}}+(2-t)\mathbf 1_{\{t\geq 1\}})$. We note that $V(x)$ can be chosen as the following bound on $V^*(x)$
\begin{align*}
    &\sup_{\tau>x}\text{Var}(f_\tau(Y))\leq e^2\sup_{\tau>x}\mathbb E(e^{-2\tau X}\tau^2 t^2)
    <e^2\sup_{\tau>x} \int_0^\infty e^{-2\tau t}\tau^2t^2(p_1+p_2t) \, dt\nonumber\\
    =&p_1\frac{e^2}{4x}+p_2\frac{3e^2}{8x^2}\leq e^2\max\left\{\frac{1}{4x}\,,\frac{3}{8x^2}\right\}.
\end{align*}

At last, we verify that the special condition of Theorem \ref{thm:rand2} holds.
By integrating the above pdf, the cdf of $Y$ is $P(t)=p_1t\mathbf 1_{\{t\leq 1\}}+p_2( t^2/2 \mathbf 1_{\{t<1\}} + (p_1+1-(t-2)^2/2)\mathbf 1_{\{t\geq 1\}})$. We note that $Q'(x)=1/p(Q(x))$ and thus for $x<P(1)$, $Q'(x)=1/(p_1+p_2Q(x))$. Therefore, for $x<P(1)$
\begin{align*}
    \frac{Q'(x)}{Q(x)}=\frac{1}{p_1Q(x)+p_2Q^2(x)}\leq \frac{1}{x},
\end{align*}
where the last inequality follows from the observation $p_1t+p_2t^2>P(t)$ for $t\leq 1$.
\subsubsection{Rotation Synchronization}
In rotation synchronization $\mathcal G = SO(3)$ and a common metric is $d_{\mathcal G}(R_1,R_2)=1/(\sqrt 2 \pi)\cdot \|\log(R_1R_2^T)\|_F$,
which is bi-invariant \cite{metric_rotation}.
Therefore, $d_{ij,k}= 1/(\sqrt2 \pi)\cdot \|\log(R_{ij}R_{jk}R_{ki})\|_F$. We remark that $\|\log(R)\|/\sqrt 2$ is the absolute value of rotation angle, theta, around the eigenspace of $R$ with eigenvalue 1. The Haar measure on $SO(3)$ is described, e.g., in \cite{wang2013exact}. 

The 
distribution of $s_{ij}^*$ is exactly the distribution of $|\theta|$ (described above) for the corresponding group ratio. It is shown in \cite{wang2013exact} that $\theta$ is supported on $[-\pi,\pi]$ with density $(1-\cos\theta)/{(2\pi)}$. Thus, the pdf of $s_{ij}^*$ is  $1-\cos(\pi x)$ for $x\in [0,1]$.
We note that if either $ik\in E_b$ or $jk\in E_b$ (but not both), then the cdf of $\max(s_{ik}^*,s_{jk}^*)$ is $\int_0^x 1-\cos(\pi t) dt = x-\sin(\pi x)/\pi$. Furthermore, if $ik, jk\in E_b$, then the cdf of $\max(s_{ik}^*,s_{jk}^*)$ is $(x-\sin(\pi x)/\pi)^2$. Thus, for $k\in B_{ij}$ and $p_1$ and $p_2$ as specified in Section \ref{subsubsec:angular},
\begin{align*}
    P_{\max}(x)=p_1(x-\sin(\pi x)/\pi)+p_2(x-\sin(\pi x)/\pi)^2\leq x-\sin(\pi x)/\pi.
\end{align*}
Furthermore, 
\begin{align*}
    z_\mathcal G = \int_0^1 (1-\cos(\pi t)) t \, dt=\frac12+\frac{2}{\pi^2}.
\end{align*}

We next specify V(x). Clearly,  
\begin{align*}
    \sup_{\tau>x}\text{Var}(f_\tau(Y))&\leq e^2\sup_{\tau>x}\mathbb E(e^{-2\tau Y}\tau^2 Y^2)= e^2\sup_{\tau>x}\int_0^2 e^{-2\tau t}\tau^2 t^2 p(t) \,dt,
\end{align*}
where $p(t)$ is the pdf of $Y$. It can be easily shown that 
\begin{align*}
    p(t)= p_1(1-\cos(\pi x))\mathbf 1_{\{t\leq 1\}}+p_2(\mathbf 1_{\{t\leq 1\}}p_A(t)+ \mathbf 1_{\{t>1\}}p_B(t)),
\end{align*}
where $p_A(t)=t-3\sin(\pi t)/(2\pi)+\cos(\pi t)t/2$ and $p_B(t)=2-t-5\sin(\pi t)/(2\pi)+ \cos(\pi t)(2-t)/2$. One can verify that $p(t)\leq p_1\pi^2t^2/2+p_2\pi^4t^5/120$.
Thus, $V(x)$ can be chosen as the final RHS of the following equation
\begin{align*}
    &\sup_{\tau>x}\text{Var}(f_\tau(Y))\leq   p_1\frac{e^2\pi^2}{2}\sup_{\tau>x}\int_0^\infty e^{-2\tau t}\tau^2 t^3 \,dt+p_2\frac{e^2\pi^4}{120}\sup_{\tau>x}\int_0^\infty e^{-2\tau t}\tau^2 t^7 \,dt =\nonumber\\
    &p_1\frac{e^2\pi^2}{2}\sup_{\tau>x}\frac{3}{8\tau^2}+p_2\frac{e^2\pi^4}{120}\sup_{\tau>x}\frac{315}{16\tau^6}=p_1\frac{3e^2\pi^2}{16x^2}+p_2\frac{21e^2\pi^4}{128x^6}<\max\left\{\frac{14}{x^2}\, ,\frac{120}{x^6}\right\}.
\end{align*}

 At last, we verify the special condition of Theorem \ref{thm:rand2}.
We first note that by the fact that $p(1)=0.5$, the pdf $p(t)$ satisfies $p_12t^2 + p_2t^5/2\leq p(t)\leq p_1\pi^2t^2/2+p_2\pi^4t^5/120$ for $t\leq 1$. Thus for $t\leq 1$ , the cdf $P(t)$ satisfies $p_12t^3/3+p_2t^6/12\leq P(t)\leq  p_1\pi^2t^3/6+p_2\pi^4t^6/720$. As a result, if $x<P(1)$, then  $Q'(x)=1/p(Q(x))\leq 1/(p_12Q^2(x)+p_2Q^5(x)/2)$. Consequently,
\begin{align*}
    \frac{Q'(x)}{Q(x)}\leq \frac{1}{p_12Q^3(x)+p_2\frac{1}{2}Q^6(x)}\leq \frac1x,
\end{align*}
where the last inequality follows from $P(t)\leq p_1\pi^2t^3/6+p_2\pi^4t^6/720\leq p_12t^3+p_2 t^6/2$ for $t\leq 1$.

\section{Numerical Experiments}\label{sec:experiment}

We demonstrate the numerical performance of CEMP-B and validate the proposed theory.  For comparison, we also test some well-known baseline approaches for group synchronization. We consider the following two representatives of discrete and continuous groups: $\mathbb Z_2$ and $SO(2)$.

Section \ref{sec:implement} summarizes various implementation details of CEMP-B and the baseline algorithms we compare with.
Section \ref{sec:num_conv} numerically verifies our theoretical implications for the choice of $\{\beta_t\}_{t=1}^T$ and our convergence estimates for CEMP-B in the setting of adversarial corruption. 
Sections \ref{sec:num_exact} and \ref{sec:num_noise} test the recovery of CEMP and other baseline algorithms under adversarial corruption without and with noise. Finally, Section \ref{sec:num_uniform} demonstrates phase transition plots of different approaches under uniform corruption. 

\subsection{Details of Implementation and Comparison}
\label{sec:implement}

Our choices of $d_\mathcal G$ for $\mathbb Z_2$ and $SO(2)$ are specified in Sections \ref{sec:Z2_demonstration} and \ref{subsubsec:angular}, respectively.
We represent the elements of $SO(2)$ by the set of angles modulo $2 \pi$, or equivalently, by elements of the unit complex circle $U(1)$. 

All implemented codes are available in the following supplementary Github page: \url{ https://github.com/yunpeng-shi/CEMP}. All experiments were performed on a computer with a 3.8 GHz 8-core i7-10700K CPU and 48 GB memory. 
For CEMP we only implemented CEMP-B, since it is our recommended practical approach. We used the following natural choice of default parameters for CEMP (i.e., CEMP-B) throughout all experiments: $\beta_t = 1.2^{t}$ for $0\leq t\leq 20$. We justify this choice in Section \ref{sec:num_conv}. Other choices of parameters are only tested in Section \ref{sec:num_conv}.
We implemented the slower version of CEMP,
with $\mathcal{C}=\mathcal{C}_3$ (instead of using a subset of $\mathcal{C}_3$ with a fixed number of $3$-cycles per edge), since it is fully justified by our theory.

In Sections \ref{sec:num_exact} and \ref{sec:num_noise}, we compare CEMP+GCW with Spectral~\cite{Z2,singer2011angular} and SDP~\cite{Z2Afonso,singer2011angular} for solving $\mathbb Z_2$ and $SO(2)$ synchronization. 
Our codes for Spectral and SDP follow their description after \eqref{eq:SDR}. 
The specific implementations are rather easy. Indeed, recalling the  notation $\bY$ of \eqref{eq:SDR}, for $kl \in E$: $\bY_{kl}\in \{1,-1\}$ for $\mathcal{G}=\mathbb Z_2$ and  $\bY_{kl}=e^{i \theta_{kl}}\in U(1)$ for $\mathcal{G}=SO(2)$. For SDP, we use a default Matlab-based CVX-SDP solver. 

For $\mathcal G=SO(2)$, we also compare with an IRLS  algorithm that aims to solve \eqref{eq:gs} with $\rho(\cdot)=\|\cdot\|_1$. It first initializes the group elements using Spectral \cite{singer2011angular} and then iteratively solves a relaxation of the following weighted least squares formulation
\begin{align}\label{eq:wls}
    \{\hat g_i(t)\}_{i\in [n]}=\argmin_{\{g_i\}_{i\in [n]}\subset \mathcal G} \sum_{ij\in E}\tilde w_{ij}(t)d^2_{\mathcal G}\left( g_{ij}, g_i g_j^{-1}\right),
\end{align}
where $\tilde w_{ij}(t) = \frac{w_{ij}(t)}{\sum_{j\in N_i}w_{ij}(t)}$, and the new weight is updated by
\begin{align*}
    w_{ij}(t+1) = \frac{1}{d_{\mathcal G}\left(g_{ij}, \hat g_i(t)\hat g_j^{-1}(t)\right)+10^{-4}};
\end{align*}
the additional regularization term $10^{-4}$ aims to avoid a zero denominator. More specifically, the relaxed solution of \eqref{eq:wls} is practically found by the weighted spectral method described after \eqref{eq:weighted}, where $\tilde p_{ij}$ in \eqref{eq:weighted} is replaced by $\tilde w_{ij}(t)$. The weights $w_{ij}(t)$ are equally initialized, and IRLS is run for maximally 100 iterations, where it is terminated  whenever the mean of the distances $d_\mathcal G(\hat g_i(t) \hat g_j^{-1}(t), \hat g_i(t+1) \hat g_j^{-1}(t+1))$ over all $ij\in E$ is less than 0.001. We remark that this approach is the $\ell_1$ minimization  version of \cite{se3_sync} for $SO(2)$. Since IRLS is not designed for discrete optimization, we do not apply it to $\mathbb Z_2$ synchronization.

In the special noiseless setting of Section \ref{sec:test_speed}, we also test CEMP+MST. Our code for CEMP+MST follows the description in Section \ref{sec:practical}, where the MST is found by Prim's algorithm.

Since we can only recover the group elements up to a global right group action, we 
use the following error metric that overcomes this issue: 
\begin{align}\label{eq:errg}
    \text{error}_\mathcal G = \frac{1}{|E|}\sum_{ij\in E}d_{\mathcal G}(\hat g_i\hat g_j^{-1}, g_i^*g_j^{*-1}).
\end{align}
We use \eqref{eq:errg} to measure the performance of CEMP+GCW, Spectral, SDP and IRLS.
We note that we cannot use \eqref{eq:errg} to measure the performance of CEMP, as it does not directly solve group elements. Thus, we evaluate CEMP by 
\begin{align}
\text{error}_S
= \frac{1}{|E|}\sum_{ij\in E}\left|s_{ij}^*-s_{ij}(T)\right|.   
\end{align}

\subsection{Numerical Implications of the Theoretical Estimates for the Adversarial Case}\label{sec:num_conv}
Theorem 
\ref{thm:ir1} suggests that in the noiseless adversarial setting, $\beta_0$ should be sufficiently small and $\beta_t$ should  exponentially increase to infinity with a sufficiently small rate $r$. Theorem 
\ref{thm:ir2} suggests that in the adversarial setting with noise level  $\delta$, $\beta_0$ should be sufficiently small and $\beta_t$ should start increasing almost exponentially with a small rate $r$ and then slow down and converge to a large number proportional to $1/\delta$.
Nevertheless, one cannot test the algorithm with arbitrarily large $t$ due to numerical instabilities; furthermore the noise level $\delta$ is unknown. Therefore, in practice, we use a simpler strategy: we start with a sufficiently small $\beta_0$, and then exponentially increase it with a sufficiently small rate $r>1$, so $\beta_{t}=r\beta_{t-1}$, and stop when $\beta_t$ exceeds a large number $\beta_{\max}$. Our default values  
are $\beta_0=1$, $\beta_{\max}=40$ and $r=1.2$. This choice leads to $T=20$ and we thus expressed it earlier in Section \ref{sec:implement} as $\beta_t = 1.2^{t}$ for $0\leq t\leq 20$. Note that if $\beta_T\approx 40$, then any $s_{ij}(T)=1$ is assigned a negligible weight $\approx \exp(-40)\approx 10^{-17}$. Therefore, enlarging the number of iterations cannot help much and it can worsen the accuracy by accumulating errors.
We remark that in some noisy scenarios lower $T$ may be preferable and we demonstrate below an issue like this when using the ``log max'' error.

We will check whether the above choices for $\{\beta_t\}_{t=1}^T$ work sufficiently well under basic corruption models for CEMP-\er. We also test two choices that contradict our theory: 1) $\beta_0=1$ and $\beta_{\max}=5$ ($\beta_{\max}$ is too small). 2) $\beta_0=30$ and $\beta_{\max}=40$ ($\beta_0$ is too large).

We fix the group $SO(2)$ and generate $G([n], E)$ by the Erd\H{o}s-R\'{e}nyi model $G(n,p)$ with $n=200$ and $p=0.5$. The ground truth elements of $SO(2)$, $\{\theta_i^*\}_{i=1}^n$, are i.i.d.$\sim \text{Haar}(SO(2)) \in (-\pi,\pi]$. 
We assume an additive noise for the uncorrupted group ratios with noise level $\sigma_{in}$. For this purpose we i.i.d.~sample  $\{\varepsilon_{ij}\}_{ij \in E}$ from either a standard Gaussian or a uniform distribution on $[-\sqrt3, \sqrt3]$ (in both cases the variance of  $\varepsilon_{ij}$ is $1$).
We also generate adversarially corrupted elements, $\{\theta_i^{adv}\}_{i=1}^n$, that are i.i.d.$\sim \text{Haar}(SO(2))$. For each $ij \in E$, the observed group ratio is independently corrupted with probability $q$ as follows:
\begin{align*} 
\theta_{ij}=\begin{cases}
 \theta_{i}^*-\theta_j^{*}+\sigma_{in}\varepsilon_{ij} \mod (-\pi, \pi], & \text{ w.p. } 1-q;\\
\theta_{i}^{adv}-\theta_j^{adv}\mod (-\pi, \pi], & \text{ w.p. } q.
\end{cases}
\end{align*}
This setting was adversarially created so that the corrupted group ratios are cycle-consistent. Clearly, the information-theoretic threshold on $q$ is $0.5$. That is, exact recovery is impossible if and only if $q\geq 0.5$. We thus fix $q=0.45$ in our first demonstration so that our setting (especially with noise) is sufficiently challenging.

We consider three noise regimes: $\sigma_{in}=0$, $\sigma_{in}=0.05$ and $\sigma_{in}=0.2$ (the last two cases include both Gaussian and uniform noise). We test the three different choices of $\beta_0$ and $\beta_{\max}$ described above (one implied by our theory and two contradicting it) with fixed $T=20$. 

\begin{figure}[htbp]
\begin{center}
   \includegraphics[width=1\linewidth]{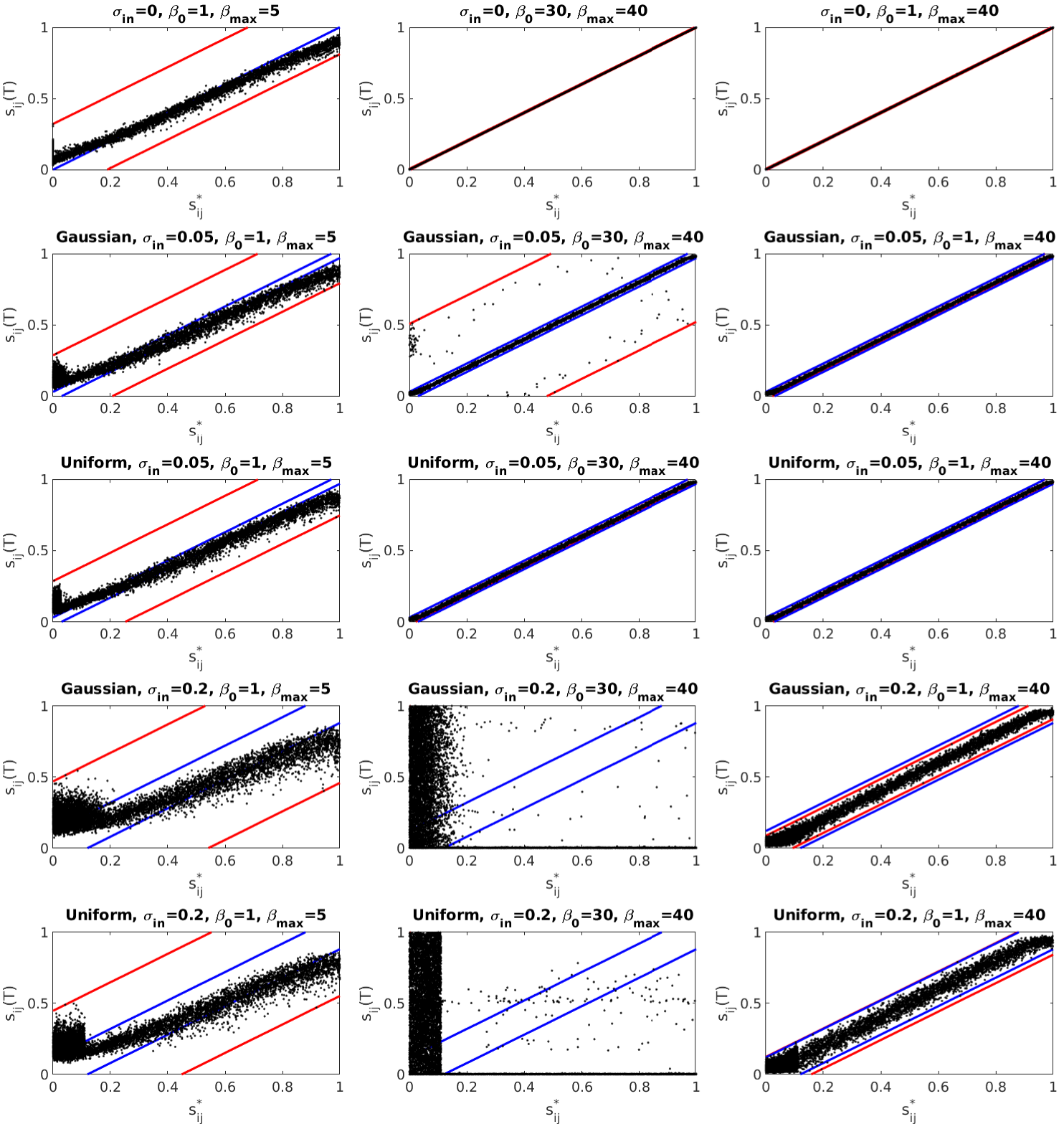}
\end{center}
   \caption{Scatter plot of the estimated corruption levels v.s. the ground truth.}\label{fig:para}
\end{figure}

Figure \ref{fig:para} presents scatter plots for the estimated corruption level, $s_{ij}(T)$, as a function of the ground truth one,  $s_{ij}^*$, for all $ij\in E$. 
The last column corresponds to application of our recommended parameters and the other two columns correspond to other choices of parameters that violate our theory.  The rows correspond to different noise levels and noise distributions. Ideally, in the case of exact recovery, the points in the scatter plot should lie exactly on the line $y=x$. However, we note that in the noisy case ($\sigma_{in}>0$), the exact estimation of $s_{ij}^*$ is impossible. The red lines form a  tight region around the main line containing all points;  their equations are $y=x+\epsilon_+$ and $y=x-\epsilon_-$, where $\epsilon_+=\max_{ij\in E} (s_{ij}(T)-s_{ij}^*)$ and $\epsilon_- = \max_{ij\in E}(s_{ij}^*-s_{ij}(T))$. The blue lines indicate variation by $0.6 \cdot \sigma_{in}$ from the main line (these are the lines $y=x \pm 0.6 \cdot \sigma_{in}$). We chose the constant 0.6 since in the third column, these lines are close to the red ones.

One can see that CEMP-\er with the recommended $\{\beta_t\}_{t=1}^T$ achieves exact recovery in the noiseless case. It approximately estimates the corruption levels in the presence of noise and its maximal error is roughly proportional to the noise level. In contrast, when $\beta_0$ and $\beta_{\max}$ are both small, the algorithm fails to recover the true noise level even when $\sigma_{in}=0$. Indeed, with a small $\beta_t$, bad cycles are assigned weights sufficiently far from $0$ and this results in inaccurate estimates of $s_{ij}^*$. When $\beta_0$ and $\beta_{\max}$ are both large, the algorithm becomes unstable in the presence of noise.
When the noise level is low ($\sigma_{in}=0.05$), the performance is fine when the distribution is uniform; however, when the distribution is Gaussian, there are already some wrong estimates with large errors. When the noise level is $0.2$, the self-consistent bad edges are wrongly recognized as inliers and assigned corruption levels $0$.

Figure \ref{fig:conv} demonstrates the convergence rate of CEMP and verifies the claimed theory. It uses the above  adversarial corruption model with Gaussian noise and the same three noise levels (demonstrated in the three columns), but it tests both $q=0.2$ and $q=0.45$ (demonstrated in the two rows). Each subplot shows three metrics of estimation: ``log max", ``log mean" and ``log median", which correspond to   $\log_{10}(\max_{ij\in E}|s_{ij}(t)-s_{ij}^*|)$, $\log_{10}(\frac{1}{|E|}\sum_{ij\in E}|s_{ij}(t)-s_{ij}^*|)$ and $\log_{10}(\text{median}(\{|s_{ij}(t)-s_{ij}^*|: ij\in E\}))$, respectively.
\begin{figure}[H]
\begin{center}
   \includegraphics[width=1\linewidth]{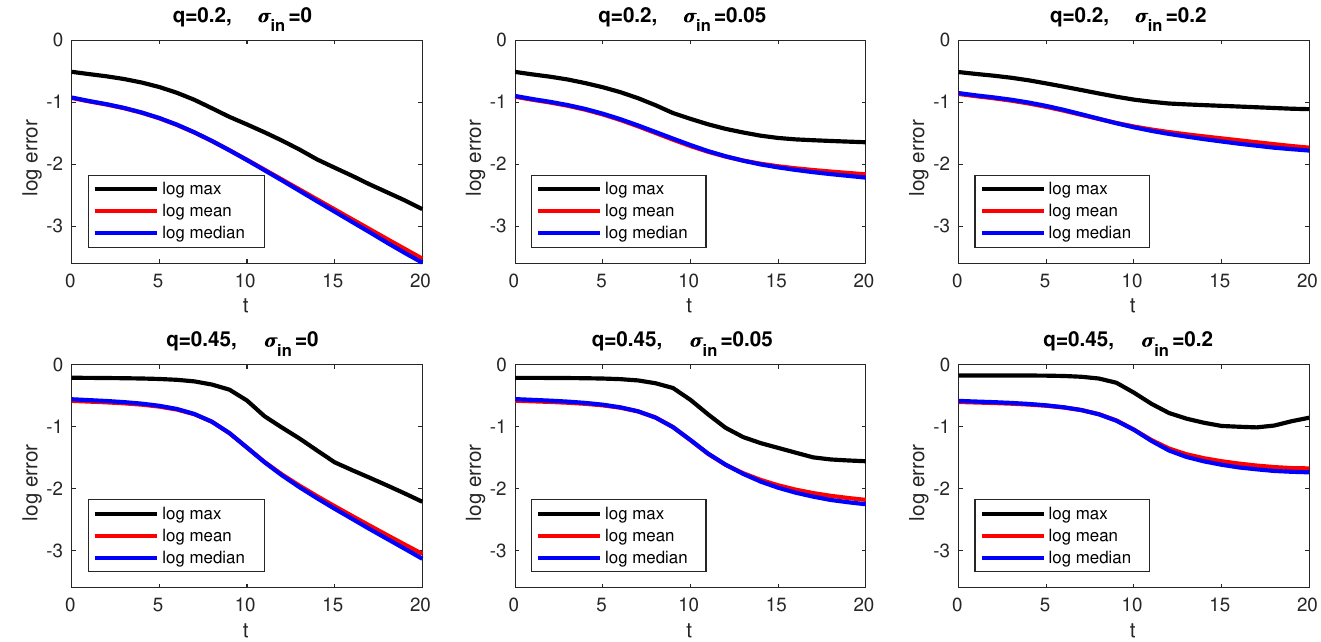}
\end{center}
   \caption{Demonstration of the rate of convergence of CEMP}\label{fig:conv}
\end{figure}

In the noiseless case, the three log errors decrease linearly with respect to $t$; indeed, Theorem \ref{thm:ir1} guarantees linear convergence of CEMP in this case. When $\sigma_{in}>0$, the log errors first demonstrate linear decay (with a smaller rate than above), but the convergence then slows down and seems to approach a constant value in most subfigures; this is consistent with Theorem \ref{thm:ir2}. When $q=0.45$ and $\sigma_{in}=0.2$, the log max error increases at the end. We believe that the source of the problem in this example is that $\beta_{\max}$ is slightly larger than what the theory recommends in this setting of high noise. 
In all plots the mean errors are close to the median errors, and they are about only 1/5 of the maximal errors (a difference of about 0.7 is noticed for the log errors and $10^{0.7} \approx 5$). This indicates that on average CEMP performs much better than its worst case (in terms of edges).

At last, we remark that for the data generated for Figures \ref{fig:para} and \ref{fig:conv}, the maximal ratio of corrupted cycles, $\lambda$, exceeds the bound $1/5$ of Theorem \ref{thm:ir2}. 
Indeed, given the underlying model one may note that $\lambda$ concentrates around $1-(1-q)^2$, which is approximately 0.7 and 0.35 when $q=0.45$ and $q=0.2$, respectively.
Nevertheless, CEMP still achieves exact recovery in these cases. Note though that the upper bound on $\lambda$, even if tight, is only a sufficient condition for good performance. Furthermore, the adversarial corruption model of this section is very special (with strong assumptions on the generation of $E$, $E_b$, the ground truth ratios and the corrupted ratios), whereas the theory was formulated for the worst-case scenario.

\subsection{Exact Recovery under Adversarial Corruption}\label{sec:num_exact}
We consider a more malicious adversarial corruption model than that of Section \ref{sec:num_conv}. Let $G([n],E)$ be generated by an Erd\H{o}s-R\'{e}nyi model $G(n,p)$ with $n=200$ and $p=0.5$. We independently draw $n_c$ graph nodes without replacement. Every time we draw a node, we randomly assign $75\%$ of its neighboring edges
to the set $\tilde E_b$ of selected edges for corruption. It is possible that an edge is assigned twice to $\tilde E_b$ (when both of its nodes are selected), but $\tilde E_b$ is not a multiset. Note that unlike the previous example of Section \ref{sec:num_conv}, the elements of $\tilde E_b$  are not independently chosen. We denote $\tilde E_g:=E \setminus \tilde E_b$.

For $\mathcal G = SO(2)$, 
$\{\theta_i^*\}_{i=1}^n$ and $\{\theta_i^{adv}\}_{i=1}^n$ are i.i.d.~Unif$(-\pi,\pi]$ and
the observed group ratios are generated as follows:
\begin{align*} 
\theta_{ij}=\begin{cases}
 \theta_{i}^*-\theta_j^{*}\mod (-\pi, \pi], & ij\in \tilde E_g;\\
\theta_{i}^{adv}-\theta_j^{adv}\mod (-\pi, \pi], & ij\in \tilde E_b.
\end{cases}
\end{align*}
For $\mathcal G=\mathbb Z_2$, $\{z_i^*\}_{i=1}^n$ and $\{z_i^{adv}\}_{i=1}^n$ are i.i.d.~Unif$\{-1,1\}$ and the observed group ratios are generated as follows:
\begin{align*} 
z_{ij}=\begin{cases}
z_i^*z_j^{*}, & ij\in \tilde E_g;\\
z_{i}^{adv}z_j^{adv}, & ij\in \tilde E_b.
\end{cases}
\end{align*}

Recall that $E_b$ is the set of actually corrupted edges. 
For $\mathcal G = SO(2)$, $E_b=\tilde E_b$; however for $\mathcal G=\mathbb Z_2$, $E_b \neq \tilde E_b$. Indeed, with probability $0.5$, $z_i^*z_j^{*}=z_{i}^{adv}z_j^{adv}$. That is, only $50\%$ of the selected edges in  $ \tilde E_b$ are expected to be corrupted. 

We note that in the case where $\mathcal G = SO(2)$ and $|E_g| \leq |E_b|$, the exact recovery of $\{\theta_i^*\}_{i\in [n]}$ becomes ill-posed, as $E_g$ is no longer the largest cycle-consistent subgraph and $\{\theta_i^{adv}\}_{i\in [n]}$ should have been labeled as the ground truth. We remark that this is not an issue for $\mathcal G=\mathbb Z_2$, since at most $50\%$ of edges in $\tilde E_b$ belong to $E_b$, and thus $|E_b|\leq |E|/2$. Therefore, for $SO(2)$, we need to control $n_c$ so that $|E_b|<|E|/2\approx n^2p/4$. We argue that $n_c/n$ needs to be less than or equal to  $0.53$. Indeed, since we subsequently corrupt $75\%$ of the neighboring edges of the $n_c$ selected nodes, the probability that $ij\in E_b$ is corrupted twice is $0.6$ (we omit the simple calculation). Namely, about $0.6|E_b|$ edges are corrupted twice and $0.4|E_b|$ edges are only corrupted once, and thus about $1.6|E_b|$ corruptions result in $|E_b|$ corrupted edges. Note that the total number of corruptions is $0.75np\cdot n_c$ for $SO(2)$, and thus we require that
\begin{align*}
    \frac{1}{1.6}\cdot \frac{3}{4}np \cdot n_c\leq \frac{n^2p}{4}\implies  \frac{n_c}{n}\leq \frac{1.6}{3}\approx 0.53.
\end{align*}


Figure~\ref{fig:adv} plots $\text{error}_S
$ for CEMP and $\text{error}_\mathcal G$ for the other algorithms as a function of the fraction of corrupted nodes, $n_c/n$. Each plotted value is an average of the estimation error over $10$ trials; it is accompanied with an error bar, which corresponds to $10\%$ and $90\%$ percentiles of the estimation errors in the 10 trials.
The figure only considers $n_c/n\leq 0.9$ for $\mathbb Z_2$ and $n_c/n\leq 0.5$ for $SO(2)$ as all algorithms performed poorly beyond these regimes. 

\begin{figure}[H]
\begin{center}
   \includegraphics[width=1\linewidth]{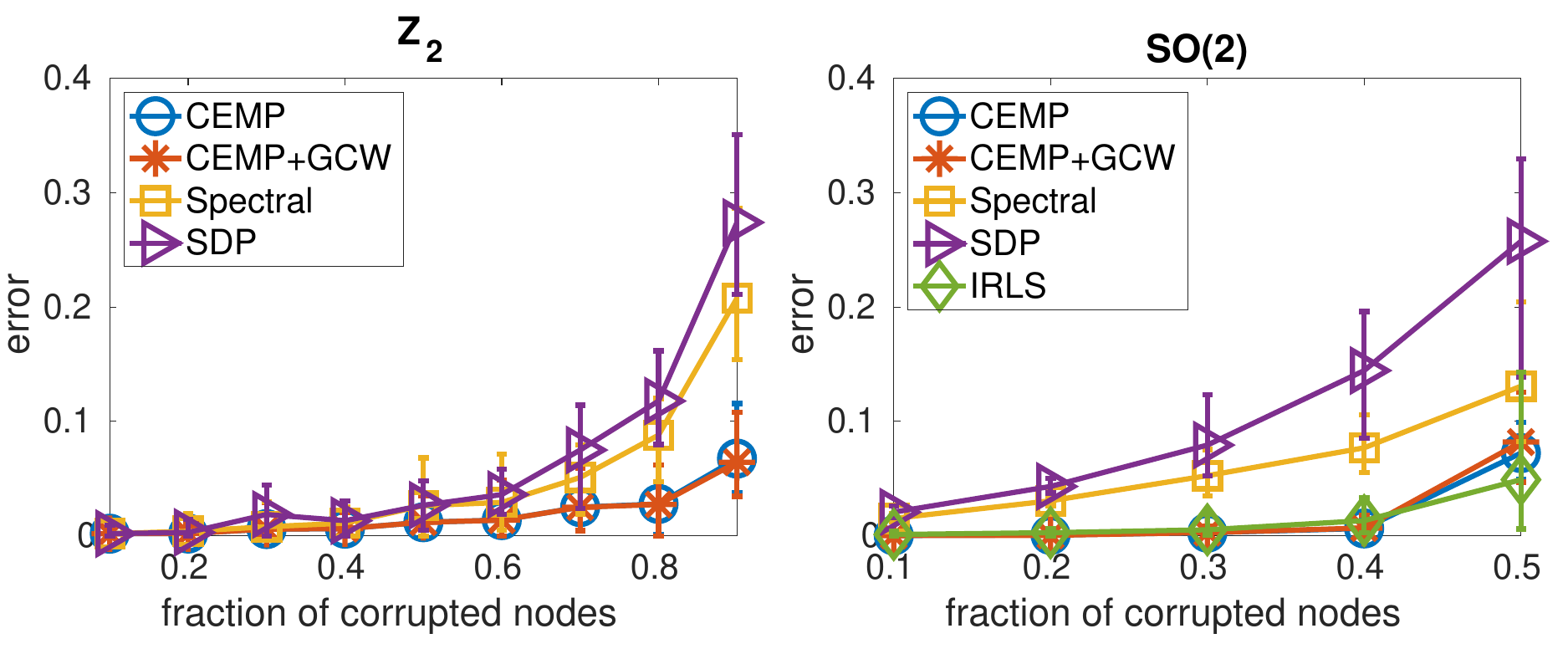}
\end{center}
   \caption{Demonstration of the estimation error under adversarial corruption without noise.}\label{fig:adv}
\end{figure}

We note that CEMP+GCW outperforms Spectral and SDP for both $\mathcal G=\mathbb Z_2$ and $\mathcal G=SO(2)$. One interesting phenomenon is that although CEMP and CEMP+GCW use two different error metrics, their errors seem to nicely align. This is strong evidence that the advantage of CEMP+GCW over Spectral is largely due to CEMP. We observe near exact recovery of CEMP+GCW when $n_c/n\leq 0.4$. In contrast, the errors of Spectral and SDP clearly deviate from 0 when $n_c/n>0.2$, where SDP performs slightly worse. 

In angular synchronization, CEMP+GCW is somewhat comparable to IRLS. It performs better than IRLS  when $n_c/n=0.4$, and worse when $n_c/n=0.5$.  We remark that unlike IRLS that solves a weighted least squares problem in each iteration, CEMP+GCW only solves a single weighted least squares problem.

\subsection{Stability to Noise under Adversarial Corruption}\label{sec:num_noise}

We use the same corruption model as in Section \ref{sec:num_exact}, while adding noise to both good and bad edges. We only consider $\mathcal G = SO(2)$. We do not consider the noisy model of $\mathbb Z_2$ since the addition of noise to a group ratio $z_{ij}\in \mathbb Z_2$ with a follow-up of projection  onto $\mathbb Z_2$ results in either $z_{ij}$ or $-z_{ij}$ and this is equivalent to corrupting $z_{ij}$ with a certain probability.

Let $\sigma_{in}$ and $\sigma_{out}$ be the noise level of inlier and outlier edges, respectively. The observed group ratio $\theta_{ij}$ is generated by
\begin{align*} 
\theta_{ij}=\begin{cases}
\theta_{i}^*-\theta_j^{*}+\sigma_{in}\varepsilon_{ij}\mod (-\pi, \pi], & ij\in E_g;\\
\theta_{i}^{adv}-\theta_j^{adv}+\sigma_{out}\varepsilon_{ij}\mod (-\pi, \pi], & ij\in E_b,
\end{cases}
\end{align*}
where $\theta_i^*$, $\theta_i^{adv}$ are i.i.d.~Unif$(-\pi,\pi]$ and the noise variable $\varepsilon$ is i.i.d.~$N(0,1)$.

\begin{figure}[htbp]
\begin{center}
   \includegraphics[width=1\linewidth]{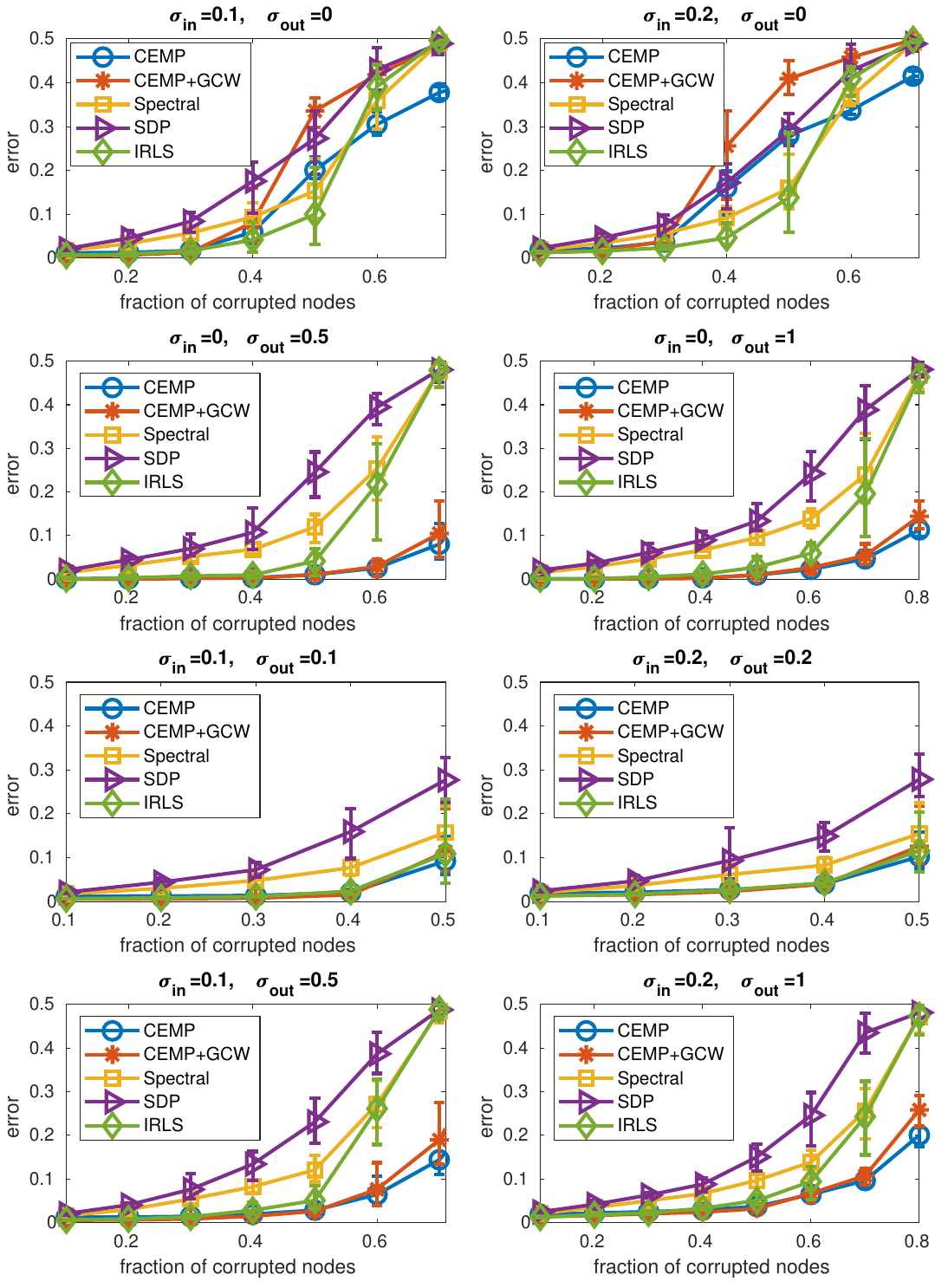}
\end{center}
   \caption{Demonstration of the estimation error under adversarial corruption with Gaussian noise.}\label{fig:adv_noise}
\end{figure}

Figure \ref{fig:adv_noise} plots $\text{error}_S
$ of CEMP and $\text{error}_\mathcal G$ of the other algorithms as a function of the fraction of corrupted nodes. As in Figure~\ref{fig:adv}, each plotted value is an average of the estimation error over $10$ trials  and is accompanied with an error bar corresponding to $10\%$ and $90\%$ percentiles. Four different scenarios are demonstrated in the four different rows of this figure.

Its first row corresponds to a very malicious case where $\sigma_{in}>\sigma_{out}=0$. In this case, the bad edges in $E_b$ are cycle-consistent (since $\sigma_{out}=0$) and the good edges in $E_g$ are only approximately cycle-consistent (due to noise). 
As explained in Section \ref{sec:num_exact}, when $\sigma_{in}=\sigma_{out}=0$ the information-theoretic bound of $n_c/n$ is 0.53. However, when $\sigma_{in}>\sigma_{out}=0$ the  information-theoretic bound is expected to be smaller. Since we do not know this theoretical bound, we first focus on two simpler regions. The first is when $n_c/n\leq 0.3$, so $E_g$ is much larger than $E_b$. In this case, CEMP and IRLS mainly select the edges in $E_g$ for the inlier graph, and CEMP+GCW is comparable to IRLS and outperforms Spectral and SDP. The second region is when $n_c/n = 0.7$ (the same result was noted when $n_c/n \geq 0.7$). Here, both CEMP and IRLS recognize $E_b$ as the edges of the inlier graph and completely ignore $E_g$; consequently they exactly recover $\{\theta_i^{adv}\}_{i=1}^n$, which results in estimation error of $0.5$ (indeed, note that $\mathbb E d_{\mathcal G}(\theta_i^*-\theta_j^*, \theta_i^{adv}-\theta_j^{adv}) =0.5$). However, Spectral and SDP cannot recover either $\{\theta_i^*\}_{i=1}^n$ or $\{\theta_i^{adv}\}_{i=1}^n$ in this regime. 
In the middle region between these two regions, CEMP seems to weigh the cycle-consistency of $E_b$ more than IRLS. On the other hand, IRLS seems to mainly weigh the relative sizes of $E_g$ and $E_b$. In particular, the transition of IRLS from mainly using $E_g$ to mainly using $E_b$ for the inlier graph occurs around the previously mentioned value of $n_c/n=0.53$ for the noiseless case, whereas CEMP transitions earlier. It seems that Spectral has a similar late transition as IRLS and SDP has an earlier transition than IRLS, but it is hard to locate it due to the poor performance of SDP.

The second row of Figure \ref{fig:adv_noise} corresponds to a less adversarial case where $\sigma_{out}>\sigma_{in}=0$. Thus, good edges are exactly cycle-consistent ($\sigma_{in}=0$), and the bad edges in $E_b$ are only approximately cycle-consistent. The information-theoretic bound of $n_c/n$ should be above 0.53 in this case. Indeed, CEMP+GCW is able to almost exactly recover the ground truth when $n_c/n=0.5$. Its estimation error is smaller than 0.1 even when $n_c/n=0.7$, and all other algorithms perform poorly in this regime. 

The third row of Figure \ref{fig:adv_noise} corresponds to the case where inlier and outlier edges have the same level of noise. Similarly to the results of Section \ref{sec:num_exact}, CEMP+GCW is comparable to IRLS and performs better than other methods. Its performance starts to degrade when $n_c/n$ approaches the information-theoretic bound of the noiseless case, $0.53$.

In the last row of Figure \ref{fig:adv_noise}, both good and bad edges are noisy, and bad edges have higher noise levels. This case is somewhat similar to the one in the second row of the figure. CEMP+GCW performs better than all other methods, especially in the high corruption regime where $n_c/n > 0.5$.

\subsection{Phase Transition under Uniform Corruption}\label{sec:num_uniform}
We demonstrate phase transition plots of the different algorithms under the uniform corruption model.  
For $\mathbb Z_2$, the group ratios are generated by
\begin{align*}
z_{ij}=\begin{cases}
z_i^*z_j^{*}, & \text{ w.p. } 1-q;\\
z^u_{ij}, &  \text{ w.p. } q,
\end{cases}
\end{align*}
where $z_i^*, z^u_{ij}$ are i.i.d.~Unif$\{1,-1\}$.
For $\mathcal G = SO(2)$, the observed group ratios $\theta_{ij}$, $ij \in E$, are generated by
\begin{align*}
\theta_{ij}=\begin{cases}
\theta_{i}^*-\theta_j^{*}\mod (-\pi, \pi], & \text{ w.p. } 1-q;\\
\theta^u_{ij}, & \text{ w.p. } q,
\end{cases}
\end{align*}
where $\theta_i^*$, $\theta_{ij}^u$ are i.i.d.~Unif$(-\pi,\pi]$.

Figures \ref{fig:Z2_uniform} and \ref{fig:angle_uniform} show the phase transition plots for $\mathbb Z_2$ and $SO(2)$ synchronization, respectively. They include plots for the averaged $\text{error}_S$ (for CEMP) and averaged $\text{error}_\mathcal G$ (for the other algorithms) over 10 different random runs for various values of $p$, $q$ and $n$ ($p$ appears on the $y$ axis, $q$ on the $x$ axis and $n$ varies with subfigures). The darker the color the smaller the error. The red and blue curves correspond to possible phase transition thresholds, which we explain below.  

\begin{figure}[htbp]
\begin{center}
   \includegraphics[width=1\linewidth]{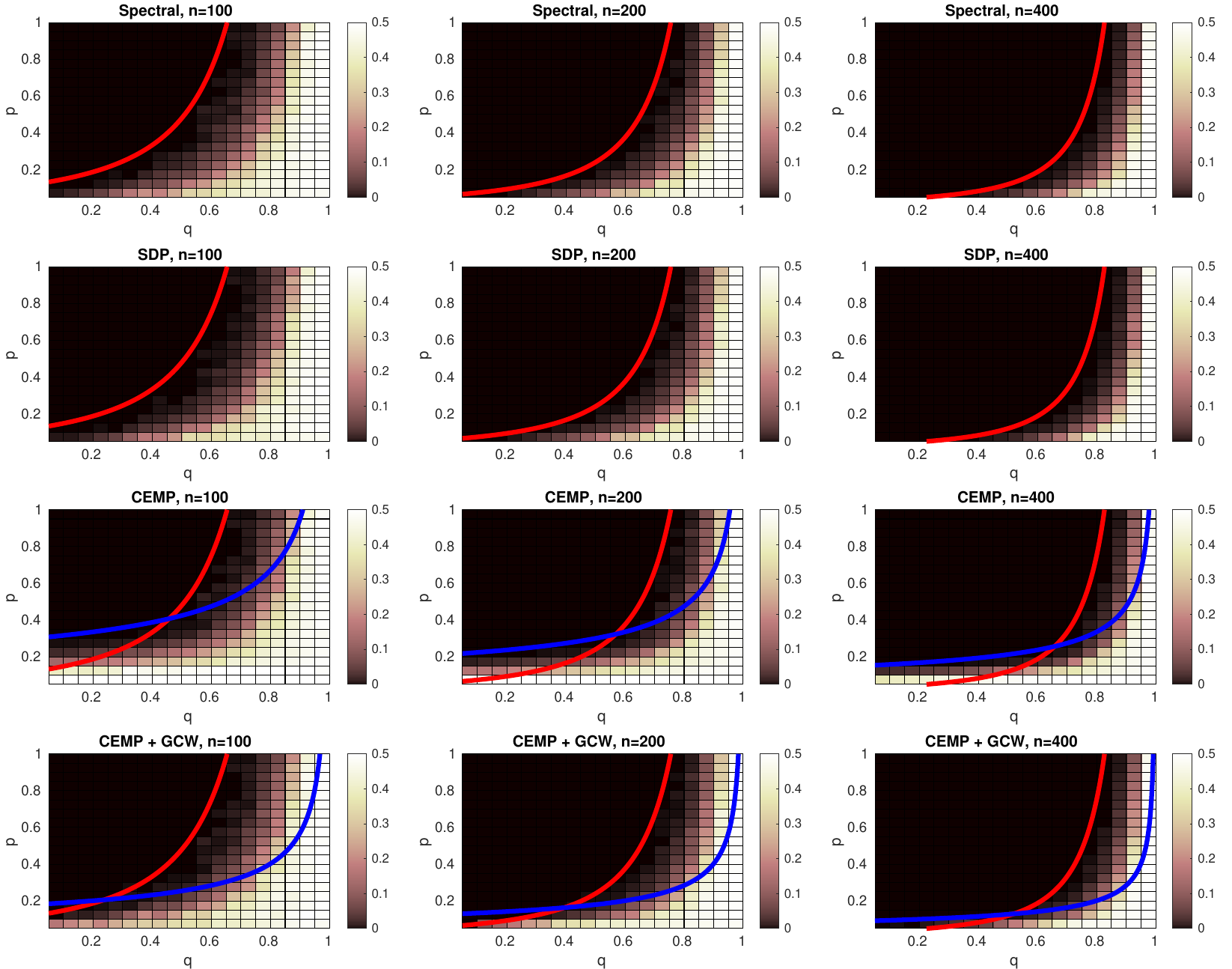}
\end{center}
   \caption{Phase transition plots for $\mathbb Z_2$ synchronization.}\label{fig:Z2_uniform}
\end{figure}

\begin{figure}[htbp]
\begin{center}
   \includegraphics[width=1\linewidth]{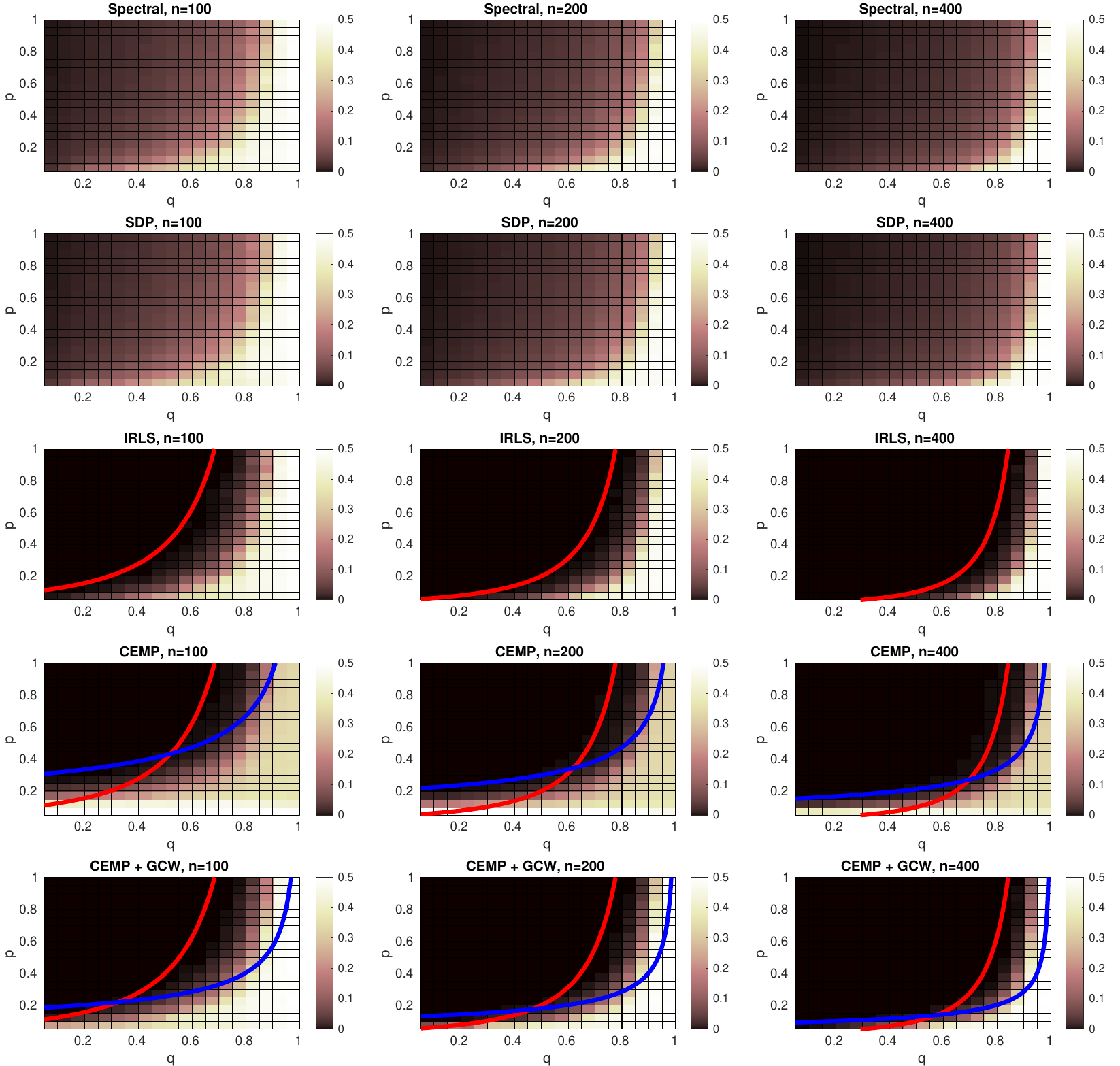}
\end{center}
   \caption{Phase transition plots for $SO(2)$ synchronization.}\label{fig:angle_uniform}
\end{figure}

For $\mathbb Z_2$, recall that \cite{Z2Afonso2} establishes the information-theoretic bound in \eqref{eq:info_bound_z2} and that \cite{Z2Afonso} and \cite{Z2} show that it also holds for SDP and Spectral, respectively.
Using this bound, while ignoring its log factor and trying to fit its unknown constant to the phase transition plots for Spectral and SDP, we set the red line as $p=12/(n(1-q)^2)$. Clearly, this is not the exact theoretical lower bound.
For $SO(2)$, recall that the information theoretic bound is the same as the above one for $\mathbb Z_2$ \cite{info_theoretic_sync}. We cannot fit a curve to Spectral and SDP since they do not exactly recover the group ratios. Instead, we try to fit a phase transition curve to IRLS  (even though there is no theory for this). We fit the red curve defined by $p=10/(n(1-q)^2)$. 

We note that the red curves (which were determined by the above-mentioned algorithms) do not align well with the phase transition plots of CEMP when $p$ is sufficiently small (that is, when the underlying graph is sparse).
Indeed, Section \ref{sec:sample_complexity} explains the limitation of CEMP (and any method based on 3-cycle-consistency) when $p$ is small. 
On the other hand, the sample complexity of CEMP might be tight as a function of $q_g=1-q$, as opposed to our current theoretical estimate (see discussion in Section \ref{sec:sample_complexity}). If this assumption is correct, then the red curve can align well with the phase transition when $p$ is not very small as may seem from the figures. For small $p$, we followed the (necessary) dependence on $p$ of our estimates in Section \ref{sec:sample_complexity} and further experimented with different powers of $(1-q)$ and consequently fit the following blue curves:  $p=3/\sqrt{n(1-q)}$ and $p=1.8/\sqrt{n(1-q)}$ for CEMP and CEMP+GCW, respectively. We used the same curves for both $\mathbb Z_2$ and $SO(2)$ synchronization.

It is evident from Figure \ref{fig:Z2_uniform} that
the phase transition plots of Spectral and SDP align well with the red curve. 
For CEMP and CEMP+GCW, the exact recovery region (dark area) seems to approximately lie in the area enclosed by the red and blue curves. 
The blue curve of CEMP+GCW is slightly closer to the $x$-axis than that of CEMP. This suggests that combining CEMP with GCW can partially help with dealing with sparse graphs.

In Figure \ref{fig:angle_uniform}, 
Spectral and SDP do not seem to exactly recover group elements in the presence of any corruption, and thus a phase transition region is not noticed for them. 
The phase transition plots of IRLS  align well with the red curve.
The exact recovery regions of both CEMP and CEMP+GCW seem to approximately lie in the area enclosed by the red and blue curves. 
Again, CEMP+GCW seems to slightly improve the required bound (its blue curve is lower). 
Nevertheless, more careful research is needed to determine in theory the correct phase transition curves of CEMP and CEMP+GCW.

\subsection{Testing the Speed of the Algorithms}
\label{sec:test_speed}
We compare the speed of the algorithms under different parameters. We assume the uniform corruption model for $SO(d)$ without noise and with $q=0.2$. We test the values $n=100$, 300, 1000, with $p=50/n$ and dimensions $d=2$, 10, 50. To be consistent with the underlying metric of other algorithms and with different choices of dimensions $d$, we use the following scaled version of the Frobenius metric: $d_\mathcal G(g_1,g_2)=\|g_1-g_2\|_F/(2\sqrt{d})\in [0,1]$. We test the same algorithms of Section \ref{sec:num_uniform} and also CEMP+MST (it can be applied here since there is no noise, though in general we don't recommend it).

Tables \ref{tab:so2}-\ref{tab:so50} 
report the runtimes of different algorithms, where each table corresponds to a different value of $d$. In order to account for a possible tradeoff between runtime and accuracy, they also report the normalized root mean squared error (NRMSE):
\begin{align*}
    \text{NRMSE} = \sqrt{\frac{1}{4d|E|}\sum_{ij\in E}\|\hat g_i\hat g_j^{-1}-g_i^*g_j^{*-1}\|_F^2}.
\end{align*}
These tables use ``NAp" (not applicable) when NRMSE is not defined (for CEMP) and ``NA" (not available) when the memory usage exceeds the 48 GB limit. They report ``0" NRMSE error whenever the error is smaller than $10^{-15}$.

\begin{table*}[htbp]
\centering 
\resizebox{1.0\columnwidth}{!}{
\renewcommand{\arraystretch}{1.3}
\tabcolsep=0.1cm
\begin{tabular}{|c||c|c||c|c||c|c|}
\hline
$n$ &  \multicolumn{2}{c||}{100} &
\multicolumn{2}{c||}{300} & \multicolumn{2}{c|}{1000} 
  \\\hline
\text{Algorithms}&  runtime & NRMSE &  runtime & NRMSE & runtime & NRMSE 
\\\hline
CEMP &0.08& NAp & 0.19 & NAp & 0.48 & NAp
\\\hline
CEMP+MST & 0.1 & \textbf{0} & 0.22 &\textbf{0} & 0.69 & \textbf{0}
\\\hline
CEMP+GCW & 0.09 & $2\times 10^{-4}$ &0.2 & $2\times 10^{-4}$ & 0.56 & $3\times 10^{-3}$
\\\hline
Spectral & \textbf{0.004} & $4\times 10^{-2}$ & \textbf{0.007} & $4\times 10^{-2}$ & \textbf{0.05} & $4\times 10^{-2}$
\\\hline
SDP & 0.48 & $4\times 10^{-2}$ & 5.41 & $4\times 10^{-2}$ &NA &NA
\\\hline
IRLS & 0.038 & $2\times 10^{-4}$ & 0.12 & $2\times 10^{-4}$ & 0.45 & $2\times 10^{-4}$
\\\hline
\end{tabular}}
\caption{Runtime (in seconds) and accuracy for $\mathcal G=SO(2)$ and $np=50$.} 
\label{tab:so2}
\end{table*}

\begin{table*}[htbp]
\centering 
\resizebox{1.0\columnwidth}{!}{
\renewcommand{\arraystretch}{1.3}
\tabcolsep=0.1cm
\begin{tabular}{|c||c|c||c|c||c|c|}
\hline
$n$ &  \multicolumn{2}{c||}{100} &
\multicolumn{2}{c||}{300} & \multicolumn{2}{c|}{1000} 
  \\\hline
\text{Algorithms}&  runtime & NRMSE &  runtime & NRMSE & runtime & NRMSE 
\\\hline
CEMP & 0.39 &NAp & 0.49 & NAp & \textbf{0.88} & NAp
\\\hline
CEMP+MST & 0.41 & \textbf{0} &0.53 &\textbf{0} & 1.2 &\textbf{0}
\\\hline
CEMP+GCW & 0.41 & $2\times 10^{-8}$ & 0.75 & $2\times 10^{-8}$ & 5.45 & $2\times 10^{-8}$
\\\hline
Spectral & \textbf{0.011} & $4\times 10^{-2}$ & \textbf{0.13} & $4\times 10^{-2}$ & 1.63 & $4\times 10^{-2}$
\\\hline
SDP & 25.52 & $4\times 10^{-2}$ &NA &NA &NA &NA
\\\hline
IRLS & 0.09 & $2\times 10^{-5}$ & 0.69 & $1\times 10^{-5}$ & 7.08 & $1\times 10^{-5}$
\\\hline
\end{tabular}}
\caption{Runtime (in seconds) and accuracy for $\mathcal G=SO(10)$ and $np=50$.} 
\label{tab:so10}
\end{table*}

\begin{table*}[htbp]
\centering 
\resizebox{1.0\columnwidth}{!}{
\renewcommand{\arraystretch}{1.3}
\tabcolsep=0.1cm
\begin{tabular}{|c||c|c||c|c||c|c|}
\hline
$n$ &  \multicolumn{2}{c||}{100} &
\multicolumn{2}{c||}{300} & \multicolumn{2}{c|}{1000} 
  \\\hline
\text{Algorithms}&  runtime & NRMSE &  runtime & NRMSE & runtime & NRMSE 
\\\hline
CEMP &21.06 & NAp &24.74 & NAp & \textbf{23.18} &NAp
\\\hline
CEMP+MST &21.08 &\textbf{0}& 24.78&\textbf{0} & 23.43 & \textbf{0} 
\\\hline
CEMP+GCW & 25.68 & $1\times 10^{-8}$ & 66.26 & $7\times 10^{-9}$& 584.62 & $2\times 10^{-8}$
\\\hline
Spectral &\textbf{1.32} &$4\times 10^{-2}$ &\textbf{12.06} &$4\times 10^{-2}$ &  147.09 & $4\times 10^{-2}$
\\\hline
SDP & NA & NA & NA &NA &NA &NA
\\\hline
IRLS & 7.17 & $2\times 10^{-5}$ & 59.05 & $2\times 10^{-5}$ &NA &NA
\\\hline
\end{tabular}}
\caption{Runtime (in seconds) and accuracy for $\mathcal G=SO(50)$ and $np=50$.} 
\label{tab:so50}
\end{table*}

We note that SDP and Spectral have the lowest accuracy since they minimize a least squares objective function, which is not robust to outliers. SDP is the slowest algorithm in all experiments, and we could not implement it on our computer (whose specifications are detailed in Section \ref{sec:implement})
when $n\geq 1000$ or $d\geq 50$. Spectral seems to be the fastest method when either $d$ or $n$ are sufficiently small. However, for $n=1000$, CEMP and CEMP+MST are faster than Spectral when $d=10$ and $d=50$; in the latter case they are more than 6 times faster. The recovery error of IRLS is small in all experiments. However, it is the second slowest algorithm. It also doubles the memory usage of Spectral since it needs to store both the original $\bY$ in \eqref{eq:SDR} and the weighted $\bY$ in each iteration. Due to this issue, IRLS exceeds the memory limit when $n=1000$ and $d=50$. We note that in most of the experiments for $d=10$, 50, CEMP+GCW achieves NRMSE 3-order of magnitude lower than that of IRLS. We also note that CEMP+MST is completely accurate in all experiments and it is the fastest method when $d\geq 10$ and $n=1000$. Since we fix $np=50$ and since the complexity of CEMP is of order $O((npd)^3)$, the runtime of CEMP does not change much in each table; whereas the runtimes of the other algorithms clearly increases with $n$. We observe that the runtime of the MST post-processing is almost negligible in comparison to CEMP. Indeed, the time complexity of building the MST is $pn^2\log n$ and that of computing $g_i=g_{ij}g_j$ along the spanning tree is $O(nd^3)$.

\section{Conclusion}
\label{sec:conclusion}
We proposed a novel message passing framework for robustly solving group synchronization problems with any compact group under adversarial corruption and sufficiently small noise. We established a deterministic exact recovery theory for finite sample size with weak assumptions on the adversarial corruption (the ratio of corrupted cycles per edge needs to be bounded by a reasonable constant). Previous works on group synchronization assumed very special generative models. Some of them  only considered asymptotic recovery and they were often restricted to special groups. Somewhat similar guarantees exist for the different problem of camera location estimation, but we already mentioned their weaknesses in view of our guarantees. 
We also established the stability of CEMP to bounded and sub-Gaussian noise.
We further guaranteed exact recovery under a previous uniform corruption model, while considering the full range of model parameters. 

There are different theoretical directions that may help in improving this work. 
First of all, the theory for adversarial corruption assumes a uniform bound on the corruption ratio per edge, whereas in practice one should allow a small fraction of edges to be contained in many corrupted cycles. We believe that it is possible to address the latter setting with CEMP by adaptively choosing $\beta_t$ for different edges. This way, instead of the current $\ell_\infty$ bound on the convergence, one can establish an $\ell_1$ or $\ell_2$ convergence bound. Nevertheless, the mathematical ideas behind guaranteeing an adaptive reweighting strategy are highly complicated and hard to verify. Instead, we prefer to clearly explain our theory with a simpler procedure. 

Another future direction is to extend the theory to other classes of reweighting functions, in addition to the indicator and exponential functions.  In particular, one may further consider finding an optimal sequence of reweighting functions under certain statistical models. This direction will be useful once an adaptive reweighting strategy is developed. On the other hand, when $\beta_t$ is the same for all edges, then Section \ref{sec:practical} advocates for the exponential reweighting function.

We emphasized the validity of the exact recovery guarantees under UCM for any $q <1$ and for various groups.
However, as we clarified in Section \ref{sec:sample_complexity}, the sample complexity bound implied by our estimates does not match the information-theoretic one. It will be interesting to see if a more careful analysis of a CEMP-type method can fill the gap. We believe that this might be doable for the dependence of the sample complexity on $q_g$, but not on $p$ (see our discussion in Section \ref{sec:sample_complexity}
and our numerical results in Section \ref{sec:num_uniform}). We expect that the use of higher-order cycles will improve the dependence on $p$ but worsen the dependence on $q_g$. 

The framework of CEMP can be relevant to other settings that exploit cycle consistency information, but with some limitations. First of all, in the case of non-compact groups, one can scale the given group elements. In particular, if both $\{g_i\}_{i=1}^n$ and $\{g^*_i\}_{i=1}^n$ lie in a ball of fixed radius, then by appropriate scaling, one can assume that $s_{ij}^* \leq 1$ for all $ij \in E$. The theory thus extends to non-compact groups with finite corruption models and bounded noise. If the distribution of the corruption or noise has infinite support, then our theory is invalid when the sample size approaches infinity, though it is still valid for a finite sample size.

We also claim that CEMP can be extended to the problem of camera location estimation. Since the scale information of group ratios is missing, one should define alternative notions of cycle consistency, inconsistency measure and corruption level, such as the ones we proposed in \cite{AAB}. In fact, using such notions, the AAB algorithm of the conference paper \cite{AAB} is CEMP-B with $s_{ik}(t)+s_{jk}(t)$ replaced by $\max\{s_{ik}(t),s_{jk}(t)\}$. We remark that there is no significant difference between these two comparable choices.
We can develop a similar, though weaker, theory for exact recovery by CEMP (or AAB) for camera location estimation. In order to keep the current work focused, we exclude this extension. The main obstacle in establishing this theory is that the metric is no longer bi-invariant and thus $d_{ij,k}$ may not equal to $s_{ij}^*$, even for uncorrupted cycles. 

A different notion of cycle consistency is also used in a histogram-based method for the identification of common lines in cryo-EM imaging \cite{cycle_common_lines}. We believe that the reweighting procedure in CEMP can be incorporated in \cite{cycle_common_lines} to reduce the rate of false positives. 

We claim that cycle consistency is also essential within each cluster of vector diffusion maps (VDM) \cite{VDM_singer}, 
which aims to solve a different problem of clustering graph nodes, for example, clustering cryo-EM images with different viewing directions \cite{VDM_jane,VDM_singer}. 
Indeed, in VDM, powers of the connection adjacency matrix give rise to ``higher-order connection affinities'' between nodes $i$ and $j$ obtained by the squared norm of a weighted sum of the products of group ratios $g_{L_{ij}}$ along paths $L_{ij}$ from $i$ to $j$ (see e.g., demonstration in Figure 4(a) in \cite{VDM_cryo}). For $i$ and $j$ in the same cluster, cycle consistency implies that each product of group ratios $g_{L_{ij}}$ is approximately $g_{ij}$ (or exactly $g_{ij}$ if there is no corruption). Consequently, for each $ij \in E$, the sum of $g_{L_{ij}}$ over ${L_{ij}}$ with fixed length (depending of the power used) is  approximately a large number times $g_{ij}$ and thus has a large norm, that is, the higher-order connection affinity is large.
On the other hand, if $i$ and $j$ belong to different clusters, then the different $g_{L_{ij}}$'s may possibly cancel or decrease the effect of each other (due to the different properties of the clusters). Consequently, the higher-order connection affinity is typically small. We note that these affinities are somewhat similar to our weighted average of cycle inconsistencies, $\sum_{L} w_L d_\mathcal G(g_L, e_\mathcal G)$. However, unlike our reweighting strategy,
VDM weighs cycles in a single step using Gaussian kernels (see (3) and (4) in \cite{VDM_jane}). We believe that a suitable reweighting strategy can be applied to VDM to improve its classification accuracy. 
After the submission of this work, \cite{shi_NEURIPS2020}
showed that for the permutation group with a very special metric, CEMP is equivalent to an iterative application of the graph connection weight (GCW) matrix, which is used in VDM in a different way (see also Section \ref{sec:computational}). 
Unfortunately, we find it unlikely to extend the ideas of \cite{shi_NEURIPS2020} to other groups and metrics.

A relevant question is the  possibility of extending all results to higher-order cycles and the usefulness of such an extension. We believe that such an extension is not difficult. As mentioned above, we expect higher-order cycles to help with sparse graphs, but possibly degrade the ability to handle very high corruption and also significantly enlarge the computational complexity. We did not find it necessary to explore this, since datasets of structure from motion seem to have enough 3-cycles per edge to guarantee that CEMP can reliably estimate the corruption levels of most edges. In this case, the combination of CEMP with other methods can improve results for edges that may not have enough 3-cycles (see e.g., \cite{AAB,MPLS}).

Finally, we mention that while the theory is very general and seems to apply well to the common compact groups $\mathbb Z_2$, $S_N$ and $SO(d)$, 
specific considerations need to be addressed for special groups and special instances. For example,  the problem of recovering orientations of jigsaw puzzles \cite{vahan_jigsaw} can be solved by $\mathbb{Z}_4$ synchronization, where its ideal graph is a two-dimensional lattice. In this setting, each edge is contained in at most two cycles of length at most 4. Thus effective inference of corruption requires cycles with length greater than $4$. 

\section*{Acknowledgments}

We thank the anonymous reviewers for their helpful suggestions. We also thank Tyler Maunu who mentioned to us that our earlier algorithm in \cite{AAB} seemed to him like a message-passing procedure. Our effort to quantitatively understand his view led to this work.

\appendix
\section{Supplementary Proofs}

\subsection{Proof of Proposition \ref{prop:equiv}}\label{sec:prop_equiv}

 \ref{item:Eg}$\to$ \ref{item:gi}: One can find the minimum spanning tree of the connected graph $G([n], E_g)$, arbitrarily assign an element $g_1$ at the root of the tree, and inductively assign the value $g_j$ to a node whose parent was assigned the value $g_i$ as follows: $g_j=g^*_{ji}g_i$ ($ij \in E_g$ and thus $g^*_{ij}$ is known). This results in correct recovery of $\{g_i^*\}_{i=1}^n$ up to right multiplication by the arbitrarily assigned element $g_1$.

\noindent
\ref{item:gi}$\to$ \ref{item:sij}: Given $\{g_{ij}\}_{ij\in E}$ and the known $\{g_i^*\}_{i\in [n]}$ one can immediately compute $g_{ij}^*=g_i^* {g_j^*}^{-1}$ and $s_{ij}^*=d_\mathcal G(g_{ij},g_{ij}^*)$ for $ij\in E$.

\noindent
\ref{item:sij}$\to$ \ref{item:Eg}: $E_g$ is exactly recovered by finding edges such that $s_{ij}^*=0$.

\subsection{Proof of Lemma~\ref{lemma:bi}}\label{sec:appbi}
Without loss of generality fix $L=\{12,23,34,\dots,n1\}$. By the bi-invariance of $d_\mathcal G$ and the triangle inequality
\begin{align}
\label{eq:1st_for_lemma_3.6}
|d_{L}-s_{12}^*|&=     \left|d_{\mathcal G} \left(g_{12}g_{23}\cdots g_{n1}\,, e_\mathcal G\right)-d_{\mathcal G} \left(g_{12}  g_{12}^{*-1}\,, e_\mathcal G\right)\right|\nonumber\\
&\leq d_{\mathcal G} \left(g_{12}g_{23}\cdots g_{n1}\,,  g_{12}  g_{12}^{*-1}\right)
=d_\mathcal G\left(g_{12}^*g_{23}\cdots g_{n1}\,, e_\mathcal G\right).
\end{align}
Note that for all $1 \leq i \leq n-1$, the bi-invariance of $d_\mathcal G$ implies that
\begin{equation}
\label{eq:si_i+1}
s_{i\,i+1}^* = d_\mathcal G(g_{i\,i+1}, g_{i\,i+1}^*) =d_\mathcal G\left(g_{12}^*\cdots g_{i-1\,i}^*\,g_{i\,i+1}\cdots g_{n1}\,\,,\, g_{12}^*\cdots g_{i\,i+1}^*g_{i+1\, i+2}\cdots g_{n1}\right).
\end{equation}
Application of \eqref{eq:cyclecons} and then several applications of the triangle inequality and \eqref{eq:si_i+1} yield
\begin{align}
\label{eq:2nd_for_lemma_3.6}
&d_\mathcal G\left(g_{12}^*g_{23}\cdots g_{n1}\,, e_\mathcal G\right)
=d_\mathcal G\left(g_{12}^*g_{23}\cdots g_{n1}\,, g_{12}^*g_{23}^*\cdots g_{n1}^*\right)\nonumber\\
\leq& d_\mathcal G\left(g_{12}^*g_{23}\cdots g_{n1}\,, g_{12}^*g_{23}^*g_{34}\cdots g_{n1}\right)+d_\mathcal G\left(g_{12}^*g_{23}^*g_{34}\cdots g_{n1}\,, g_{12}^*g_{23}^*\cdots g_{n1}^*\right)\nonumber\\
\leq & \sum_{i=2}^n d_\mathcal G\left(g_{12}^*\cdots g_{i-1\,i}^*\,g_{i\,i+1}\cdots g_{n1}\,\,,\, g_{12}^*\cdots g_{i\,i+1}^*g_{i+1\, i+2}\cdots g_{n1}\right)\nonumber\\
=& \sum_{i=2}^n s_{i\,i+1}^*=\sum_{ij\in L\setminus \{12\}}s_{ij}^*.
\end{align}
We conclude the proof by combining \eqref{eq:1st_for_lemma_3.6} and \eqref{eq:2nd_for_lemma_3.6}.

\subsection{Proof of Proposition \ref{prop:cycle}}\label{sec:proofs_bijection}

For any $ij\in E$, let $L_{ij}$ denote a path between nodes $i$ and $j$. 
We claim that $h$ is invertible and its inverse is $h^{-1}((g_{ij})_{ij\in E})=[(g_{L_{ik}}g_k)_{i\in [n]}]$, where $k\in [n]$ (due to the equivalence relationship, this definition is independent of the choice of $k$). Using the definitions of $h^{-1}$ and then $h$, basic group properties and 
the cycle consistency constraints for cycles in $\mathcal C$,
\begin{multline*}
    hh^{-1}((g_{ij})_{ij\in E}) =h([(g_{L_{ik}}g_k)_{i\in [n]}])=
    \\(g_{L_{ik}}g_k(g_{L_{jk}}g_k)^{-1})_{ij\in E}=(g_{L_{ik}}g_{L_{jk}}^{-1})_{ij\in E}=(g_{ij})_{ij\in E}.
\end{multline*}
For $ij \in E$, define $\hat g_{ij}=g_ig_j^{-1}$ and note that $(\hat g_{ij})_{ij\in E}$ is cycle-consistent for any cycle in $G([n], E)$. Thus, $\hat g_{L_{ik}} =\hat g_{ik} = g_i g_k^{-1}$. Using this observation and the definitions of $h$ and $h^{-1}$
$$
h^{-1}h([(g_i)_{i\in [n]}])=h^{-1}((\hat g_{ij})_{ij\in E})=[(\hat g_{L_{ik}}g_k)_{i\in [n]}]=[(g_i)_{i\in [n]}].
$$
The combination of the above two equations concludes the proof.

\subsection{Proof of Proposition~\ref{prop:shift}}\label{sec:prop_shift}

For each $ij\in E$ and $L\in N_{ij}$, the cycle weight computed from the shifted corruption levels $\{s_{ij}(t)+s\}_{ij\in E}$ is
\begin{align*}
w_{ij,L}(t) 
&= \frac{\exp(-\beta_t\sum_{ab\in N_L\setminus \{ij\}}(s_{ab}(t)+s))}{\sum_{L'\in N_{ij}} \exp(-\beta_t\sum_{ab\in N_{L'}\setminus \{ij\}}(s_{ab}(t)+s))}\\
&= \frac{\exp(-\beta_t\sum_{ab\in N_L\setminus \{ij\}}s_{ab}(t))\exp(-\beta_t ls)}{\sum_{L'\in N_{ij}} \exp(-\beta_t\sum_{ab\in N_{L'}\setminus \{ij\}}s_{ab}(t))\exp(-\beta_t ls)}\\
&= \frac{\exp(-\beta_t\sum_{ab\in N_L\setminus \{ij\}}s_{ab}(t))}{\sum_{L'\in N_{ij}} \exp(-\beta_t\sum_{ab\in N_{L'}\setminus \{ij\}}s_{ab}(t))},
\end{align*} 
which equals the original cycle weight.

\subsection{Proof of Theorem~\ref{thm:it2}}\label{sec:it2}
We prove the theorem by induction. For $t=0$, we note that a similar argument to \eqref{eq:proof_for_t_0} implies that 
\begin{align}
\label{eq:arg_for_t_0_app}
    \epsilon_{ij}(0) &\leq \frac{\sum\limits_{k\in N_{ij}}|d_{ij,k}-s^*_{ij}|}{|N_{ij}|}= \frac{\sum\limits_{k\in G_{ij}}|d_{ij,k}-s^*_{ij}|+\sum\limits_{k\in B_{ij}}|d_{ij,k}-s^*_{ij}|}{|N_{ij}|}\nonumber \\
    &\leq \frac{\sum\limits_{k\in G_{ij}}|s^*_{ik}+s_{jk}^*|+|B_{ij}|}{|N_{ij}|}\leq \frac{|G_{ij}|}{|N_{ij}|}\cdot 2\delta+\frac{|B_{ij}|}{|N_{ij}|}\leq \lambda + 2\delta \leq \frac{1}{\beta_0}-\delta.
\end{align}

Next, we assume that $\epsilon (t)+\delta<\frac{1}{\beta_t}$ for an arbitrary $t>0$, and show that $\epsilon (t+1)+\delta<\frac{1}{\beta_{t+1}}$.
We use similar notation and arguments as in the proof of Theorem \ref{thm:it1}.
We note that $\frac{1}{\beta_t}\geq \epsilon (t)+\delta \geq \max_{ij\in E_g}\epsilon_{ij} (t)+\delta=\max_{ij\in E_g}s_{ij}(t)$ and thus for any $ij\in E$, $G_{ij} \subseteq A_{ij}(t)$.
We also note that  $s_{ij}(t)\leq\frac{1}{\beta_t}$ implies that $s_{ij}^*\leq s_{ij}(t)+\epsilon (t)\leq s_{ij}(t)+\frac{1}{\beta_t}-\delta\leq 2\frac{1}{\beta_t}-\delta$.
We use these observations in an argument analogous to \eqref{eq:more_ratio_eps}, which we describe in short as follows:
\begin{align}
\nonumber
\epsilon_{ij}(t+1)
&\leq\frac{\sum\limits_{k\in G_{ij}}\mathbf{1}_{\{s_{ik}(t), s_{jk}(t)\leq \frac{1}{\beta_t}\}}2\delta+\sum\limits_{k\in B_{ij}}\mathbf{1}_{\{s_{ik}(t), s_{jk}(t)\leq \frac{1}{\beta_t}\}}(4\frac{1}{\beta_t}-2\delta)}{|A_{ij}(t)|}\\
&\leq \frac{|G_{ij}|}{|N_{ij}|}2\delta+\frac{|B_{ij}|}{|N_{ij}|}(4\frac{1}{\beta_t}-2\delta).
\label{eq:another_arg_1}
\end{align}
Maximizing over $ij\in E$ the LHS and RHS of \eqref{eq:another_arg_1}
and using the assumptions $\lambda<1/4$ and $4\lambda\frac{1}{\beta_t}+(3-4\lambda)\delta<\frac{1}{\beta_{t+1}}<\frac{1}{\beta_t}$, we conclude \eqref{eq:itall} as follows
\begin{align*}
\epsilon (t+1)+\delta\leq 2(1-\lambda)\delta+2\lambda(2\frac{1}{\beta_t}-\delta)+\delta=4\lambda\frac{1}{\beta_t}+(3-4\lambda)\delta<\frac{1}{\beta_{t+1}}.
\end{align*}

At last, since $\frac{1}{\beta_{t+1}}>4\lambda\frac{1}{\beta_t}+(3-4\lambda)\delta$ and  $\beta_0<\frac{1-4\lambda}{(3-4\lambda)\delta}$, $\beta_t<\frac{1-4\lambda}{(3-4\lambda)\delta}$ for all $t\geq 0$. The latter inequality and the fact that $\{ \beta_t \}_{t \geq 0}$ is increasing imply that $\varepsilon$ is well defined (that is, 
$\lim_{t\to\infty}\beta_t$ exists) and $0<\varepsilon\leq 1$. Taking the limit of \eqref{eq:itall} when $t \to \infty$ yields \eqref{eq:itinf}.

\subsection{Proof of Theorem~\ref{thm:ir2}}\label{sec:ir2}
We prove the theorem by induction. For $t=0$, \eqref{eq:arg_for_t_0_app} implies that $\epsilon(0) \leq \lambda + 2\delta$ and thus
$\frac{1}{4\beta_0}>\lambda +\frac52 \delta\geq \epsilon(0)+\frac12\delta$. Next, we assume that $\epsilon (t)+\frac12\delta<\frac{1}{4\beta_t}$ and show that $\epsilon (t+1)+\frac12\delta<\frac{1}{4\beta_{t+1}}$.

By a similar proof to \eqref{eq:bound_2_eps} and \eqref{eq:bound_2_eps2}, while using the current model assumption
$\max_{ij\in E_g}s_{ij}^*$ $<\delta$, we obtain that
\begin{align*}
\epsilon (t+1)_{ij}
&\leq 2\delta+\frac{\sum\limits_{k\in B_{ij}}e^{-\beta_t\left(s_{ik}^*+s_{jk}^*\right)}\left(s_{ik}^*+ s_{jk}^*\right)e^{\beta_t\left(\epsilon_{ik}(t)+\epsilon_{jk} (t)\right)}}{\sum\limits_{k\in G_{ij}}e^{-\beta_t\left(2\delta+\epsilon_{ik} (t)+\epsilon_{jk} (t)\right)}}.
\end{align*}

The same arguments of proving \eqref{eq:bound_2_eps3} and \eqref{eq:bound_2_eps4} yield the estimate
\begin{align*}
\epsilon (t+1)
\leq 2\delta+\frac{\lambda}{1-\lambda}\frac{1}{e\beta_t}e^{\beta_t(2\delta+4\epsilon (t))}.
\end{align*}
We conclude \eqref{eq:irall} by applying the assumptions $\epsilon (t)+\frac12\delta\leq \frac{1}{4\beta_t}$ and $\frac52\delta+\frac{\lambda}{1-\lambda}\frac{1}{\beta_t}<\frac{1}{4\beta_{t+1}}<\frac{1}{4\beta_t}$ to the above equation as follows
\begin{align*}
\epsilon (t+1)+\frac12\delta
&\leq \frac52\delta+\frac{\lambda}{1-\lambda}\frac{1}{e\beta_t}e^{\beta_t(2\delta+4\epsilon (t))}\leq \frac52\delta+\frac{\lambda}{1-\lambda}\frac{1}{\beta_t}<\frac{1}{4\beta_{t+1}}.
\end{align*}

Establishing $0<\varepsilon\leq 1$ and \eqref{eq:irinf} is the same as in the proof of 
Theorem~\ref{thm:it2}
in Section \ref{sec:it2}.

\subsection{Proof of Theorem \ref{thm:subg}}\label{sec:proofs_subg}

For the fixed $x>0$, we define 
$$G^x_{ij}: = \left\{k\in G_{ij}:\max\{s_{ik}^*,s_{jk}^*\}<\sigma\mu+\sigma x\right\}  \text{ and } \lambda_x=\max_{ij\in E}\left(1-|G_{ij}^x|/|N_{ij}|\right).$$ 
Since $s_{ij}^*\sim sub(\mu,\sigma^2)$, for any $k\in G_{ij}$ and $ij\in E$
\begin{align*}
\Pr(k\notin G_{ij}^x)<\exp(-x^2/2).
\end{align*}
We note that the random variable $X_k = \mathbf 1_{\{k \notin G_{ij}^x\}}$ is a Bernoulli random variable with mean  $p_k < \exp(-x^2/2)$. Using the above notation, $\bar p=1/|G_{ij}|\cdot \sum_{k\in G_{ij}}p_k<\exp(-x^2/2)$. We define $c=\exp(-x^2/2)/\bar p> 1$.
The application of the one-sided Chernoff bound in \eqref{eq:chernoff1} to the independent Bernoulli random variables $\{X_k\}_{k \in G_{ij}}$
with $\eta=2c-1>1$ results in
\begin{align*}
\Pr\left(1-\frac{|G_{ij}^x|}{|G_{ij}|}>2e^{-\frac{x^2}{2}}\right)&=
\Pr\left( \frac{1}{|G_{ij}|} \sum_{k=1}^{|G_{ij}|} X_k > 2c\bar p \right)
< e^{-\frac13 \eta\bar p |G_{ij}| }.
\end{align*}
Since $(1+\eta)\bar p=2\exp(-x^2/2)$ and $\eta>1$, we obtain that $\eta \bar p>(1+\eta)\bar p/2=\exp(-x^2/2)$, and consequently,
\begin{align}\label{eq:gchern}
\Pr\left(1-\frac{|G_{ij}^x|}{|G_{ij}|}>2e^{-\frac{x^2}{2}}\right)<e^{-\frac13 e^{-\frac{x^2}{2}}|G_{ij}|}.
\end{align}
Application of a union bound over $ij\in E$ to \eqref{eq:gchern} yields
\begin{align}\label{eq:gunion}
\Pr\left(1-\min_{ij\in E}\frac{|G_{ij}^x|}{|G_{ij}|}>2e^{-\frac{x^2}{2}}\right)
<|E|e^{-\frac13 e^{-\frac{x^2}{2}}\min\limits_{ij\in E}|G_{ij}|}<|E|e^{-\frac13 e^{-\frac{x^2}{2}}\min\limits_{ij\in E}|N_{ij}|(1-\lambda)}.
\end{align}
We note that
\begin{align}\label{eq:gunion2}
    \lambda_x =1-\min_{ij\in E}\frac{|G_{ij}^x|}{|N_{ij}|}=1-\min_{ij\in E}\frac{|G_{ij}|}{|N_{ij}|}\frac{|G_{ij}^x|}{|G_{ij}|}\leq 1-(1-\lambda)\min_{ij\in E}\frac{|G_{ij}^x|}{|G_{ij}|}.
\end{align}
The combination of \eqref{eq:gunion} and \eqref{eq:gunion2} results in
\begin{align*}
\Pr\left(\lambda_x<1-(1-\lambda)\left(1-2e^{-\frac{x^2}{2}}\right)\right)>1-|E|e^{-\frac13 e^{-\frac{x^2}{2}}\min\limits_{ij\in E}|N_{ij}|(1-\lambda)}.
\end{align*}
Applying the inequality $1-(1-\lambda)(1-2e^{-\frac{x^2}{2}})<\lambda+2e^{-\frac{x^2}{2}}$ for $0< \lambda <1/4$ to the above equation yields
\begin{align}
\label{eq:last_sub_gauss}
    \Pr\left(\lambda_x<\lambda+2e^{-\frac{x^2}{2}}\right)> 1-|E|e^{-\frac13 e^{-\frac{x^2}{2}}\min\limits_{ij\in E}|N_{ij}|(1-\lambda)}.
\end{align}
That is, with the probability indicated on the RHS of \eqref{eq:last_sub_gauss}, for any $ij\in E$, there is a subset of $N_{ij}$ whose proportion is at least $1-\lambda-2\exp(-x^2/2)$ and for any element indexed by $k$ in this subset, both $s_{ik}^*$ and $s_{jk}^*$ are bounded above by $\sigma\mu+\sigma x$. We thus conclude the proof by applying Theorems~\ref{thm:it2} and \ref{thm:ir2}, while replacing their parameters $\delta$ and $\lambda$ with the current parameters $\sigma\mu+\sigma x$ and $\lambda+2\exp(-x^2/2)$, respectively.

\subsection{Proof of Proposition \ref{prop:mode}}
\label{sec:prop_mode}

We note that under UCM, $\hat s_{ij} = s_{ij}^*$ with probability $1$ if and only if $|G_{ij}|\geq 2$. Indeed,  $|G_{ij}|\geq 2$ if and only if there are at least two elements in $D_{ij}$ that equal to $s_{ij}^*$. Furthermore, for $k\in B_{ij}$, the values of $d_{ij,k}$ are distinct from each other with probability $1$, due to the smooth density of the Haar measure over the continuous group. Thus, with probability 1, $s_{ij}^*$ is the mode of $D_{ij}$, or equivalently, $\hat s_{ij} = s_{ij}^*$.  

In order to conclude the proof we show that the condition $|G_{ij}|\geq 2$ holds for all $ij\in E$ with sufficiently high probability. We note that $X_k :=\mathbf{1}_{\{k\in G_{ij}\}}$ for $k\in [n]$ are i.i.d.~Bernoulli random variables with mean $\mu=p^2q_g^2$. By applying Chernoff bound in \eqref{eq:chernoff2} to $X_k$, we obtain that
\begin{align}\label{eq:Gij2}
    \Pr(|G_{ij}|\geq 2) = \Pr\left(\frac{1}{n}\sum_{k\in [n]} X_k\geq \frac{2}{np^2q_g^2}\mu\right)
    >  1-e^{-\frac13\left(1-\frac{2}{np^2q_g^2}\right)^2 p^2q_g^2 n}.
\end{align}
If $n/\log n\geq  10/(p^2q_g^2)$, then $2/(n p^2q_g^2)<1/5$ for $n>2$ and thus \eqref{eq:Gij2} implies that 
\begin{align*}
    \Pr(|G_{ij}|\geq 2) 
    >  1-e^{-\frac{16}{75}p^2q_g^2 n}.
\end{align*}
By taking a union bound over $ij\in E$ and applying the assumption $n/\log n\geq c/(p^2q_g^2)$ for $c\geq 10$, we obtain that
\begin{align*}
    &\Pr(\min_{ij\in E}|G_{ij}|\geq 2) > 1-n^2e^{-\frac{16}{75}p^2q_g^2 n} \geq  1-n^2 e^{-\frac{16c}{75}\log n}
    \\
    =& 1-n^{2-\frac{16c}{75}} \geq 1- n^{-\frac{2}{15}}.
\end{align*}
and consequently with the same probability $\hat s_{ij}=s_{ij}^*$ for all $ij\in E$.

\subsection{Proofs of the Preliminary Results of Section \ref{sec:prelim_ucm}}\label{sec:proofs_prelim_ucm}

\subsubsection{Proof of Proposition~\ref{prop:lambda}}

We first assume the case where $\sqrt2/2<q_*<1$, or equivalently, $q_{\min} =1-q_*^2$. For any fixed $ij\in E$, we define the random variables $X_k=\mathbf 1_{\{k\in B_{ij}\}}$, $k\in N_{ij}$.
We note that they are i.i.d.~Bernoulli with mean $q_{\min}=1-q_*^2$. 
We further note that 
$\lambda_{ij}$ is the average of $X_k$ over all $k\in N_{ij}$. Thus, direct application of Chernoff bound in
\eqref{eq:chernoff2} implies \eqref{eq:lambda1} in this case.

Next, we assume the case where $q_* \leq \sqrt2/2$, or equivalently, $q_{\min} =q_*^2$. For any $ij \in E$, define the random variables $Y_k=\mathbf 1_{\{k\in G_{ij}\}}$,  $k\in N_{ij}$. We note that they are i.i.d.~Bernoulli with mean $q_*^2$. By applying Chernoff bound (see \eqref{eq:chernoff2}) with  $\{Y_k\}_{k\in N_{ij}}$, we obtain \eqref{eq:lambda1} in this case and thus in general.

At last,  
applying a union bound over $ij\in E$ to \eqref{eq:lambda1} yields \eqref{eq:lambda2}.

\subsubsection{Proof of Proposition~\ref{prop:lambda2}}

The idea of the proof is similar to that of Proposition \ref{prop:lambda}. We first note that for any fixed $ij\in E$: the mean of $\mathbf 1_{\{k\in G_{ij}\}}$ for $k\in N_{ij}$ equals $q_*^2$; $\Pr(\lambda_{ij}>x)=\Pr(1-\lambda_{ij}<1-x)$; and $1-\lambda_{ij}$ is the average over ${k \in N_{ij}}$ of the Bernoulli random variables $\mathbf 1_{\{k\in G_{ij}\}}$. The proposition is concluded by applying a one-sided version of Chernoff bound in \eqref{eq:chernoff2} with $\{\mathbf 1_{\{k\in G_{ij}\}}\}_{k\in N_{ij}}$ (for each fixed $ij \in E$), $\eta=1-(1-x)/q_*^2$, $\mu=q_*^2$ and $m=|N_{ij}|$.

\subsubsection{Proof of Proposition~\ref{prop:initial1}}

We consider three disjoint cases of $k$'s in the sum of \eqref{eq:initial_corruption}. Since $\mathbb E(s_{ij}(0))=\mathbb E(d_{ij,k})$, we compute in each case the contribution of that case to the expectation of $d_{ij,k}$ given that case.

The first case is when $k\in G_{ij}$, so $d_{ij,k}=s_{ij}^*$, and thus the corresponding elements in \eqref{eq:initial_corruption} equal  $s_{ij}^*$. This case occurs w.p.~$q_g^2$.

The second case is when $k\in B_{ij}$ and either  $ik$ or $jk$ (but not both) is corrupted, and it occurs with probability $2q_g(1-q_g)$. Without loss of generality, we assume that $ik\in E_g$ and $jk\in E_b$. Using the bi-invariance of $d_\mathcal G$, we obtain that in this case, $d_{ij,k}=d_\mathcal G(g_{ij}g_{jk}g_{ki}^*,e_\mathcal G)=d_\mathcal G(g_{ki}^*g_{ij}g_{jk},e_\mathcal G)$. For any given $g_{ki}^*$ and $g_{ij}$, $g_{ki}^*g_{ij}g_{jk}\sim \text{Haar}(\mathcal G)$, due to the fact that $g_{jk}\sim \text{Haar}(\mathcal G)$ and the definition of Haar measure.
Thus, in this case $\mathbb E(d_{ij,k}|ik\in E_g, jk\in E_b)=z_\mathcal G.$

The last case is when $k\in B_{ij}$ and both $ik$ and $jk$ are corrupted. This case occurs with probability $(1-q_g)^2$. We claim that since $g_{jk}$, $g_{ki}\sim \text{Haar}(\mathcal{G})$
and $g_{jk}$ and $g_{ki}$ are independent, 
$g_{jk}g_{ki}\sim \text{Haar}(\mathcal{G})$.
Indeed, for any $g\in \mathcal G$, $gg_{jk}\sim \text{Haar}(\mathcal{G})$, and furthermore, $g_{ki}$ is independent of both $g_{jk}$ and $gg_{jk}$. Thus, $g_{jk}g_{ki}$ and $gg_{jk}g_{ki}$ are identically distributed  for any $g\in\mathcal G$ and thus $g_{jk}g_{ki}\sim \text{Haar}(\mathcal{G})$. Consequently, for fixed $g_{ij}$,  $g_{ij}g_{jk}g_{ki}\sim \text{Haar}(\mathcal{G})$
and thus
 \[\mathbb E(d_{ij,k}|ik\in E_b, jk\in E_b)=
 \mathbb E(d_\mathcal G(g_{ij}g_{jk}g_{ki},e_\mathcal G)|ik\in E_b, jk\in E_b) =
 z_\mathcal G.\] Combining all the three cases, we conclude \eqref{eq:initial1}.

\subsection{Proof of Lemma~\ref{lemma:tau1}}\label{sec:lemma:tau1}
Denote $\gamma_{ij} = |s_{ij}(0) - \mathbb E(s_{ij}(0))|$ 
for $ij \in E$ and $\gamma=\max_{ij \in E} \gamma_{ij}$, so that the condition of the lemma can be written more simply as 
${1}/{(4\beta_0)}\geq \gamma$.  By rewriting $s_{ij}(0)$ as $q_g^2s_{ij}^*+(1-q_g^2)z_\mathcal G+\gamma_{ij}$ and applying \eqref{eq:estimate_eps} with $t=0$,
\begin{align*}
\epsilon_{ij}(1)
\leq \frac{\sum\limits_{k\in B_{ij}}e^{-\beta_0\left(q_g^2s_{ik}^*+\gamma_{ik}+q_g^2s_{jk}^*+\gamma_{jk}\right)}\left(s_{ik}^*+ s_{jk}^*\right)}{\sum\limits_{k\in G_{ij}}e^{-\beta_0\left(q_g^2s_{ik}^*+\gamma_{ik}+q_g^2s_{jk}^*+\gamma_{jk}\right)}}.
\end{align*}
By first applying the obvious facts: $|\gamma_{ik}|,|\gamma_{jk}|\leq \gamma$ and $s_{ik}^*=s_{jk}^*=0$ for $k\in G_{ij}$, then applying the assumption ${1}/{(4\beta_0)}\geq \gamma$, and at last the inequality $xe^{-ax}\leq 1/(ea)$ for $x$, $a>0$ with $x=s_{ik}^{*}+s_{jk}^{*}$ and $a=\beta_0q_g^2$, we obtain that
\begin{align*}
\epsilon_{ij}(1)
&\leq  \frac{\sum\limits_{k\in B_{ij}}e^{-\beta_0q_g^2\left(s_{ik}^*+s_{jk}^*\right)}\left(s_{ik}^*+ s_{jk}^*\right)e^{4\beta_0\gamma}}{|G_{ij}|}\nonumber\\ &\leq
\frac{e\sum\limits_{k\in B_{ij} }e^{-\beta_0q_g^2\left(s_{ik}^*+s_{jk}^*\right)}\left(s_{ik}^*+ s_{jk}^*\right)}{|G_{ij}|}\leq \frac{|B_{ij}|}{|G_{ij}|}\frac{1}{q_g^2\beta_0}.
\end{align*}
The lemma is concluded by maximizing over $ij\in E$ both the LHS and RHS of the above inequality.
\subsection{Proof of Lemma~\ref{lemma:M}}\label{sec:lemma:M}
We prove \eqref{eq:lemma:M}, or equivalently, $\epsilon (t)<{1}/({4\beta_t})$ for all $t\geq 1$, by induction. We note that $\epsilon (1)<({1}/{4\beta_1})$ is an assumption of the lemma.  We next show that $\epsilon (t+1)<{1}/({4\beta_{t+1}})$ if  $\epsilon (t)<{1}/({4\beta_t})$. By combining \eqref{eq:bound_2_eps2} and the induction assumption $\epsilon (t)<{1}/({4\beta_t})$ and then using the definition of $\lambda$,
\begin{align*}
\epsilon_{ij}(t+1)
\leq \frac{e \sum\limits_{k\in B_{ij}}e^{-\beta_t\left(s_{ik}^*+s_{jk}^*\right)}\left(s_{ik}^*+ s_{jk}^*\right)}{|G_{ij}|}
\leq e\frac{\lambda}{1-\lambda}\frac{1}{|B_{ij}|}\sum\limits_{k\in B_{ij}}e^{-\beta_t\left(s_{ik}^*+s_{jk}^*\right)}\left(s_{ik}^*+ s_{jk}^*\right).
\end{align*}
Combining \eqref{eq:Mbeta} with the above equation, then applying the definition of $M$, and at last using  $\beta_{t+1}=\beta_t/r$,
\begin{align*}
\epsilon (t+1)<e\frac{\lambda}{1-\lambda}\frac{1}{M\beta_t}=\frac{r}{4\beta_t}=\frac{1}{4\beta_{t+1}}.
\end{align*}

\subsection{Proof of Lemma~\ref{lemma:risk}}\label{sec:risk}
We arbitrarily fix $ij\in E$ and $\beta>0$. We denote $m=|B_{ij}|$ and assume that $k=1, \ldots, m$ index the elements of $B_{ij}$.
We use the i.i.d.~random variables $X_k=s_{ik}^*+s_{jk}^*$, $k=1, \ldots, m$, with cdf denoted (as earlier) by $P$.  Let $\mathcal P$ and $\mathcal P_m$ denote the functionals that provide the expectation with respect to the probability and empirical measures of $\{X_k\}_{k=1}^m$, respectively. That is, $\mathcal Pf=\int f(x) d P(x)$ and $\mathcal P_m f=\frac{1}{m}\sum_{k=1}^m f(X_k)$. For any functional $\mathcal Y: \mathcal F(\beta)\to \mathbb R$, let $\| \mathcal Y\|_{\mathcal F(\beta)}=\sup_{f\in \mathcal F(\beta)}|\mathcal Y(f)|$.
Given this notation, we can rewrite 
\eqref{eq:riskH} that we need to prove as follows 
\begin{align}\label{eq:riskH1}
\Pr\left(\|\mathcal P_m-\mathcal P\|_{\mathcal F(\beta)}>V(\beta)+c\sqrt{\frac{\log m}{m}}\right)<e^{-\frac{1}{3}mV(\beta)}.
\end{align}

The above formulation is similar to the following uniform version of Bennett's inequality in our setting (see Theorem 2.3 of \cite{Bousquet}): For any $t>0$
\begin{align}\label{eq:bousquet}
\Pr(\|\mathcal P_m-\mathcal P\|_{\mathcal F(\beta)}>\mathbb E\| \mathcal P_m- \mathcal P\|_{\mathcal F(\beta)}+t)<e^{-mvh\left(\frac tv\right)},
\end{align}
where $h(x)=(x+1)\log(x+1)-x$ and $v=V(\beta)+2\mathbb E\| \mathcal P_m- \mathcal P\|_{\mathcal F(\beta)}$ ($V$ is the same as ours).  We remark that \eqref{eq:bousquet} holds under the condition that $\sup_{f_\tau\in \mathcal F(\beta)}\|f_\tau-\mathcal P f_\tau\|_\infty\leq 1$. This condition holds in our setting since $0\leq f_\tau (x)\leq 1$ for any $\tau\geq 0$ and $x\geq 0$.

In order to conclude \eqref{eq:riskH1} from \eqref{eq:bousquet}, we formulate the following lemma that provides an upper bound for $\mathbb E\| \mathcal P_m- \mathcal P\|_{\mathcal F(\beta)}$ in \eqref{eq:bousquet}. We prove it in Section~\ref{sec:entropy} below.

\begin{lemma}\label{lemma:entropy}
Assume the setting of Theorem \ref{thm:rand2}. There exists an absolute constant $c_1$ such that for all $\beta, m>0$
\begin{align}\label{eq:lemma:entropy}
\mathbb E\| \mathcal P_m- \mathcal P\|_{\mathcal F(\beta)} \leq  c_1\sqrt{\frac{\log m}{m}}.
\end{align}
\end{lemma}

By letting  $t=V(\beta)+2c_1\sqrt{\log m/m}$
in \eqref{eq:bousquet} 
and $c=3c_1$ in \eqref{eq:riskH1} and applying Lemma \ref{lemma:entropy}, we conclude that the event of \eqref{eq:riskH1} contains the event of \eqref{eq:bousquet}. It thus remains to show that the probability bound in \eqref{eq:riskH1} controls the one in \eqref{eq:bousquet}. This follows from the facts that $t/v>1$ (which follows by direct application of Lemma \ref{lemma:entropy}) and 
 $h(x)>x/3$ when $x\geq 1$ (which is a direct calculus exercise).

\subsubsection{Proof of Lemma~\ref{lemma:entropy}}\label{sec:entropy}
In order to upper bound $\mathbb E\| \mathcal P_m- \mathcal P\|_{\mathcal F(\beta)}$ we use tools from empirical risk minimization. Define $\mathcal R_m(f)=\frac{1}{m}\sum_{k=1}^m \epsilon_k f_\tau(X_k)$, where $\epsilon_k$ are i.i.d.~Rademacher random variables. We first note that $\mathbb E\| \mathcal P_m- \mathcal P\|_{\mathcal F(\beta)}$ can be controlled by the Rademacher complexity of $\mathcal F(\beta)$, which is defined as $\mathbb E \|\mathcal R_m\|_{\mathcal F(\beta)}$. Specifically,
 Theorems 2.1 and 3.11 of \cite{oracle_ERM} state that there exists an absolute constant $c_2$ such that
\begin{align}\label{eq:Dudley}
\mathbb E\| \mathcal P_m- \mathcal P\|_{\mathcal F(\beta)}\leq 2\mathbb E\|\mathcal R_m\|_{\mathcal F(\beta)}\leq \frac{c_2}{\sqrt m}\mathbb E\int_0^{2\sigma_m}\sqrt{\log N(\mathcal F(\beta);\ell_2(\mathcal P_m);\varepsilon)}d\varepsilon,
\end{align}
where $\sigma_m^2=\sup_{f_\tau\in \mathcal F(\beta)}\mathcal P_m f^2$ and  $N(\mathcal F(\beta);\ell_2(P_n);\varepsilon)$ is the covering number of $\mathcal F(\beta)$ using $\ell_2(P_m)$-balls of radius $\varepsilon$. Note that the $\ell_2(P_m)$-ball of radius $\varepsilon$ centered at any function $f_{\tau^*}(x)\in \mathcal F(\beta)$ is defined as $\left\{f_\tau\in \mathcal F(\beta): \frac1m\sum_{k=1}^m (f_\tau(X_k)-f_{\tau^*}(X_k))^2<\varepsilon^2\right\}$. In view of  \eqref{eq:Dudley}, since $\mathcal F(\beta)\subseteq \mathcal F(0)$ for any $\beta>0$, we can prove  \eqref{eq:lemma:entropy} by showing that there exists an absolute constant $c_3$ such that 
\begin{align}\label{eq:erm}
    \mathbb E\int_0^{2\sigma_m}\sqrt{\log N(\mathcal F(0);\ell_2(\mathcal P_m);\varepsilon)}d\varepsilon\leq c_3\sqrt{\log m}.
\end{align}
In order to conclude \eqref{eq:erm}, we first give an upper bound for $N(\mathcal F(0);\ell_2(P_m);\varepsilon)$ for fixed $\varepsilon$, $m$ and $\{X_k\}_{k=1}^m$ by constructing a specific $\ell_2(P_m)$-ball covering $\{B_i\}_{i=1}^{N_\varepsilon}$ of $\mathcal F(0)$. We note that since  $f_\tau(X_k)\leq 1$ for any $f\in \mathcal F(0)$ and $X_k\geq 0$, the covering number $N(\mathcal F(0);\ell_2(P_m);\varepsilon)$ equals 1 for all $\varepsilon \geq 1$; therefore, its log is zero and in this case there is no contribution to the integral in \eqref{eq:erm}. It is thus sufficient to consider $\varepsilon<1$. For simplicity, we represent each ball $B_i$ in our proposed cover by an interval $I_i=[a_i, b_i)$ that indicates the range of parameters $\tau$ of functions in $B_i$. In our construction, $I_1=[a_1,\infty)$, $b_{i+1} = a_{i}$ for $i=1, \ldots, N_\varepsilon-1$ and $\{I_i\}_{i=1}^{N_\varepsilon}$ cover $[0,\infty)$. This implies that $B_i=\{f_\tau: \tau \in I_i\}$, $i=1, \ldots, {N_\varepsilon}$, cover $F(0) = \{f_\tau: \tau \in [0, \infty)\}$.

We define
\begin{align}
\label{eq:def_I1}
I_1 = \left(\frac{2\log (\frac1\varepsilon)+2}{\min_{1\leq k\leq m}{X_k}},\infty\right).
\end{align}
We claim that the ball $B_1=\{f_\tau: \tau \in I_1\}$ is contained in $B(0,\varepsilon)$, whose center $f_\tau(x)\equiv 0$ corresponds to $\tau=\infty$.  
Indeed, if $\tau\in I_1$ and $\varepsilon<1$, then $\tau X_k>2\log(1/\varepsilon)+2>2$ and in particular  $\exp(\frac12\tau X_k)>\tau X_k$. Using these inequalities, we verify our claim as follows
\begin{align*}
\sqrt{\frac1m\sum_{k=1}^m (e^{-\tau X_k+1}\tau X_k)^2}\leq \max_{1\leq k\leq m} e^{-\tau X_k+1}\tau X_k<\max_{1\leq k\leq m} e^{-\frac{1}{2}\tau X_k+1}<\varepsilon.
\end{align*}

Given $I_i=(a_i,b_i]$, we define $I_{i+1}=(a_{i+1}, b_{i+1}]$, where $b_{i+1}=a_i$ and $a_{i+1}=a_{i}-\varepsilon/(2e)$, so that $|I_{i+1}| = \varepsilon/(2e)$. We claim that $B_{i+1}=\{f_\tau: \tau \in I_{i+1}\}$ is contained in $B(f_{b_{i+1}},\varepsilon)$. Indeed, since the function $xe^{x+1}$ is Lipschitz with constant $e$ and $0\leq X_k\leq 2$, for any $\tau \in I_{i+1}$
\begin{align*}
&\sqrt{\frac1m\sum_{k=1}^m (f_{\tau}(X_k)-f_{b_{i+1}}(X_k))^2}\leq 
\max_{1\leq k\leq m}|e^{-\tau X_k+1} \tau X_k-e^{-b_{i+1} X_k+1}b_{i+1} X_k|\nonumber\\
&\leq e\max_{1\leq k\leq m}|\tau-b_{i+1}|X_k\leq 2e|a_{i+1}-b_{i+1}| = \varepsilon.
\end{align*}
We have thus obtained a covering of $\mathcal F(0)$ by $\ell_2(P_m)$-balls with radius $\varepsilon$. The total number of corresponding intervals (where intervals $I_i$, $i \geq 2$, cover $(0,a_1)$ and have length $\varepsilon/(2 e)$)  is at most $2e
a_1/\varepsilon+1$. Using this observation and the value of $a_1$ specified in \eqref{eq:def_I1}, then applying the facts  $X_k\leq 2$ and $\varepsilon<1$, and at last the inequality $1+\log x\leq x$, we obtain that
\begin{align}\label{eq:Nupper}
N(\mathcal F(\beta);\ell_2(P_m);\varepsilon)\leq 4e\frac{(\log (\frac1\varepsilon)+1)\frac{1}{\varepsilon}}{\min_{1\leq k\leq m}{X_k}}+1< 6e\frac{(\log (\frac1\varepsilon)+1)\frac{1}{\varepsilon}}{\min_{1\leq k\leq m}{X_k}}< e^3\frac{1}{\min_{1\leq k\leq m}{X_k}}\frac{1}{\varepsilon^2}.
\end{align}

We note that the cdf of ${\min_{1\leq k\leq m}{X_k}}$ is $1-(1-P(x))^m$. Combining this observation, the fact that $\varepsilon < 1$ and \eqref{eq:Nupper}, and then applying basic inequalities, using the notation $a_+:=\max(a,0)$, and in particular final application of Jensen's inequality with the concave function $\sqrt x$, we obtain that 
\begin{align}
&\mathbb E\int_0^{2\sigma_m}\sqrt{\log N(\mathcal F(\beta);\ell_2(\mathcal P_m);\varepsilon)}d\varepsilon\nonumber
<  \int_0^2\int_0^1\sqrt{\log\frac{1}{x}+2\log\frac{1}{\varepsilon}+3}\,\,d\varepsilon\, d(1-(1-P(x))^m)\nonumber\\
\leq & \int_0^2\int_0^1\left(\sqrt{\left(\log\frac{1}{x}\right)_+}+\sqrt{2\left(\log\frac{1}{\varepsilon}\right)}\right)d\varepsilon\, d(1-(1-P(x))^m)+\sqrt 3\nonumber\\
=&\,\,\int_0^1 \sqrt{\log\frac{1}{x}}\,d(1-(1-P(x))^m)+\sqrt2 \int_0^1\sqrt{\log\frac{1}{\varepsilon}}d\varepsilon+\sqrt3\nonumber\\
\leq&\sqrt {\int_0^1 \log\frac{1}{x}\,d(1-(1-P(x))^m)} +\sqrt2+\sqrt 3.
\label{eq:logm}
\end{align}

Next, we give an upper bound for the first term in the RHS of \eqref{eq:logm}, while considering the two cases of Theorem \ref{thm:rand2}. If $X_k$, $1 \leq k \leq m$, is supported on $[a,\infty)$ and $a\gtrsim 1/m$, then
\begin{align}
&\int_0^1 {\log\frac{1}{x}}\,d(1-(1-P(x))^m)\leq\left(\log \frac{1}{a}\right)_+\lesssim \log m.
\end{align}

If on the other hand, the quantile function $Q(x)$ is differentiable and $Q'(x)/Q(x)\lesssim 1/x$ for $x<P(1)$, 
then we substitute $u=1-P(x)$ and obtain that
\begin{align}\label{eq:smoothlog}
&\int_{x\in[0,1]} {\log\frac{1}{x}}\,d(1-(1-P(x))^m)
=-\int_{u\in [1-P(1),1]} \log\frac{1}{Q(1-u)}\,d(1-u^m)\nonumber\\
=&\int_{u\in [1-P(1),1]} (1-u^m)d\log\frac{1}{Q(1-u)}
=\int_{1-P(1)}^1 (1-u^m)\frac{Q'(1-u)}{Q(1-u)}du \nonumber\\ \lesssim & \int_{0}^1 \frac{1-u^m}{1-u}du
= \int_{0}^1 \sum_{i=0}^{m-1} u^i du
=\sum_{i=1}^m \frac1i\leq (\log m+1).
\end{align}

Combining \eqref{eq:logm}-\eqref{eq:smoothlog}, we conclude \eqref{eq:erm} and thus Lemma \ref{lemma:entropy}.


%
%

\bibliographystyle{spmpsci}      
\bibliography{robust_sync.bib}   

\end{document}